%% file: main_arxiv.tex
\documentclass[11pt]{article}
\input{preamble_arxiv}

\title{The Cost of an Unknown Horizon for the Last Iterate\\
of Clipped SGD under Heavy-Tailed Noise}
\author[1]{Yuxing Peng}
\author[2,*]{Zhiqing Tang}
\author[2,3]{Weijia Jia}
\affil[1]{Faculty of Arts and Sciences, Beijing Normal University, Zhuhai 519087, China}
\affil[2]{Institute of Artificial Intelligence and Future Networks, Faculty of Arts and Sciences, Beijing Normal University, Zhuhai 519087, China}
\affil[3]{Guangdong Key Lab of AI and Multi-Modal Data Processing, Beijing Normal-Hong Kong Baptist University, Zhuhai 519087, China}
\affil[*]{Corresponding author: \texttt{zhiqingtang@bnu.edu.cn}}
\date{}

\begin{document}
\maketitle

\begin{abstract}
\input{sections/00_abstract}
\end{abstract}

\input{sections/01_introduction}
\input{sections/02_setting}
\input{sections/03_known_horizon}

\input{sections/04_time_uniform}
\input{sections/05_fixed_time}
\input{sections/06_momentum}
\input{sections/07_discussion}

\bibliographystyle{abbrvnat}
\bibliography{references}

\section*{AI Usage Statement}
During the preparation of this manuscript, the authors used AI tools, including OpenAI's GPT models and Anthropic's Claude, for mathematical derivations and their verification, and to assist with writing. The authors take full responsibility for the content of this paper.

\clearpage
\appendix
\input{appendices/related}
\input{appendices/experiment}
\input{appendices/open_problems}
\input{appendices/appendix}

\end{document}

%% file: preamble_arxiv.tex
\usepackage[T1]{fontenc}
\usepackage{lmodern}
\usepackage[letterpaper,margin=1in]{geometry}
\usepackage{microtype,xurl,needspace}
\usepackage{booktabs,array,tabularx}
\usepackage{xcolor}
\usepackage{colortbl}
\usepackage{natbib}
\AtBeginDocument{\setcitestyle{round,semicolon}}
\usepackage{amsmath,amssymb,amsthm,mathtools}
\usepackage{algorithm,algorithmic}
\usepackage{authblk}

\usepackage{hyperref}
\hypersetup{
  colorlinks=true,
  allcolors=blue!60!black,
  pdfborder={0 0 0},
  hypertexnames=false,
  pdftitle={The Cost of an Unknown Horizon for the Last Iterate of Clipped SGD under Heavy-Tailed Noise},
  pdfauthor={Yuxing Peng, Zhiqing Tang, Weijia Jia},
  pdfkeywords={stochastic convex optimization, gradient clipping, heavy-tailed noise, last iterate, time-uniform guarantees}
}
\newcommand{\leanavailability}{The code is provided as ancillary files of this arXiv entry and at \url{https://github.com/Thefootofhappiness/clipped-sgd-last-iterate-lean}.}
\newcommand{\experimentavailability}{The script that produces the figure and its outputs are provided as ancillary files of this arXiv entry.}
\newenvironment{resulttable}{\begin{table}[t]}{\end{table}}
\input{preamble_macros}

%% file: preamble_macros.tex
\usepackage{tikz}
\usetikzlibrary{arrows.meta,decorations.pathreplacing,calc}
\usepackage{pgfplots}
\pgfplotsset{compat=1.18}
\allowdisplaybreaks[2]
\theoremstyle{plain}
\newtheorem{theorem}{Theorem}
\newtheorem{lemma}[theorem]{Lemma}
\newtheorem{proposition}[theorem]{Proposition}

\theoremstyle{definition}
\newtheorem{definition}[theorem]{Definition}
\newtheorem{remark}[theorem]{Remark}
\newtheorem{assumption}{Assumption}

\renewenvironment{proof}{\par\addvspace{6pt}\noindent\textbf{Proof}\ \ignorespaces}%
  {\hfill$\blacksquare$\par\addvspace{6pt}}

\newcommand{\R}{\mathbb R}
\newcommand{\E}{\mathbb E}
\newcommand{\Prob}{\mathbb P}
\newcommand{\F}{\mathcal F}
\newcommand{\norm}[1]{\lVert#1\rVert}
\newcommand{\ip}[2]{\langle#1,#2\rangle}
\newcommand{\clip}{\operatorname{clip}}
\newcommand{\ind}{\mathbf 1}
\newcommand{\Et}{\E_{t-1}}

%% file: sections/00_abstract.tex
When training stops at a data-dependent time, a guarantee for the returned model must hold at all times simultaneously. We study the last iterate of clipped SGD on smooth convex objectives under heavy-tailed stochastic gradients with a bounded central $p$th moment, $1<p\le2$. With a known horizon $T$, clipped SGD reaches the optimal last-iterate rate $O(T^{-\alpha_p})$, where $\alpha_p=(p-1)/p$ (Liu, 2026; Jiang et al., 2026). Without a known horizon, Chezhegov et al.\ (ICLR 2026) obtained the time-uniform rate $\tilde O(t^{-2(p-1)/(3p)})$ and conjectured that a better schedule reaches the optimal rate. We show that for clipped SGD, and more generally for every rule whose iterate does not advance on rounds without new gradient information, the optimal time-uniform rate is $\tilde O(t^{-\beta_p})$, where $\beta_p=2(p-1)/(3p-1)$. Our deterministic schedule for clipped SGD uses neither a horizon nor a confidence input and attains $\tilde O(t^{-\beta_p})$ at all times simultaneously, and hence almost surely. Conversely, no such rule attains $\tilde O(t^{-r})$ at all times for any $r>\beta_p$; for clipped SGD, this holds even with adaptive, randomized step sizes and clipping thresholds. Under finite variance, the known-horizon, optimal time-uniform, and previous time-uniform exponents are $1/2$, $2/5$, and $1/3$. Since $\beta_p<\alpha_p$, the conjecture, which holds with a known horizon, fails without one. By contrast, an update that queries gradients at the running average can advance the iterate on rounds without new gradient information and attains $\tilde O(t^{-\alpha_p})$ at all times. Memory alone does not suffice: heavy-ball momentum with geometric forgetting and polynomially bounded step sizes cannot improve on the exponent $\beta_p$.

%% file: sections/01_introduction.tex
\section{Introduction}\label{sec:introduction}

The length of a training run is rarely fixed before optimization begins. Early stopping \citep{aolaritei2026} and checkpoint selection pick a stopping point from the observed trajectory, large-scale runs increasingly avoid committing to a fixed duration \citep{hagele2024}, and schedule-free methods are designed so that no stopping time has to be specified \citep{defazio2024roadless}. The model that is kept is the state of the run at the chosen time rather than an average of past iterates. In such runs, the stochastic gradients are often heavy-tailed \citep{simsekli2019,zhang2020adaptive}, and gradient clipping \citep{pascanu2013} is a standard remedy used routinely in training large language models \citep{touvron2023llama}. We ask what an unknown horizon costs: how fast can the last iterate of clipped stochastic gradient descent (SGD; \citealt{robbinsmonro1951}) converge with high probability when the run may stop at a data-dependent time?

Fixed-time guarantees do not answer this question. A bound that holds with probability at least $1-\delta$ at each deterministic step $t$ relies on an event that may change with $t$, so the bound can fail when the stopping time is chosen from the observed trajectory, just as continuously monitoring an A/B test with a fixed-sample rule inflates Type~I error \citep{howard2021}. A guarantee holds at every bounded stopping time if and only if it is \emph{time-uniform}, meaning that a single event of probability at least $1-\delta$ covers all steps $t\ge1$ simultaneously (Section~\ref{sec:setting}). Under sub-Gaussian noise, time-uniform guarantees pay only logarithmic factors beyond the fixed-horizon exponent $1/2$ \citep{lichenwang2026}.

To study heavy-tailed noise, we consider a differentiable convex $L$-smooth objective $f:\R^d\to\R$ with a minimizer $x^\star$ satisfying $\norm{x_1-x^\star}\le R$, and conditionally unbiased stochastic gradients $g_t$ queried at $x_t$ that satisfy
\[
 \E\bigl[\norm{g_t-\nabla f(x_t)}^p\bigm|\F_{t-1}\bigr]\le\sigma^p,\qquad 1<p\le2 .
\]
Clipped SGD updates its iterate according to
\begin{equation}\label{eq:recursion}
 \begin{gathered}
 x_{t+1}=x_t-\eta_t\clip(g_t,\lambda_t),\\
 \clip(g,\lambda)=\min\{1,\lambda/\norm g\}\,g,
 \end{gathered}
\end{equation}
where $\eta_t$ is the step size and $\lambda_t$ is the clipping threshold. Without clipping, SGD with predefined step sizes attains the minimax optimal rate in expectation on convex problems whose stochastic gradients have bounded $p$th moments \citep{fatkhullin2026}, but its high-probability guarantees can require polynomial dependence on $1/\delta$ \citep{sadiev2023,fatkhullin2026}. Clipping achieves logarithmic dependence on $1/\delta$ \citep{nazin2019,gorbunov2020,cutkosky2021,sadiev2023,nguyen2023}, as in robust mean estimation \citep{catoni2012,devroye2016,lugosi2019}. Most of these high-probability guarantees apply to an averaged or auxiliary output rather than the last iterate. For an averaged output of clipped stochastic mirror descent, \citet[Theorem~4.1]{nguyen2023} proved that $T$ oracle calls suffice to reach, with probability at least $1-\delta$, the statistical scale
\begin{equation}\label{eq:statistical-scale}
 \sigma R\Bigl(\frac{\log(1/\delta)}{T}\Bigr)^{(p-1)/p}
\end{equation}
plus a deterministic term of order $LR^2\log(1/\delta)/T$. We write $\alpha_p=(p-1)/p$ for the statistical exponent in~\eqref{eq:statistical-scale}. For $\delta\le1/8$ and $T\ge\log(1/\delta)$, \citet[Proposition~1]{jiang2026linear} proved that no algorithm using $T$ stochastic gradients improves on the minimum of $LR^2$ and the scale~\eqref{eq:statistical-scale} by more than a constant factor.

For the last iterate $x_{t+1}$ itself, existing bounds depend on whether the horizon is given in advance. With a known horizon $T$, \citet[Theorem~2]{liu2026lastiterate} and \citet[Theorem~1]{jiang2026linear} proved that clipped SGD with linearly decaying step sizes attains the optimal scale~\eqref{eq:statistical-scale} at $x_{T+1}$; we restate the explicit bound of \citet{jiang2026linear} as Theorem~\ref{thm:upper}. Without a known horizon, \citet[Theorem~1 and Corollary~1]{chezhegov2026} obtained the slower rate $\tilde O(t^{-2(p-1)/(3p)})$ using a schedule that takes $\delta$ as input; because their proof controls nested events across all steps, their guarantee is time-uniform. Guided by numerical experiments, \citet[Section~5]{chezhegov2026} conjectured that a better schedule reaches the optimal rate, while \citet[Section~5]{liu2026lastiterate} left the unknown-horizon case open. Since linear decay closes the last-iterate gap for a known horizon and time uniformity costs only logarithmic factors under sub-Gaussian noise, one would expect the gap between the exponents $2(p-1)/(3p)$ and $\alpha_p$ to be an artifact of the analysis.

\subsection{Main Results}\label{sec:results}\label{sec:meaning}

We show that under heavy-tailed noise, part of this gap is inherent. Across the two regimes, clipped SGD attains its optimal last-iterate errors
\begin{gather*}
 O\bigl(T^{-\alpha_p}\bigr)\ \text{ with a known horizon }T,\\
 \tilde O\bigl(t^{-\beta_p}\bigr)\ \text{ at all times }t\text{ simultaneously},
\end{gather*}
where $\beta_p=2(p-1)/(3p-1)$ lies strictly between the exponents $2(p-1)/(3p)$ and $\alpha_p$, and $\tilde O(\cdot)$ hides constants and factors polylogarithmic in $t/\delta$. Under finite variance ($p=2$), the known-horizon, optimal time-uniform, and previous time-uniform exponents are $1/2$, $2/5$, and $1/3$. The conjecture of \citet{chezhegov2026} holds with a known horizon, as the results of \citet{liu2026lastiterate} and \citet{jiang2026linear} show, but fails without one, because $\beta_p<\alpha_p$. On the upper-bound side, we give a deterministic schedule with $\eta_t\propto t^{-(p+1)/(3p-1)}/\log^2(100t^2)$ and $\lambda_t\propto t^{2/(3p-1)}$ that uses neither the horizon nor the confidence level, and we prove that clipped SGD with this schedule attains error $\tilde O(t^{-\beta_p})$ at all times simultaneously (Theorem~\ref{thm:anytime-upper}). Because a single trajectory serves every $\delta\in(0,1)$, the rate $\tilde O(t^{-\beta_p})$ also holds almost surely. At $p=2$, this upper bound separates clipped SGD from unclipped SGD with deterministic step sizes: in concurrent work, \citet[Theorem~2.2]{anonymous2026cuberoot} prove that deterministic unclipped schedules attain no almost sure rate $\tilde O(t^{-r})$ for $r>1/3$, and we show that their construction also excludes time-uniform exponents above $1/3$ on our problem class (Appendix~\ref{app:cuberoot}, Remark~\ref{rem:cuberoot}).

A lower bound with exponent $r$ takes the following form: for every $r'>r$, any claimed guarantee $\tilde O(t^{-r'})$ fails on some instance with probability greater than $\delta$, so no exponent above $r$ is attainable. We begin with horizon-free \emph{gradient-only rules} (Definition~\ref{def:class}), which at each round $t$ query one fresh stochastic gradient $g_t$ at the current iterate $x_t$ and update $x_{t+1}=x_t-c_t g_t$ with a scalar $c_t\ge0$ that may depend on the observed history, on $g_t$, and on independent randomness. Every clipped SGD method with adaptive, randomized step sizes and thresholds belongs to this class. Even on one-dimensional instances with i.i.d.\ noise, we prove that no gradient-only rule attains a time-uniform guarantee $\tilde O(t^{-r})$ for any $r>\beta_p$ (Theorem~\ref{thm:anytime}). We then extend this lower bound to every horizon-free rule whose iterate does not advance on rounds that bring no new gradient information, meaning that on one-dimensional instances a round with $g_t=0$ does not move the iterate past its earlier maximum in a direction fixed before the run (Proposition~\ref{prop:record-lower}). This broader class includes updates with fixed anchors, contracting weight decay toward the initialization, clipped KT coin betting \citep{orabona2016coin}, and the conic-lifting methods CBA, CBA$^+$, and conic FTRL of \citet{grandclement2026conic} in one dimension (Appendix~\ref{app:learning-rate-free}).

Why can these rules not reach the exponent $\alpha_p$ without a known horizon? Our hard instances are translated one-dimensional smooth convex objectives whose stochastic gradients contain rare Bernoulli spikes; we index them by their population minimizers, or targets, $y\in[\mu,1]$, where $\mu\in(0,1/2]$ is a parameter. Before the iterate crosses $y-\mu$, the oracle of every target in this family returns zero gradient on quiet rounds and the same negative spike on rare rounds, which occur with probability $\pi\asymp\mu^{p/(p-1)}$. Because the oracles of all uncrossed targets produce the same observation transcript, any rule whose iterate does not advance on zero gradients can move forward toward distant targets only on the rare spike rounds. Suppose that the rule promises an objective gap below an envelope $b_t$ at every round $t$. If the iterate first crosses $y-\mu$ at round $t$ and lands more than $\mu+\sqrt{2b_t}$ beyond this boundary, it overshoots $y$ and violates that target's envelope. When we integrate the probabilities of failing at the first crossing over the targets $y$, we find that if a rule keeps the gap below $b_t$ at all rounds $t\le N$ with probability at least $1-\varepsilon$ on every target, then the envelope values satisfy the progress budget $\pi\sum_{t=1}^N(\mu+\sqrt{2b_t})\ge(1-2\varepsilon)/2$ for every $\mu\in(4b_N,1/2]$ (Proposition~\ref{prop:simultaneous-budget}). For an envelope that decays like $t^{-u}$ with $u>\beta_p$ and $\mu$ of order $b_N$, the left-hand side tends to zero as $N\to\infty$ (Section~\ref{sec:anytime}), a contradiction when $\varepsilon<1/2$. A method that queries stochastic gradients at the running average is not bound by this budget: even on rounds with zero gradient, its iterate keeps moving toward a base sequence that accumulates past clipped gradient steps. Using a $\delta$-dependent schedule, we prove that a clipped running-average update attains $\tilde O(t^{-\alpha_p})$ at all times simultaneously (Theorem~\ref{thm:averaging-positive}). The three-sequence accelerated clipped method of \citet{nguyen2023} also attains $\alpha_p$ without a known horizon; their theorem is stated at each fixed time, and their proof gives the time-uniform form (Appendix~\ref{sec:averaging-positive}). Reaching $\alpha_p$ without a known horizon therefore requires an output that can advance without new gradient information; rescaling individual gradients is not enough.

When a guarantee is required only at one deterministic time $t$ rather than at all times simultaneously, the cost of an unknown horizon moves from the time exponent into the confidence factor (Section~\ref{sec:marginal}). For every $r\in(\beta_p,\alpha_p]$, we give horizon-free clipped SGD schedules that attain the fixed-time error $\tilde O(\delta^{-c}t^{-r})$ with $c=(3p-1)(r-\beta_p)/(p-1)$ when $\delta$ is supplied as input, and with a larger exponent $c$ without a confidence input; at $r=\alpha_p$, the confidence factors are $\delta^{-\alpha_p}$ with a confidence input and $\delta^{-1}$ without one (Theorem~\ref{thm:fixed-time-upper}). Conversely, by applying our progress budget across finite prefixes via a union bound, we prove that any family of gradient-only rules attaining $\tilde O(\delta^{-c}t^{-r})$ at each fixed time with $\delta$-independent constants must satisfy $c\ge r-\beta_p$ (Theorem~\ref{thm:anytime-marginal}). Thus the optimal fixed-time confidence exponent $c^\star(r)$ satisfies $r-\beta_p\le c^\star(r)\le(3p-1)(r-\beta_p)/(p-1)$; its exact value remains open (Appendix~\ref{app:open}). With polylogarithmic dependence on $1/\delta$, gradient-only rules are again limited to the rate $\tilde O(t^{-\beta_p})$.

Expressed as heavy-ball momentum \citep{polyak1964}, the running-average update of Theorem~\ref{thm:averaging-positive} has momentum coefficients $\beta_t$ with $1-\beta_t$ of order $1/t$, so its effective memory grows over time (Section~\ref{sec:momentum-boundary}). We prove that memory that does not grow cannot improve on the exponent $\beta_p$ when step sizes are polynomially bounded. Starting from zero velocity, clipped heavy-ball momentum with geometric forgetting ($\beta_t\le\bar\beta<1$) attains no time-uniform bound $\tilde O(t^{-r})$ for $r>\beta_p$ even in one dimension; under polynomial forgetting $1-\beta_t\ge\rho(1+t)^{-a}$ with $a\in[0,1)$, we prove the exponent barriers $(1+a)\beta_p$ in one dimension (Theorem~\ref{thm:momentum-obstruction}) and $\max\{\beta_p,a\}$ when the control and rate constants are shared across dimensions (Theorem~\ref{orth:thm:main}). For learning-rate-free optimization, we also prove that the zero-initialized conic Blackwell algorithms CBA and CBA$^+$ in the form of \citet{grandclement2026conic} admit no time-uniform last-iterate bound that tends to zero at all, even under i.i.d.\ additive noise (Proposition~\ref{prop:cba-no-rate}, Appendix~\ref{app:learning-rate-free}). We summarize the optimal last-iterate rates at $p=2$ in Table~\ref{tab:summary}.

\begin{resulttable}
\centering
\renewcommand{\arraystretch}{1.55}
\setlength{\tabcolsep}{7pt}
\begin{tabular}{@{}lccc@{}}
\toprule
 & \multicolumn{2}{c}{Best last-iterate error at $p=2$} & \\
\cmidrule(lr){2-3}
Guarantee required & \shortstack{Gradient at the\\current iterate} & \shortstack{Gradient at the\\running average} & Thm.\\
\midrule
At a known horizon $T$ & $T^{-1/2}$ & $T^{-1/2}$ & \ref{thm:upper}\\
\midrule
At one fixed time, confidence factor $\delta^{-1/2}$ & \cellcolor{gray!16}$t^{-1/2}$ & $t^{-1/2}$ & \ref{thm:fixed-time-upper}\\
\midrule
At one fixed time, confidence factor $\log\frac1\delta$ & \cellcolor{gray!16}$t^{-2/5}$ & $t^{-1/2}$ & \ref{thm:anytime-upper}, \ref{thm:anytime-marginal}\\
\midrule
At all times (any stopping time) & \cellcolor{gray!16}$t^{-1/3}\to t^{-2/5}$ & $t^{-1/2}$ & \ref{thm:anytime-upper}, \ref{thm:anytime}, \ref{thm:averaging-positive}\\
\bottomrule
\end{tabular}
\caption{Best last-iterate error of clipped methods under finite variance ($p=2$), up to constants and logarithmic factors. Shaded cells mark rates that are new in this paper; the arrow starts from the previous time-uniform rate of \citet{chezhegov2026}, and the known-horizon rate is due to \citet{liu2026lastiterate} and \citet{jiang2026linear}. Down the current-iterate column, knowing less about when the run stops costs a polynomial factor in $t$ or $1/\delta$; the running-average column shows that querying gradients at the running average removes this cost (Theorem~\ref{thm:averaging-positive}). For the guarantees in the bottom two rows, no horizon-free gradient-only rule, with any adaptive step sizes and thresholds, exceeds the exponent $2/5$, and for time-uniform guarantees neither does heavy-ball momentum with geometric forgetting and polynomially bounded step sizes (Theorems~\ref{thm:momentum-obstruction} and~\ref{orth:thm:main}). For $1<p\le2$, replace $1/2$ by $\alpha_p$, $2/5$ by $\beta_p$, $1/3$ by $2(p-1)/(3p)$, and $\delta^{-1/2}$ by $\delta^{-\alpha_p}$.}
\label{tab:summary}
\end{resulttable}

\subsection{Techniques and Organization}\label{sec:techniques}

Our upper bounds for clipped SGD rest on a single potential (Lemma~\ref{lem:potential}) that combines a recursive average of past objective gaps, weighted using inverse tail sums of remaining step sizes following the comparators of \citet{zamani2025} and \citet{liuzhou2024}, with the distance from the current iterate to a matching average of the iterates. When the horizon $T$ is known, these weights place full mass on the final iterate at step $T$, which yields Theorem~\ref{thm:upper} directly \citep{jiang2026linear}. Without a known horizon, we add a positive constant to the step-size tail sums so that the weights remain bounded for all $t$. Smoothness then converts the potential into a last-iterate error bound, and the concentration events for all deterministic horizons can be intersected into one event. We obtain each of these events from a single stopped Freedman inequality (Lemma~\ref{lem:energy-budget}) rather than from the step-by-step induction and union bound that \citet{nguyen2023} identify as the source of extra $\log T$ factors in earlier analyses. To achieve fixed-time exponents $r>\beta_p$, we instead truncate late clipping errors at levels set by the target time and absorb the truncation probability into the polynomial $\delta$-factor. Both this bounded-weight potential and our pre-crossing progress budget for translated objectives (Proposition~\ref{prop:simultaneous-budget}) may be useful in other time-uniform last-iterate analyses.

Appendix~\ref{sec:related} discusses related work, Appendix~\ref{sec:experiment} gives numerical illustrations at $p=2$, Appendix~\ref{app:open} lists open problems, and the remaining appendices contain the proofs. We also provide machine-checked proofs in Lean~4 with Mathlib of Theorems~\ref{thm:anytime-upper}, \ref{thm:anytime}, and~\ref{orth:thm:main}, together with the known-horizon bound of Theorem~\ref{thm:upper} (Section~\ref{sec:discussion}).

%% file: sections/02_setting.tex
\section{Problem Setting}\label{sec:setting}

Throughout, $\norm{\cdot}$ and $\ip{\cdot}{\cdot}$ denote the Euclidean norm and inner product, and all logarithms are natural.\looseness=-1

\begin{assumption}[Objective]\label{as:objective}
The function $f:\R^d\to\R$ is differentiable and convex, and
\begin{equation*}
 \norm{\nabla f(x)-\nabla f(y)}\le L\norm{x-y}
 \quad\text{for all }x,y\in\R^d.
\end{equation*}
There is a minimizer $x^\star$, with $f^\star=f(x^\star)$, and a known bound $\norm{x_1-x^\star}\le R$, where $L,R>0$.\looseness=-1
\end{assumption}

\begin{assumption}[Oracle]\label{as:oracle}
Let $\F_{t-1}$ contain the initial state, any independent algorithmic randomness, and all observations up to time $t-1$, and write $\Et[\cdot]=\E[\cdot\mid\F_{t-1}]$. The query $x_t$ is $\F_{t-1}$-measurable, and the oracle returns one fresh observation $g_t$ with, for some $\sigma\ge0$,
\begin{align}
 \Et[g_t]&=\nabla f(x_t),\nonumber\\
 \Et\bigl[\norm{g_t-\nabla f(x_t)}^p\bigr]&\le\sigma^p,\qquad 1<p\le2.\nonumber
\end{align}
\end{assumption}

Assumption~\ref{as:oracle} covers both i.i.d.\ samples from a fixed oracle with uniform moment bounds and conditionally unbiased martingale difference noise. Bounding only the central $p$th moment permits asymmetric noise and infinite variance when $p<2$, without requiring a uniform bound on $\Et\norm{g_t}^p$; individual sample losses need not be smooth or convex. Following prior analyses \citep{sadiev2023,nguyen2023,chezhegov2026}, our schedules may use the parameters $L,R,\sigma,p$ and, when noted, a confidence parameter $\delta$, whereas the time-uniform schedule of Section~\ref{sec:anytime-upper} requires no $\delta$.

At each step, clipped SGD applies Euclidean clipping,
\begin{equation*}
 \clip(g,\lambda)=
 \begin{cases}
  \min\{1,\lambda/\norm g\}\,g,&g\ne0,\\
  0,&g=0,
 \end{cases}
\end{equation*}
which projects the stochastic gradient onto the ball of radius $\lambda$ centered at zero; the iterates themselves are never projected. After $t$ oracle calls, the algorithm outputs the last iterate $x_{t+1}$ directly, without averaging, checkpoint selection, or restarts.

\emph{Fixed-time and time-uniform guarantees.} What does high-probability convergence mean when the stopping time is not fixed before the run? Let $\varepsilon_t(\delta)$ be a deterministic error envelope. A \emph{fixed-time} guarantee requires that for every $\delta\in(0,1)$ and every deterministic $t\ge1$,
\[
 \Prob\bigl\{f(x_{t+1})-f^\star\le\varepsilon_t(\delta)\bigr\}\ge1-\delta,
\]
where the high-probability event may vary with $t$. Such a bound covers only times chosen before the run begins, while a known-horizon method requires the bound only at its terminal step $t=T$. By contrast, a \emph{time-uniform} guarantee, analogous to a confidence sequence \citep{howard2021}, requires the stronger simultaneous bound
\[
 \Prob\bigl\{f(x_{t+1})-f^\star\le\varepsilon_t(\delta)\ \text{for all}\ t\ge1\bigr\}\ge1-\delta.
\]
Time uniformity is equivalent to validity at every data-dependent stopping time. On the one hand, the simultaneous event implies $f(x_{\tau+1})-f^\star\le\varepsilon_\tau(\delta)$ at any random time $\tau$, such as one chosen by early stopping or checkpoint selection. On the other hand, any guarantee that holds at every bounded stopping time is time-uniform: apply it to $\min\{\tau,N\}$, where $\tau$ is the first time the envelope fails, and let $N\to\infty$. A union-bound conversion from fixed-time to time-uniform guarantees splits the failure budget $\delta$ across steps, and under heavy-tailed noise that cost cannot always be absorbed into logarithmic factors (Section~\ref{sec:marginal}).\looseness=-1

%% file: sections/03_known_horizon.tex
\section{Known Horizon: An Optimal Last-Iterate Rate}\label{sec:upper}

We first review the known-horizon baseline. When the total number of steps $T$ is fixed in advance, clipped SGD with linearly decaying step sizes attains the statistical scale~\eqref{eq:statistical-scale} at its last iterate \citep{liu2026lastiterate,jiang2026linear}; Algorithm~\ref{alg:clipped} is the version of \citet{jiang2026linear}. We then turn to unknown horizons in Section~\ref{sec:anytime-section}.

\begin{algorithm}[t]
\caption{Clipped SGD for a known horizon}\label{alg:clipped}
\begin{algorithmic}[1]
\STATE Input: $x_1$, $L$, $R$, $\sigma$, $p$, $T$, $\delta$
\STATE $\ell\gets\log(8/\delta)$
\STATE $\lambda\gets\max\{4LR,\sigma(T/\ell)^{1/p}\}$
\STATE $\eta\gets R/(128\lambda\ell)$
\FOR{$t=1,\ldots,T$}
\STATE Draw one fresh stochastic gradient $g_t$ at $x_t$
\STATE $\eta_t\gets\eta(T-t+1)/T$
\STATE $x_{t+1}\gets x_t-\eta_t\clip(g_t,\lambda)$
\ENDFOR
\STATE Return $x_{T+1}$
\end{algorithmic}
\end{algorithm}

Algorithm~\ref{alg:clipped} makes $T$ oracle calls using $O(d)$ memory and no knowledge of $x^\star$. It keeps the clipping threshold fixed for the entire run and decays the step size linearly from $\eta$ to $\eta/T$; linear decay is the worst-case optimal schedule in the analysis of \citet{defazio2024}. \citet[Theorem~1]{jiang2026linear} prove that this schedule attains the optimal rate $O(T^{-\alpha_p})$ at its last iterate, with explicit constants; we restate their result below.

\begin{theorem}[Known horizon; \citealt{jiang2026linear}, Theorem~1]\label{thm:upper}
Let Assumptions~\ref{as:objective} and~\ref{as:oracle} hold, and let $T\ge1$ and $0<\delta<1$. With probability at least $1-\delta$, the output of Algorithm~\ref{alg:clipped} satisfies, with $\ell=\log(8/\delta)$,
\begin{equation}\label{eq:upper}
 f(x_{T+1})-f^\star\le\frac{1024LR^2\ell}{T}
       +256\sigma R\left(\frac\ell T\right)^{(p-1)/p}.
\end{equation}
\end{theorem}

\citet[Theorem~2]{liu2026lastiterate} obtain a bound of the same order using a state-dependent clipping threshold (Appendix~\ref{sec:related-last}).

%% file: sections/04_time_uniform.tex
\section{Unknown Horizon: Time-Uniform Guarantees}\label{sec:anytime-section}

When the horizon is unknown, Algorithm~\ref{alg:clipped} cannot tune its threshold or linear decay to $T$. We now ask what last-iterate rate clipped SGD can achieve at all times simultaneously. In Section~\ref{sec:anytime-upper}, we construct a deterministic schedule that attains $\tilde O(t^{-\beta_p})$ uniformly over time without using a horizon or a confidence input. In Section~\ref{sec:anytime}, we prove that no gradient-only rule can attain a larger exponent.

\subsection{A Horizon-Free Upper Bound}\label{sec:anytime-upper}

Define the exponents
\begin{equation}\label{eq:anytime-exponents}
 \begin{gathered}
 \beta=\beta_p=\frac{2(p-1)}{3p-1},\\
 a=\frac{p+1}{3p-1}=1-\beta,\qquad
 b=\frac2{3p-1}.
 \end{gathered}
\end{equation}
Let $Q=\max\{LR,\sigma\}$ and $\ell_t=\log(100t^2)$. We run the recursion~\eqref{eq:recursion} for all $t\ge1$ with
\begin{equation}\label{eq:anytime-upper-schedule}
 \eta_t=\frac{R}{Q}\,\frac{c_p\,t^{-a}}{\ell_t^2},
 \qquad \lambda_t=4Qt^b,
\end{equation}
where $c_p>0$ depends only on $p$. At $p=2$, the step size decays as $\eta_t\propto t^{-3/5}/\ell_t^2$ and the threshold grows as $\lambda_t\propto t^{2/5}$. For the analysis, write $\iota_t(\delta)=\log(100t^2/\delta)$.

\begin{theorem}[Time-uniform upper bound]\label{thm:anytime-upper}
Let Assumptions~\ref{as:objective} and~\ref{as:oracle} hold, and run~\eqref{eq:recursion} with the schedule~\eqref{eq:anytime-upper-schedule}. There is a finite positive constant $c_p$ depending only on $p$, and we may take $D_p=4/c_p$ (see Appendix~\ref{app:anytime-upper-constants}), such that for every $\delta\in(0,1)$, with probability at least $1-\delta$, simultaneously for every integer $T\ge1$,
\begin{align}
 f(x_{T+1})-f^\star&\le D_p\,RQ\,T^{-\beta}\iota_T(\delta)^4,
 \label{eq:anytime-upper-rate}\\
 \norm{x_{T+1}-x^\star}^2&\le R^2[\log(100/\delta)]^2.
 \label{eq:anytime-upper-radius}
\end{align}
\end{theorem}

Because the schedule~\eqref{eq:anytime-upper-schedule} does not depend on $\delta$, the algorithm produces a single trajectory that serves every confidence level simultaneously; only the high-probability event and its bounds vary with $\delta$. Hence the rate also holds almost surely (Appendix~\ref{app:anytime-upper-as}).

Our proof produces conservative numerical constants, giving $c_2<10^{-25}$ and $D_2>10^{25}$ at $p=2$ (Appendix~\ref{app:anytime-upper-constants}). The exponents of the schedule and of the rate do not depend on these constants, however, and the numerical illustration in Appendix~\ref{sec:experiment} uses step-size constants of order one. \citet{chezhegov2026} obtain the exponent $2(p-1)/(3p)$ (equal to $1/3$ at $p=2$) with a $\delta$-dependent schedule; their guarantee is also time-uniform, because their proof controls nested events over all steps. Our improvement to $\beta_p$ comes from the modified potential in Appendix~\ref{sec:analysis}, whose weights stay bounded without a horizon, so that, for each analysis horizon, a single stopped Freedman inequality controls the potential. We next show that $\beta_p$ is optimal.

\subsection{A Matching Lower Bound for Horizon-Free Rules}\label{sec:anytime}

To show that $\beta_p$ cannot be improved, we formalize the class of algorithms that query a stochastic gradient at their current iterate and move along that gradient. We refer to algorithms of this form as \emph{gradient-only rules}; Definition~\ref{def:class} fixes their update direction and oracle access, excluding momentum, extra queries, restarts, output selection, and averaging.

\begin{definition}[Gradient-only rules]\label{def:class}
Let $p\in(1,2]$, $L=R=\sigma=1$, and $x_1=0$. A horizon-free \emph{gradient-only rule} queries one fresh stochastic gradient $g_t$ at the current iterate $x_t$ and updates
\begin{equation*}
 x_{t+1}=x_t-c_t\,g_t ,
\end{equation*}
where the finite step factor $c_t\ge0$ is a measurable function of $t$, $\delta$, $p$, the bounds $L,R,\sigma$, the gradients $g_1,\ldots,g_t$, and independent randomness. Equivalently, $c_t g_t=\eta_t\clip(g_t,\lambda_t)$ for some finite $\eta_t\ge0$ and $\lambda_t>0$ chosen in the same way.
\end{definition}

\begin{theorem}[Time-uniform lower bound]\label{thm:anytime}
Fix $p\in(1,2]$ and $\delta\in(0,1/2)$. For every rule of Definition~\ref{def:class}, every $C>0$, $m\ge0$, and
\begin{equation}\label{eq:beta}
 r>\beta_p=\frac{2(p-1)}{3p-1},
\end{equation}
there is an instance satisfying Assumptions~\ref{as:objective} and~\ref{as:oracle} with $L=R=\sigma=1$ and $x_1=0$ such that
\begin{equation}\label{eq:anytime-failure}
 \Prob\Bigl\{\exists t\ge1:\ f(x_{t+1})-f^\star>
 C t^{-r}[\log(et)]^m\Bigr\}>\delta.
\end{equation}
The rule, $C$, and $m$ may depend arbitrarily on $\delta$ and $p$. The instance is one-dimensional, with i.i.d.\ additive noise and convex, globally $1$-smooth sample losses that attain their minima.
\end{theorem}

Together, Theorems~\ref{thm:anytime-upper} and~\ref{thm:anytime} establish that $\beta_p$ is the optimal time-uniform exponent for gradient-only rules, strictly below the known-horizon exponent $\alpha_p$. Proposition~\ref{prop:record-lower} extends Theorems~\ref{thm:anytime} and~\ref{thm:anytime-marginal} to every horizon-free rule whose iterate never moves past its earlier maximum in one fixed direction on zero-gradient rounds, such as updates with fixed anchors or contracting weight decay toward the initialization, clipped KT coin betting \citep{orabona2016coin}, and the conic-lifting methods of \citet{grandclement2026conic} in one dimension (Appendix~\ref{app:learning-rate-free}). Querying gradients at the running average escapes this barrier (Section~\ref{sec:momentum-boundary}).

\emph{Idea of the proof.} We make the argument sketched in Section~\ref{sec:results} quantitative. The hard instances use the sample losses $F_{\mu,y}(x;Y)=\tfrac12(x-y+\mu)_+^2-AYx$, where $Y\sim\operatorname{Bernoulli}(\mu/A)$ and $A=(2\mu)^{-1/(p-1)}$, so that the expected objective is a translated one-sided Huber function minimized at $y\in[\mu,1]$ and the spikes $Y=1$ arrive with probability $\pi=\mu/A$. Proposition~\ref{prop:simultaneous-budget} turns the first-crossing argument into a progress budget: if a deterministic envelope $b_1,\ldots,b_N$ holds on every instance with probability at least $1-\varepsilon$, then for every $\mu\in(4b_N,1/2]$,
\[
 \pi\sum_{t=1}^N\bigl(\mu+\sqrt{2b_t}\bigr)\ge\frac{1-2\varepsilon}{2}.
\]
The budget holds for arbitrary adaptive jump sizes, and the targets $y$ are fixed before the run rather than chosen after observing the path. Given a claimed exponent $r>\beta_p$, we pick $\beta_p<u<\min\{r,1\}$ and dominate the first $N$ values of the claimed envelope by $b_{N,t}=C\log^m(eN)t^{-u}$. With $\mu=8b_{N,N}$ and $\pi\asymp\mu^{p/(p-1)}$, the left-hand side scales, up to logarithmic factors, as $N^{1-u(p/(p-1)+1/2)}$. Because $u>\beta_p=1/(p/(p-1)+1/2)$, this sum tends to zero as $N\to\infty$, contradicting the lower bound $(1-2\varepsilon)/2$ when $\varepsilon<1/2$. Appendix~\ref{app:anytime} carries out this calculation, which also yields Theorems~\ref{thm:anytime-marginal} and~\ref{thm:momentum-obstruction}.

%% file: sections/05_fixed_time.tex
\section{Unknown Horizon: Fixed-Time Guarantees}\label{sec:marginal}

Section~\ref{sec:anytime-section} showed that no gradient-only rule can exceed the exponent $\beta_p$ uniformly over all times. We now ask what changes if the guarantee is required only at a single deterministic time $T$, while the horizon remains unknown to the algorithm. We show that every exponent in $(\beta_p,\alpha_p]$ becomes attainable, but the cost of not knowing the horizon shifts into a polynomial dependence on $1/\delta$.

\begin{theorem}[Fixed-time confidence--rate tradeoff]\label{thm:fixed-time-upper}
Let Assumptions~\ref{as:objective} and~\ref{as:oracle} hold, and let $Q$, $c_p$, $D_p$, $\ell_t$, and $\iota_T(\delta)$ be as in Theorem~\ref{thm:anytime-upper}. Fix $r\in(\beta_p,\alpha_p]$ and $u\ge0$, and use the horizon-free schedule
\begin{equation}\label{eq:fixed-time-schedule}
 \begin{gathered}
 \eta_t=\frac{R}{Q}\,\frac{c_p\,\delta^u\,t^{-(1-r)}}{\ell_t^2},\\
 \lambda_t=4Q\max\bigl\{1,\delta^{u/(p-1)}t^{r/(p-1)}\bigr\}.
 \end{gathered}
\end{equation}
Put
\begin{gather*}
 c_{\rm in}(r)=\frac{3p-1}{p-1}(r-\beta_p),\\
 c_{\rm free}(r)=\frac{(3p-1)r-2(p-1)}{(2p-1)r-(p-1)}.
\end{gather*}
If $u=c_{\rm in}(r)$, let $c=c_{\rm in}(r)$; if $u=0$, let $c=c_{\rm free}(r)$. Then for every $\delta\in(0,1)$ and every deterministic $T\ge1$, with probability at least $1-\delta$,
\begin{equation}\label{eq:fixed-time-rate}
 f(x_{T+1})-f^\star\le D_p\,RQ\,T^{-r}\delta^{-c}\,\iota_T(\delta)^4.
\end{equation}
\end{theorem}

Clipped SGD with the schedule~\eqref{eq:fixed-time-schedule} is horizon-free and steps along the clipped current stochastic gradient, so it is a gradient-only rule in the sense of Definition~\ref{def:class}. When $u=0$, the schedule uses neither the horizon nor the confidence parameter, so the same trajectory serves every $\delta\in(0,1)$: for each $\delta$, \eqref{eq:fixed-time-rate} holds with probability at least $1-\delta$, although the success event depends on both $\delta$ and $T$. At the statistical exponent $r=\alpha_p$, the confidence factor is $\delta^{-(p-1)/p}$ when the schedule takes $\delta$ as an input and $\delta^{-1}$ when it does not. Both confidence exponents $c_{\rm in}(r)$ and $c_{\rm free}(r)$ approach zero as $r\downarrow\beta_p$, matching the polylogarithmic confidence dependence of Theorem~\ref{thm:anytime-upper}.

Can any horizon-free gradient-only rule attain a fixed-time exponent above $\beta_p$ with only polylogarithmic dependence on $1/\delta$? We prove that it cannot: the progress budget of Proposition~\ref{prop:simultaneous-budget} constrains fixed-time guarantees as well, via a union bound over finite prefixes.\looseness=-1

\begin{theorem}[Fixed-time lower bound with polynomial confidence cost]\label{thm:anytime-marginal}
Let $(\mathrm{rule}_\delta)_{\delta\in(0,1)}$ be any family of rules of Definition~\ref{def:class}, and let $C>0$, $m\ge0$, $c\in[0,1-\beta_p)$, and $r>\beta_p+c$. There exist $\delta\in(0,1)$, a deterministic integer $t\ge1$, and an instance in the class of Theorem~\ref{thm:anytime} such that $\mathrm{rule}_\delta$ satisfies
\begin{equation}\label{eq:anytime-marginal-failure}
 \Prob\Bigl\{f(x_{t+1})-f^\star>Ct^{-r}\delta^{-c}[\log(et/\delta)]^m\Bigr\}>\delta.
\end{equation}
Here $C$, $c$, and $m$ do not depend on $\delta$, $t$, or the instance, while each rule may depend arbitrarily on $\delta$, the history, independent randomness, and the current stochastic gradient.
\end{theorem}

Setting $c=0$ in Theorem~\ref{thm:anytime-marginal} shows that gradient-only rules cannot attain any fixed-time exponent above $\beta_p$ with uniformly polylogarithmic dependence on $1/\delta$. Appendix~\ref{app:anytime-marginal} gives the proof.

For each $r\in(\beta_p,\alpha_p]$, let $c^\star(r)$ be the infimum of exponents $c$ for which some family of gradient-only rules achieves a fixed-time bound $Ct^{-r}\delta^{-c}[\log(et/\delta)]^m$ across all instances. Because $\iota_T(\delta)\le5\log(eT/\delta)$, combining Theorems~\ref{thm:fixed-time-upper} and~\ref{thm:anytime-marginal} sandwiches $c^\star(r)$ between two linear functions of $r-\beta_p$:
\begin{equation}\label{eq:confidence-tradeoff}
 r-\beta_p\le c^\star(r)\le\frac{3p-1}{p-1}(r-\beta_p).
\end{equation}
The exact value of $c^\star(r)$ remains open (Appendix~\ref{app:open}).

%% file: sections/06_momentum.tex
\section{Memory and the Momentum Boundary}
\label{sec:momentum-boundary}

Gradient-only rules cannot surpass the time-uniform exponent $\beta_p$ without a known horizon (Theorem~\ref{thm:anytime}). Can momentum break this barrier? For clipped heavy-ball momentum with polynomially bounded step sizes, we prove that momentum helps only if its memory grows over time. To establish this boundary, we compare the running-average update of Section~\ref{sec:running-average} with two classes of clipped heavy-ball rules. Each method we study queries one stochastic gradient at its current iterate and outputs the updated state.

\subsection{Forgetting in One Dimension}

We first consider the stochastic heavy-ball recursion \citep{polyak1964} in velocity form:
\begin{equation}
 \begin{gathered}
 v_{t+1}=\beta_t v_t+\clip(g_t,\lambda_t),\\
 x_{t+1}=x_t-\eta_t v_{t+1},\qquad x_1=v_1=0,
 \end{gathered}
 \label{eq:momentum-main-recursion}
\end{equation}
where $\eta_t$ is the step size and $\beta_t\in[0,1)$ is the momentum coefficient. We allow a horizon-free rule to choose its scalar controls $\beta_t$, $\eta_t$, and $\lambda_t$ measurably from the entire past history, an independent random seed, and the current raw gradient $g_t$, provided that on every trajectory they satisfy
\begin{equation}
 \begin{gathered}
 0\le\beta_t<1,\quad 1-\beta_t\ge\rho(1+t)^{-a}\ (t\ge t_0),\\
 0\le\eta_t\le H(1+t)^h,\quad 0\le\lambda_t<\infty,
 \end{gathered}
 \label{eq:momentum-bounds}
\end{equation}
for deterministic constants $a\in[0,1)$, $\rho\in(0,1]$, $t_0\in\mathbb N$, and $H,h\ge0$ that depend only on the rule and its confidence input $\delta$, not on the instance, the samples, or the random seed. Geometric forgetting, $\beta_t\le\bar\beta<1$, corresponds to the case $a=0$ with $\rho=1-\bar\beta$ and $t_0=1$.

\begin{theorem}[One-dimensional momentum lower bound]
\label{thm:momentum-obstruction}
Fix $p\in(1,2]$, $a\in[0,1)$, and a confidence $0<\delta<1/2$, and let a horizon-free rule satisfy
\eqref{eq:momentum-main-recursion} and~\eqref{eq:momentum-bounds}. For every finite
$C>0$, $m\ge0$, and $r>(1+a)\beta_p$, there is a fixed
one-dimensional instance with $L=R=\sigma=1$ such that
\begin{equation}
 \Prob\!\left\{\exists t\ge1:
 f(x_{t+1})-f^\star>Ct^{-r}\log^m(et)\right\}>\delta.
 \label{eq:momentum-obstruction}
\end{equation}
The instance satisfies Assumptions~\ref{as:objective} and~\ref{as:oracle} with i.i.d.\ additive noise and convex, globally $1$-smooth sample losses that attain their minima, and it is fixed before the run.
\end{theorem}

Under geometric forgetting ($a=0$), no rule satisfying~\eqref{eq:momentum-main-recursion} and~\eqref{eq:momentum-bounds} attains a time-uniform exponent above $\beta_p$, matching the barrier for gradient-only rules in Theorem~\ref{thm:anytime}. In particular, Theorem~\ref{thm:momentum-obstruction} covers any fixed coefficient $\beta\in[0,1)$ paired with arbitrary adaptive step sizes bounded by a constant. Even under slower polynomial forgetting with $0\le a<(p-1)/(2p)$, we have $(1+a)\beta_p<\alpha_p$, so the statistical exponent $\alpha_p$ remains unattainable. To prove Theorem~\ref{thm:momentum-obstruction}, we bound the progress that rare gradient spikes can produce after passing through the decaying momentum kernel (Appendix~\ref{app:momentum}).\looseness=-1

\subsection{Polynomial Forgetting with Constants Independent of Dimension}
\label{sec:momentum-dimension}

Why does the one-dimensional lower bound in Theorem~\ref{thm:momentum-obstruction} stop at $(1+a)\beta_p$ rather than $\beta_p$? Under polynomial forgetting, the momentum kernel needs $\tilde O(N^a)$ steps to decay on a horizon $N$, so in the one-dimensional budget argument each rare gradient spike can move the iterate for that many steps. If the dimension may grow and the control and rate constants are shared across dimensions, this extra factor disappears and the exponent barrier drops to $\max\{\beta_p,a\}$.

To state this bound, fix $p\in(1,2]$ and $a\in[0,1)$. The class $\mathcal A_{\mathrm{unif}}(a)$ consists of all rules defined in every finite dimension that run~\eqref{eq:momentum-main-recursion} from $x_1=0$ and $v_1=0$, observe only stochastic gradients, and choose their controls as measurable functions of the dimension, time, the confidence level $\delta$, the bounds $L=R=\sigma=1$, the gradients $g_1,\ldots,g_t$, and an independent seed. For each rule in $\mathcal A_{\mathrm{unif}}(a)$, there must exist deterministic constants $\rho\in(0,1]$, $t_0\in\mathbb N$, and $H,h\ge0$ such that in every dimension and on every history,
\begin{equation}\label{eq:momentum-dimension-envelope}
 \begin{gathered}
 0\le\beta_t<1,\quad 0\le\lambda_t<\infty,\quad
 0\le\eta_t\le H(1+t)^h,\\
 1-\beta_t\ge\rho(1+t)^{-a}\quad(t\ge t_0).
 \end{gathered}
\end{equation}
These constants may depend on the rule, $p$, $a$, and $\delta$, but not on the dimension, the instance, or the seed.

\begin{theorem}[Dimension-uniform lower bound under polynomial forgetting]
\label{orth:thm:main}
Let $C>0$ and $k\ge0$ be finite, let $r>\max\{\beta_p,a\}$, and fix a confidence $\delta\in(0,1/2)$. For any rule in $\mathcal A_{\mathrm{unif}}(a)$, there is a finite-dimensional
instance with $L=R=\sigma=1$ for which
\begin{equation}\label{eq:momentum-dimension-failure}
 \Prob\{\exists t\ge1:\ f(x_{t+1})-f^\star>
                   Ct^{-r}\log^k(et)\}>\delta.
\end{equation}
Failure occurs on a deterministic finite prefix.
The instance satisfies Assumptions~\ref{as:objective} and~\ref{as:oracle} with i.i.d.\ additive noise and convex, globally $1$-smooth sample losses that attain their minima, and it is fixed before the samples and the seed are drawn; $C$ and $k$ are shared across dimensions.
\end{theorem}

For $0\le a\le\beta_p$, the optimal dimension-uniform exponent in $\mathcal A_{\mathrm{unif}}(a)$ is therefore $\beta_p$: our horizon-free schedule in Theorem~\ref{thm:anytime-upper} attains this exponent with $\beta_t=0$ and bounded step sizes, which satisfy~\eqref{eq:momentum-dimension-envelope} for every such $a$. When $0<a<\alpha_p$, the barrier $\max\{\beta_p,a\}$ lies strictly below $\min\{(1+a)\beta_p,\alpha_p\}$. Theorem~\ref{orth:thm:main} is proved in Appendix~\ref{app:momentum-dimension}. Propositions~\ref{prop:momentum-fixed-time} and~\ref{orth:prop:fixed-time} in the appendix give fixed-time versions of Theorems~\ref{thm:momentum-obstruction} and~\ref{orth:thm:main} when $\rho$, $t_0$, and $h$ do not depend on $\delta$ and $H$ grows at most polynomially in $1/\delta$.\looseness=-1

\subsection{Querying Gradients at the Running Average}\label{sec:running-average}

Both lower bounds above require $1-\beta_t\ge\rho(1+t)^{-a}$ for some $a<1$. When the memory grows faster, so that $1-\beta_t$ is of order $1/t$, a clipped method that queries gradients at the running average of a base sequence attains the statistical exponent $\alpha_p$ at all times simultaneously (Theorem~\ref{thm:averaging-positive}).

\begin{theorem}[Time-uniform averaging momentum]
\label{thm:averaging-positive}
Let Assumptions~\ref{as:objective} and~\ref{as:oracle} hold.
For $\delta\in(0,1)$, define deterministic horizon-free parameters
\begin{equation}
 \begin{gathered}
 c_0=\frac1{128},\quad h_\delta=\log\frac4\delta+\frac{4p}{p-1},\quad
 \bar\ell_t=h_\delta+\log t,\\
 \lambda_t=\max\{4LR,\sigma t^{1/p}\},\quad
 \gamma_t=\frac{c_0R}{\lambda_t\bar\ell_t^2}.
 \end{gathered}
 \label{eq:averaging-schedule}
\end{equation}
Starting from $z_1=x_1$, make one oracle call at $x_t$ and update
\begin{equation}
 \begin{gathered}
 d_t=\operatorname{clip}(g_t,\lambda_t),\qquad
 z_{t+1}=z_t-\gamma_t d_t,\\
 x_{t+1}=\frac{t}{t+1}x_t+\frac1{t+1}z_{t+1}.
 \end{gathered}
 \label{eq:averaging-algorithm}
\end{equation}
With probability at least $1-\delta$, simultaneously for every integer
$T\ge1$, one has $\|z_{T+1}-x^\star\|<\sqrt2R$ and
\begin{align}
 f(x_{T+1})-f^\star
 &\le\frac{128R\lambda_T\bar\ell_T^2+2LR^2}{T+1}
 \label{eq:averaging-exact-rate}\\
 &\le514\bar\ell_T^2\left(\frac{LR^2}{T}
                   +\frac{\sigma R}{T^{(p-1)/p}}\right).
 \label{eq:averaging-rate}
\end{align}
The point $x_{T+1}$ is the actual updated state after exactly $T$ calls.
\end{theorem}

The update~\eqref{eq:averaging-algorithm} adapts the anytime online-to-batch conversion of \citet[Algorithm~1 and Theorem~1]{cutkosky2019anytime} to clipped stochastic gradients on an unbounded domain. This conversion builds on classical online-to-batch conversion \citep{cesabianchi2004}, and \citet[Theorem~1]{defazio2021primal} shows that updates of this form are equivalent to heavy-ball momentum. Unlike the schedule in Theorem~\ref{thm:anytime-upper}, the schedule~\eqref{eq:averaging-schedule} takes $\delta$ as an input, though it remains independent of the horizon. The trajectory of~\eqref{eq:averaging-algorithm} can be written in the heavy-ball velocity form~\eqref{eq:momentum-main-recursion} with bounded step sizes and momentum coefficients satisfying $\alpha_p/(2t)\le1-\beta_t\le1/t$ for all $t\ge2$ (Proposition~\ref{prop:averaging-momentum}), so this representation violates the forgetting condition in~\eqref{eq:momentum-bounds} for every $a<1$. The same trajectory can also be realized with a constant momentum coefficient $\beta_t=b\in(0,1)$ and exponentially growing step sizes (Proposition~\ref{prop:averaging-unbounded-gain}), so Theorem~\ref{thm:momentum-obstruction} needs some restriction on the growth of the step sizes. Unlike the three-sequence accelerated method of \citet{nguyen2023}, which also attains $\alpha_p$ without a known horizon (Appendix~\ref{sec:averaging-positive}), the update~\eqref{eq:averaging-algorithm} stays within the two-sequence heavy-ball form~\eqref{eq:momentum-main-recursion} used in our momentum lower bounds.

%% file: sections/07_discussion.tex
\section{Discussion}\label{sec:discussion}

\emph{Implications.} Knowing the horizon $T$ in advance allows an algorithm to tune its step sizes and clipping threshold to $T$, whereas a guarantee that covers adaptive checkpoint selection or early stopping must be time-uniform (Section~\ref{sec:setting}). Under heavy-tailed noise, our results show that an unknown horizon has a polynomial cost for every rule whose iterate does not advance on rounds without new gradient information, including clipped SGD with any adaptive step sizes and thresholds: for these rules, the optimal exponent of a time-uniform last-iterate guarantee drops from $\alpha_p$ to $\beta_p$ (Theorems~\ref{thm:anytime-upper} and~\ref{thm:anytime} and Proposition~\ref{prop:record-lower}). Recovering $\alpha_p$ without a known horizon therefore requires a rule outside this class, such as the running-average update of Theorem~\ref{thm:averaging-positive}, whose iterate $x_t$ keeps moving toward the sequence $z_t$ of~\eqref{eq:averaging-algorithm} when the gradient is zero. Learning-rate-free methods do not avoid this cost either. Proposition~\ref{prop:record-lower} covers, in one dimension, clipped KT coin betting and the conic-lifting methods of \citet{grandclement2026conic}, and for their zero-initialized conic Blackwell algorithms CBA and CBA$^+$ we prove more: the last iterate admits no time-uniform bound that tends to zero, even under i.i.d.\ additive noise (Proposition~\ref{prop:cba-no-rate}).

\emph{Formalization.} We provide machine-checked proofs in Lean~4 with Mathlib of Theorems~\ref{thm:anytime-upper}, \ref{thm:anytime}, and~\ref{orth:thm:main} and of the known-horizon bound of Theorem~\ref{thm:upper}, including their proof dependencies. \leanavailability{} The formalized arguments include the recursive potential and the two martingale tail bounds, the integration of the probabilities of failing at the first crossing over a continuum of targets, and the orthogonal-projection argument behind the dimension-uniform momentum bound.

%% file: appendices/related.tex
\section{Related Work}\label{sec:related}

\subsection{Heavy-Tailed Noise and Gradient Clipping}

Under the classical model with bounded $p$th moments \citep{nemirovski1983,vural2022}, truncated mirror descent \citep{nazin2019}, smooth clipped methods \citep{gorbunov2020,sadiev2023,nguyen2023}, and nonsmooth methods \citep{parletta2024,parletta2025} achieve high-probability statistical rates, typically for averaged or auxiliary outputs. \citet{nekrasova2026} relax the moment bound to a weak-$L^p$ tail condition, with a noise scale that may grow with the gradient norm, and obtain high-probability rates for clipped SGD at its averaged iterate and for its accelerated variant at its last iterate, with clipping levels tuned to the horizon. Clipping is not always beneficial for averages: \citet{suvorikova2026clipped} compare clipped and unclipped Polyak--Ruppert averaged SGD \citep{ruppert1988,polyakjuditsky1992} under finite $p$th moments with $p\ge2$ and show that, for $p>2$, clipping need not improve the leading term. Directional moments sharpen rates for structured instances \citep{zhang2026,liu2026refined}.

In decentralized smooth nonconvex optimization,
\citet{armacki2026decentralized} analyze clipped DSGD under bounded
central $p$th noise moments and establish stationarity bounds in high probability and
in expectation, with linear speedup in the leading
term. They bound averages over time of squared gradient norms
across agents; we bound convex objective error at the
last iterate and distinguish fixed-time from time-uniform behavior.\looseness=-1

\subsection{Last-Iterate Convergence}\label{sec:related-last}

Schedules for a known horizon remove the logarithmic last-iterate factor of standard nonsmooth SGD \citep{shamir2013,harvey2019,jain2021,zamani2025}; for one-dimensional Lipschitz objectives with i.i.d.\ noise, constant step sizes tuned to the horizon already suffice, which fails without the i.i.d.\ assumption \citep{beretta2026}. Smooth analyses based on moving comparators \citep{liuzhou2024} and online-to-batch regret reductions \citep{defazio2024} require stronger tail assumptions for high-probability guarantees. The bound in expectation under finite $p$th moments of \citet[Section~5]{liuzhou2026extended} uses a $p/(p-1)$-uniformly convex mirror map, excluding Euclidean SGD at $p<2$. \citet{parletta2025} treat the last iterate of clipped SGD on Lipschitz objectives in expectation. For AdaGrad \citep{duchi2011} and AdaGrad-Norm \citep{streetermcmahan2010,ward2020} with fixed parameters, \citet[Theorem~3.1]{ouwang2026adagrad} show that the expected last-iterate error can decay arbitrarily slowly under an oracle whose noise depends on time and history. In our smooth setting, \citet[Theorem~1 and Corollary~1]{chezhegov2026} obtain the horizon-free exponent $2(p-1)/(3p)$; for clipped SGD with an unknown horizon, we establish $\beta_p$ for time-uniform guarantees and $\alpha_p$ for fixed-time guarantees with polynomial dependence on $1/\delta$.

With a known horizon, recent work attains $\alpha_p$ at the last iterate. \citet{liu2026lastiterate} treats Lipschitz and smooth convex objectives, strongly star-convex objectives, and composite and constrained problems under a fine-grained heavy-tailed noise model; Theorem~2 there covers smooth convex objectives with a state-dependent threshold and a refined noise dependence. Theorem~1 of \citet{jiang2026linear} is the bound we restate as Theorem~\ref{thm:upper}; it avoids any $\log T$ factor and carries constants independent of $p$ and the dimension. They also treat composite objectives with proximal updates and oracles under adversarial sample corruption, and establish two-point statistical lower bounds \citep{lecam1973,tsybakov2009}. Both papers assume a known horizon; \citet[Remark~4]{liu2026lastiterate} does so to avoid a potential logarithmic dependence on~$T$. For gradient-only rules (Definition~\ref{def:class}) with polylogarithmic dependence on $1/\delta$, Theorems~\ref{thm:anytime-upper}, \ref{thm:anytime}, and~\ref{thm:anytime-marginal} show that when the horizon is unknown the optimal exponent drops from~$\alpha_p$ to~$\beta_p$. \citet[Section~5]{liu2026lastiterate} lists matching lower bounds for smooth objectives as open; \citet[Proposition~1]{jiang2026linear} matches the stochastic term when the stochastic error dominates.

The proof of \citet{jiang2026linear} combines the tail-sum potential (Appendix~\ref{sec:analysis}) with a stopping-time localization of the iterates. When the stochastic term dominates, setting $\lambda=\sigma(T/\ell)^{1/p}$ balances the cumulative clipping bias $T\sigma^p\lambda^{1-p}$ against the deviation scale $\lambda\ell$, while each update moves the iterate by at most $R/(128\ell)$; the final state $x_{T+1}$ then attains the statistical scale~\eqref{eq:statistical-scale}, which is optimal up to constants in this regime \citep[Proposition~1]{jiang2026linear}. The deterministic term $LR^2\ell/T$ matches unaccelerated gradient descent, whereas acceleration attains $LR^2/T^2$ without noise \citep{nesterov1983,lan2012}.

\subsection{Unknown Horizons and Time-Uniform Guarantees}

For the subgradient method, no single sequence of step sizes attains the optimal last-iterate rate at every iteration \citep{zamani2025}. Deterministic anytime lower bounds use flat objectives and overshoots
\citep{kornowski2024,tsai2026,kornowski2026}. Our stochastic lower bound
additionally handles adaptive responses to a shared i.i.d.\ transcript.
Under conditional norm-sub-Gaussian noise,
\citet[Theorems~2.3, 3.4, and~4.1]{lichenwang2026} characterize exactly which
overhead factors over the rate $t^{-1/2}$ are attainable at all times; by their Corollary~2.4 and the example after it, $(\log t)^{(1+\varepsilon)/2}$ is attainable for every $\varepsilon>0$, whereas $(\log t)^{1/2}$ is not. Finite second moments already
permit the polynomial separation studied here. 

\emph{Almost sure convergence.}
Under the ABC condition on conditional second moments \citep{khaled2023},
\citet[Theorem~5]{liuyuan2022} and \citet[Theorem~13]{liuyuan2024}
obtain almost sure last-iterate rates $O(t^{-1/3+\varepsilon})$ on smooth convex objectives
for unclipped SGD, heavy-ball \citep{polyak1964}, and Nesterov accelerated methods with deterministic
step sizes of order $t^{-2/3-\varepsilon}$, $0<\varepsilon<1/3$.
\citet[Theorems~8 and~13]{sebbouh2021} obtain
$o(1/\sum_{j<k}\eta_j)$ for weighted SGD averages and moving-average
heavy-ball iterates under smooth convex sample losses, finite residual
variance at the minimizer, and suitable step sizes and averaging weights.
Concurrent and independent work by \citet[Theorem~2.2]{anonymous2026cuberoot}
shows that, for every positive deterministic schedule of unclipped step sizes and every
$r>1/3$, a one-dimensional smooth convex instance with a fixed i.i.d.\
finite-variance oracle satisfies
$\limsup_{t\to\infty}(t+1)^r(f(x_t)-f^\star)=\infty$ almost surely.
Their Theorem~2.1 and Corollary~2.4 identify $1/3$ as the supremal
almost sure exponent under the ABC condition, and Proposition~2.5 establishes a
matching rate up to logarithmic factors.
Remark~\ref{rem:cuberoot} adapts their lower bound to the
time-uniform setting at a fixed confidence level $\delta$ and compares it with clipping.\looseness=-1

Anytime online-to-batch conversion \citep{cutkosky2019anytime}, primal averaging
\citep{defazio2021primal,defazio2024roadless}, and accelerated clipped methods
\citep{nguyen2023} query coupled or averaged states. Schedule-free SGD \citep{defazio2024roadless} evaluates gradients between a running average and a base sequence. When it evaluates them at the average, its update with clipped gradients has the form~\eqref{eq:averaging-algorithm}, except that its main form uses a constant step size, whereas $\gamma_t$ in~\eqref{eq:averaging-algorithm} decreases. In overparameterized linear regression, \citet{meterez2026} show that horizon-free, polynomially decaying step sizes combined with weight averaging attain the minimax expected rate of SGD up to logarithmic factors under source and capacity conditions, and in language-model pretraining they find such schedules with averaging comparable to a tuned cosine schedule. Section~\ref{sec:momentum-boundary}
compares such memory with geometric forgetting. The upper proofs use
the maximal concentration inequalities in Appendix~\ref{app:concentration};
for broader time-uniform martingale methods, see \citet{howard2021}.

Momentum lower bounds on quadratics also appear in
\citet{ganesh2023}, for constant parameters and expected error.
\citet[Theorem~2]{drori2020} use orthogonal noise for nonconvex
stationarity, with coordinates indexed by time and step sizes and weights that may depend
on the norms and inner products of all past iterates and stochastic gradients. Our momentum lower bound
uses a fixed i.i.d.\ coordinate distribution and time-uniform convex objective error;
its scalar controls may depend arbitrarily on signs, coordinates, and
an independent seed, with constants shared across dimensions.

%% file: appendices/experiment.tex
\section{A Numerical Illustration}\label{sec:experiment}

We illustrate the separation between $\alpha_p$ and $\beta_p$ numerically. Figure~\ref{fig:illustration} compares the methods on the hard instances of Appendix~\ref{app:lower-witnesses} at $p=2$. For each horizon $N$ we use the targets $y\in\{1/4,1/2,3/4,1\}$, the slopes $\mu=\min\{0.2,1.3N^{-\alpha_p}\}$ and $\mu=\min\{0.2,1.3N^{-\beta_p}\}$, and the one-sided spike noise of~\eqref{eq:rare-progress-parameters}; every method sees the same instances and the same spikes. Every method starts at $x_1=0$, with $z_1=0$ for averaging and zero initial velocity for heavy-ball momentum, and $N$ ranges over $500\cdot2^k$ for $k=0,\dots,6$. Clipped SGD uses the time-uniform schedule $\eta_t=t^{-(p+1)/(3p-1)}$, $\lambda_t=2t^{2/(3p-1)}$; heavy-ball momentum uses $\beta_t=0.8$, the same thresholds, and step sizes $0.2\,t^{-(p+1)/(3p-1)}$; averaging uses~\eqref{eq:averaging-algorithm} with $\gamma_t=0.5\,t^{-1/p}$ and $\lambda_t=2t^{1/p}$; and Algorithm~\ref{alg:clipped} uses $\lambda=2N^{1/p}$ and $\eta_t=1.5(N-t+1)/(N\lambda)$. These constants are of order one and do not change with $N$; the theorems use smaller step sizes, different thresholds, and logarithmic factors. \experimentavailability{} For each method we take, over $1000$ runs, the $0.9$-quantile of the largest gap $f(x_{t+1})-f^\star$ over $N/2\le t\le N$, or of the terminal gap for Algorithm~\ref{alg:clipped}, and report its maximum over the eight instances.\looseness=-1

\begin{figure}[t]
\centering
\begin{tikzpicture}
\begin{loglogaxis}[width=0.66\textwidth,height=0.42\textwidth,
  xlabel={horizon $N$},ylabel={worst $0.9$-quantile of the gap},
  xmin=400,xmax=40000,ymin=4e-3,ymax=1e-1,
  legend cell align=left,legend columns=2,
  legend style={font=\footnotesize,draw=none,fill=none,at={(0.5,-0.2)},anchor=north,/tikz/every even column/.append style={column sep=1em}},
  tick label style={font=\footnotesize},label style={font=\footnotesize}]
\addplot[thick,mark=*,mark size=1.6pt] coordinates
 {(500,5.008e-2) (1000,3.574e-2) (2000,2.521e-2) (4000,1.760e-2) (8000,1.266e-2) (16000,8.774e-3) (32000,6.424e-3)};
\addlegendentry{known horizon, slope $-0.50$}
\addplot[thick,dashed,mark=o,mark size=2pt] coordinates
 {(500,7.038e-2) (1000,5.004e-2) (2000,3.660e-2) (4000,2.573e-2) (8000,1.695e-2) (16000,1.104e-2) (32000,7.241e-3)};
\addlegendentry{averaging, slope $-0.55$}
\addplot[thick,gray,mark=triangle*,mark size=2pt] coordinates
 {(500,7.614e-2) (1000,6.488e-2) (2000,5.039e-2) (4000,3.964e-2) (8000,3.001e-2) (16000,2.207e-2) (32000,1.952e-2)};
\addlegendentry{clipped SGD, slope $-0.35$}
\addplot[thick,gray,dotted,mark=diamond,mark size=2.2pt] coordinates
 {(500,5.645e-2) (1000,4.905e-2) (2000,3.968e-2) (4000,3.223e-2) (8000,2.607e-2) (16000,1.954e-2) (32000,1.783e-2)};
\addlegendentry{heavy ball, slope $-0.30$}
\end{loglogaxis}
\end{tikzpicture}
\caption{Algorithm~\ref{alg:clipped} for a known horizon, averaging (Theorem~\ref{thm:averaging-positive}), time-uniform clipped SGD (Theorem~\ref{thm:anytime-upper}), and heavy-ball momentum with $\beta_t=0.8$, with the exponents of the theorems and illustrative constants, on the hard instances of Appendix~\ref{app:lower-witnesses} at $p=2$: worst case over eight instances of the $0.9$-quantile, over $1000$ runs, of the largest gap over $N/2\le t\le N$ (terminal gap for Algorithm~\ref{alg:clipped}). Slopes are least-squares fits in $\log N$; their $95\%$ bootstrap intervals over the runs (200 resamples) have half-width below $0.02$.}
\label{fig:illustration}
\end{figure}
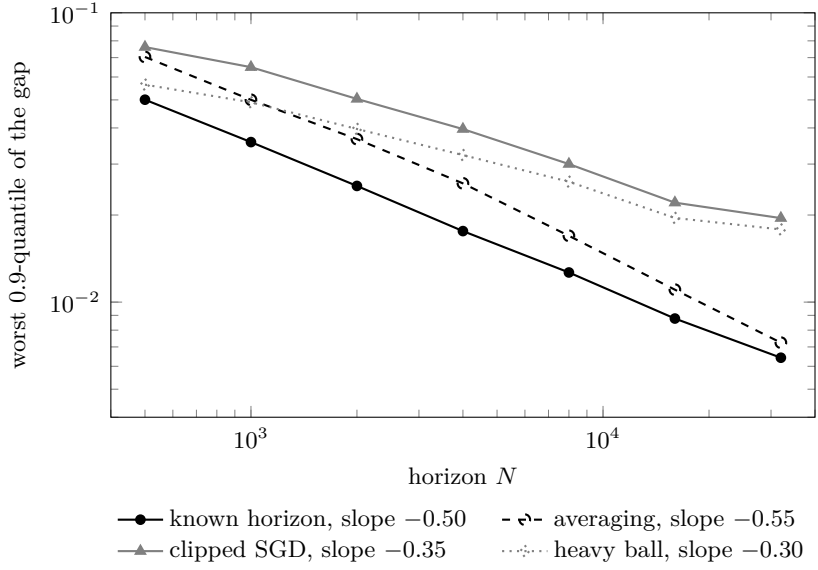

Algorithm~\ref{alg:clipped} and averaging decay with fitted slopes $-0.50$ and $-0.55$, close to $-\alpha_2=-1/2$, whereas clipped SGD and heavy-ball momentum decay with slopes $-0.35$ and $-0.30$, and at $N=3.2\times10^4$ their worst gaps are $2.5$ to $3$ times those of the first two. Slopes fitted over this range of $N$ also reflect constants and logarithmic factors. At $p=1.5$ the experiment does not separate averaging from the other methods. There the exponents $\alpha_p=1/3$ and $\beta_p=2/7$ differ by less than $0.05$. The fitted slopes are $-0.31$ for Algorithm~\ref{alg:clipped} and for averaging and $-0.24$ and $-0.22$ for the other two methods, but at every $N$ in this range the worst gap of averaging stays above those of clipped SGD and heavy-ball momentum; only the worst gap of Algorithm~\ref{alg:clipped} falls below them, at $N=3.2\times10^4$.\looseness=-1

%% file: appendices/open_problems.tex
\section{Open Problems}\label{app:open}
Our schedules use $L,R,\sigma,p$ and, except for the time-uniform schedule and the schedule of Theorem~\ref{thm:fixed-time-upper} with $u=0$, a confidence input; adaptation to unknown problem parameters is a natural next step. The matching results concern polynomial exponents, so sharp logarithmic factors and constants, parameter interpolation, and the optimal confidence exponent $c^\star(r)$ of fixed-time guarantees, bracketed in~\eqref{eq:confidence-tradeoff}, are left for future work. At $r=\alpha_p$, the upper end of this bracket equals $\alpha_p$, whereas the lower end equals the gap between the statistical and time-uniform exponents,
\begin{equation*}
 \alpha_p-\beta_p=\frac{(p-1)^2}{p(3p-1)}>0,
\end{equation*}
which is $1/10$ at $p=2$. The two ends of this bracket come from different proof mechanisms: the upper bound truncates late clipping errors at levels set by the analysis horizon and bounds the truncation probability via the $p$th moment, whereas the lower bound applies a union bound over the first $N$ steps without using the fact that iterates on our hard instance move only at rare gradient spikes before crossing. The two momentum lower bounds concern updates with zero initial velocity, queries and outputs at the current state, gradients as the only oracle information, and polynomially bounded nonnegative step sizes; the dimension-uniform theorem also uses control and rate constants shared across dimensions. Within this class, three questions remain: how the gap between the one-dimensional barrier $(1+a)\beta_p$ and the dimension-uniform barrier $\max\{\beta_p,a\}$ closes, which rates are attainable for $a>\beta_p$, also under a two-sided condition $1-\beta_t\asymp t^{-a}$, and how the confidence factor of a fixed-time exponent above $\max\{\beta_p,a\}$ grows under shared constants. Lower bounds for general update directions are also open. Concurrent analyses for a known horizon cover Lipschitz, strongly star-convex, composite, and constrained problems \citep{liu2026lastiterate,jiang2026linear}; horizon-free counterparts of these settings, and of non-Euclidean or $(L_0,L_1)$-smooth objectives \citep{elkadi2026}, are further directions. For nonconvex stationarity, the bound on the average squared gradient norm over the trajectory in \citet[Theorem~B.2]{nguyen2023} retains the exponent for a known horizon up to logarithmic factors; obtaining a guarantee for a single selected iterate requires a separate selection argument, for example stationarity tests whose oracle calls are counted in the budget, combined with the doubling trick \citep{cesabianchi2006}.\looseness=-1

%% file: appendices/appendix.tex
\Needspace{7\baselineskip}
\section{Deterministic Last-Iterate Calculations}\label{app:deterministic}
This appendix proves the one-step inequality and the recursive potential (Lemma~\ref{lem:potential}).

\subsection{One Inexact Gradient Step}
\begin{lemma}[Inexact smooth convex step]\label{lem:one-step}
Let $f$ be convex and $L$-smooth. Fix $x,z,e\in\R^d$, $0<\gamma\le1/(2L)$, and $y=x-\gamma(\nabla f(x)+e)$. Then
\begin{equation}\label{eq:one-step}
 \gamma[f(y)-f(z)]+\frac12\norm{y-z}^2
 \le\frac12\norm{x-z}^2+\gamma\ip{e}{z-x}+\gamma^2\norm e^2.
\end{equation}
\end{lemma}
\begin{proof}
Smoothness, convexity and the three-point norm identity give
\[
 \gamma[f(y)-f(z)]+\tfrac12\|y-z\|^2
 \le\tfrac12\|x-z\|^2+\gamma\langle e,z-x\rangle
 -\tfrac{1-\gamma L}{2}\|x-y\|^2+\gamma\langle e,x-y\rangle.
\]
Use $\gamma\langle e,x-y\rangle\le\gamma^2\|e\|^2+\|x-y\|^2/4$
and $\gamma L\le1/2$.
\end{proof}

\subsection{A Recursively Averaged Potential}\label{sec:analysis}
Potential functions are a standard tool for analyzing first-order methods \citep{taylor2019}.
Following the comparator weighted by inverse tail sums in \citet{zamani2025} and \citet{liuzhou2024}, we express the potential recursively. With boundary mass $\zeta=0$ below, it is the tail-sum potential that \citet[Lemma~4]{jiang2026linear} use to prove Theorem~\ref{thm:upper}; the boundary mass $\zeta>0$, which keeps the weights bounded without a horizon, is what our unknown-horizon results add.
For any positive steps $\eta_t\le1/(2L)$ and arbitrary
realized errors, consider
\begin{equation*}
 x_{t+1}=x_t-\eta_t(\nabla f(x_t)+e_t).
\end{equation*}
For a deterministic analysis horizon $T$ and boundary mass $\zeta\ge0$, set
\[
 G_t=f(x_t)-f^\star,\quad S_t=\sum_{r=t}^T\eta_r,\quad
 U_t=\zeta+S_t,\quad W=U_1,\quad\theta_t=\eta_t/U_t,
\]
with $S_0=S_1$, $S_{T+1}=0$ and $U_0=U_1$. All denominators below
involve only $U_0,\ldots,U_T>0$, so $\zeta=0$ is allowed.
Starting from $z_0=z_1=x^\star$ and $A_0=0$, define only for the analysis
\begin{equation*}
 z_{t+1}=(1-\theta_t)z_t+\theta_tx_{t+1},\qquad
 A_t=(1-\theta_t)A_{t-1}+\theta_tG_{t+1}.
\end{equation*}
Let
\begin{equation*}
 \begin{split}
 d_t&=z_{t-1}-x_t,\qquad H_0=\tfrac12\|d_1\|^2,\qquad
 H_t=\frac{W}{2U_t}\|z_t-x_{t+1}\|^2,\\
 q_t&=\frac{2W\eta_t^2}{U_t},\qquad
 h_t=\frac{W\eta_t}{U_{t-1}}d_t.
 \end{split}
\end{equation*}
In particular $H_T$ uses $z_T-x_{T+1}$, not $z_{T+1}-x_{T+1}$.

\begin{lemma}[Unified recursive potential]\label{lem:potential}
For every $m\le T$, $A_m,H_m\ge0$ and
\begin{equation}\label{eq:potential}
 WA_m+H_m\le\frac{R^2}{2}
       +\sum_{t\le m}q_t\|e_t\|^2
       +\sum_{t\le m}\langle e_t,h_t\rangle.
\end{equation}
The full terminal potential also gives
\begin{equation}\label{eq:potential-endpoint}
 \begin{split}
 G_{T+1}
 &\le\left(1+\frac{L\zeta^2}{\zeta+\eta_T}\right)
       \frac{WA_T+H_T}{W}\\
 &\le(1+L\zeta)\frac{WA_T+H_T}{W}.
 \end{split}
\end{equation}
In particular, $\zeta=0$ gives $A_T=G_{T+1}$, while $\zeta=1$
keeps the inverse tail weights bounded.
\end{lemma}
\begin{proof}
Convexity and the synchronized recursions give
$f(z_t)-f^\star\le A_{t-1}$ by induction. Thus
$A_t-A_{t-1}\le\theta_t[f(x_{t+1})-f(z_t)]$.
Also $z_t-x_t=(U_t/U_{t-1})d_t$, including $t=1$ by convention.
Multiply \eqref{eq:one-step} of Lemma~\ref{lem:one-step}, with $z=z_t$ and $\gamma=\eta_t$,
by $W/U_t$ to obtain
\begin{equation*}
 WA_t+H_t\le WA_{t-1}+\frac{U_t}{U_{t-1}}H_{t-1}
       +\frac{q_t}{2}\|e_t\|^2+\langle e_t,h_t\rangle.
\end{equation*}
Since $U_t\le U_{t-1}$, summation proves \eqref{eq:potential},
using $q_t/2\le q_t$. Nonnegativity follows
from the definitions. The final synchronized update gives
$f(z_{T+1})-f^\star\le A_T$ and
\[
 \|z_{T+1}-x_{T+1}\|^2
 =\frac{2\zeta^2H_T}{W(\zeta+\eta_T)}.
\]
By $L$-smoothness evaluated at $z-\nabla f(z)/L$, we have
$\|\nabla f(z)\|^2\le2L[f(z)-f^\star]$.
Thus, writing $G(x)=f(x)-f^\star$ and
$\rho=L\zeta^2/(\zeta+\eta_T)$, smoothness and Cauchy--Schwarz give
\[
 G(x_{T+1})
 \le\left(\sqrt{G(z_{T+1})}
          +\sqrt{L/2}\,\|x_{T+1}-z_{T+1}\|\right)^2
 \le(1+\rho)(A_T+H_T/W).
\]
Since $\rho\le L\zeta$, this proves~\eqref{eq:potential-endpoint}.
If $\zeta=0$, $\theta_T=1$ also gives $A_T=G_{T+1}$ directly.
\end{proof}
The proofs of Theorems~\ref{thm:anytime-upper} and~\ref{thm:fixed-time-upper} take $\zeta=1$; $\zeta=0$ gives $W=S_1$ and $U_t=S_t$.

\section{Conditional Clipping Estimates}\label{app:clipping}
\begin{lemma}[Bias and second moment after clipping]\label{lem:clipping}
Conditionally on a $\sigma$-field $\mathcal G$, suppose $\mu$ is measurable, $\E[g\mid\mathcal G]=\mu$, and $\E[\norm{g-\mu}^p\mid\mathcal G]\le\sigma^p$, for $1<p\le2$. If $\norm\mu\le\lambda/2$, set
\[
 c=\clip(g,\lambda),\quad e=c-\mu,\quad b=\E[e\mid\mathcal G],\quad u=e-b.
\]
Then
\begin{align}
 \norm e&\le\min\{\norm{g-\mu},3\lambda/2\},&
 \norm u&\le2\lambda,\nonumber\\
 \norm b&\le2^{p-1}\sigma^p\lambda^{1-p}
            \le2\sigma^p\lambda^{1-p},\nonumber\\
 \E[\norm e^2\mid\mathcal G]&\le(3\lambda/2)^{2-p}\sigma^p
            \le2\sigma^p\lambda^{2-p},&
 \E[\norm u^2\mid\mathcal G]&\le\E[\norm e^2\mid\mathcal G].
 \nonumber
\end{align}
\end{lemma}
\begin{proof}
These are the standard clipping bounds (cf.\ \citealp[Lemma~5.1]{sadiev2023});
the following Euclidean calculation records the sharper constants used here.
Write $P=\operatorname{clip}(\cdot,\lambda)$ and $\xi=g-\mu$.
Since $P\mu=\mu$, nonexpansiveness and the clipping radius give
\[
 \|e\|\le\min\{\|\xi\|,3\lambda/2\},\qquad
 \|u\|=\|Pg-\E[Pg\mid\mathcal G]\|\le2\lambda.
\]
Also $\|Pg-g\|\le\|\xi\|\mathbf1_{\{\|\xi\|>\lambda/2\}}$, so
\[
 \|b\|=\|\E[Pg-g\mid\mathcal G]\|
 \le(\lambda/2)^{1-p}\E[\|\xi\|^p\mid\mathcal G].
\]
Finally,
\[
 \E[\|u\|^2\mid\mathcal G]\le\E[\|e\|^2\mid\mathcal G]
 \le(3\lambda/2)^{2-p}\E[\|\xi\|^p\mid\mathcal G],
\]
where the first inequality is conditional variance reduction by centering.
\end{proof}

\Needspace{10\baselineskip}
The next decomposition imposes no bound on the true gradient. Related
Lipschitz estimates for the clipping factor appear in
\citet[Appendix C.1.2, Eqs.~(13)--(14)]{tian2026inosgd}, under bounded noise.
We prove the conditional $p$-moment bounds used here.

\begin{lemma}[Conditional mean and variance after clipping]\label{lem:global-clipping}
Conditionally on a $\sigma$-field $\mathcal G$, let $\mu$ be measurable,
$\E[g\mid\mathcal G]=\mu$, and
$\E[\norm{g-\mu}^p\mid\mathcal G]\le\sigma^p$.
For deterministic or $\mathcal G$-measurable $\lambda>0$, put
\[
 \alpha(v)=\begin{cases}\min\{1,\lambda/\norm v\},&v\ne0,\\1,&v=0,
 \end{cases}\quad
 C=\clip(g,\lambda),\quad A=\E[\alpha(g)\mid\mathcal G].
\]
Then $0\le A\le1$ and
\[
 \E[C\mid\mathcal G]=A\mu+\varepsilon,
 \qquad \norm\varepsilon\le\sigma^p\lambda^{1-p}.
\]
Moreover, $u=C-\E[C\mid\mathcal G]$ satisfies
\[
 \norm u\le2\lambda,\qquad
 \E[\norm u^2\mid\mathcal G]\le(2\lambda)^{2-p}\sigma^p.
\]
In particular, if $\sigma\le1$ and $\lambda\ge4$, then
$d=C-A\mu=\varepsilon+u$ satisfies
\[
 \norm d\le3\lambda,\qquad
 \E[\norm d^2\mid\mathcal G]\le4\lambda^{2-p}.
\]
\end{lemma}
\begin{proof}
The scalar function $r\mapsto\min\{1,\lambda/r\}$, extended by $1$ at zero,
is $1/\lambda$-Lipschitz on $[0,\infty)$ and takes values in $[0,1]$.
Writing $\xi=g-\mu$ and using its conditional mean zero gives
\begin{align*}
 \varepsilon
 &=\E[\alpha(g)\xi\mid\mathcal G]
   =\E[(\alpha(g)-\alpha(\mu))\xi\mid\mathcal G],\\
 \norm\varepsilon
 &\le\E\!\left[\min\{\norm\xi/\lambda,1\}\norm\xi
                   \mid\mathcal G\right]
 \le\lambda^{1-p}\E[\norm\xi^p\mid\mathcal G].
\end{align*}
The final pointwise inequality follows separately for $\norm\xi\le\lambda$
and $\norm\xi>\lambda$, using $p\le2$ and \mbox{$p>1$}, respectively.
For the projection $P=\clip(\cdot,\lambda)$, variance reduction and
nonexpansiveness give
\[
 \E[\norm u^2\mid\mathcal G]
 \le\E[\norm{Pg-P\mu}^2\mid\mathcal G]
 \le\E[\min\{\norm\xi,2\lambda\}^2\mid\mathcal G]
 \le(2\lambda)^{2-p}\sigma^p.
\]
Both $C$ and its mean belong to the radius-$\lambda$ ball, proving the range.
Under the last assumptions, $\norm\varepsilon\le\lambda^{1-p}\le\lambda$,
so $\norm d\le3\lambda$; and
\[
 \E[\norm d^2\mid\mathcal G]
 =\E[\norm u^2\mid\mathcal G]+\norm\varepsilon^2
 \le2\lambda^{2-p}+\lambda^{2-2p}
 \le4\lambda^{2-p}.
\]
\end{proof}

\Needspace{10\baselineskip}
\section{Elementary Maximal Concentration Bounds}\label{app:concentration}
We record a standard scalar concentration bound and one energy estimate used throughout the high-probability proofs. Both use only predictable local bounds.

\begin{lemma}[Maximal inequality for bounded martingale increments]\label{lem:martingale}
Let $X_t$ be adapted to $(\F_t)_{t=0}^T$, with $\Et X_t=0$, $|X_t|\le b$ almost surely, and
\[
 \sum_{t=1}^T\Et X_t^2\le v\quad\text{almost surely},
\]
where $b,v\ge0$ are deterministic. For any $\ell>0$,
\begin{equation*}
 \Prob\left\{\max_{1\le m\le T}\sum_{t=1}^mX_t
       >2\sqrt{v\ell}+2b\ell\right\}\le e^{-\ell}.
\end{equation*}
\end{lemma}
\begin{proof}
The maximal form of Freedman's inequality
\citep[Theorem~1.6]{freedman1975}, after scaling the increments, gives
\[
 \Prob\left\{\max_{m\le T}\sum_{t\le m}X_t>s\right\}
 \le\exp\!\left\{-\frac{s^2}{2(v+bs)}\right\}.
\]
Set $s=2\sqrt{v\ell}+2b\ell$; then $s^2\ge2\ell(v+bs)$.
Zero budgets follow by continuity.
\end{proof}

\begin{lemma}[A common energy bound]\label{lem:energy-budget}
Let $e_t$ be integrable and adapted, $P_0,\ldots,P_N\ge0$ adapted,
$J_t\in\{0,1\}$ predictable
and nonincreasing, and $\mathcal E$ an event, for instance $\mathcal E=\Omega$. On $\mathcal E\cap\{J_m=1\}$, suppose
\[
 P_m\le P_0+\sum_{t\le m}\{a_t\|e_t\|^2+\langle e_t,w_t\rangle\},
\]
where $w_t$ is predictable with $\|w_t\|\le s_t\sqrt{P_{t-1}}$ on $J_t=1$, and $a_t,s_t\ge0$ are deterministic.
On $J_t=1$, assume
$\|\mathbb E_{t-1}e_t\|\le B_t$,
$\mathbb E_{t-1}\|e_t\|^2\le V_t$,
$\|e_t\|\le M_t$, and
$\|e_t-\mathbb E_{t-1}e_t\|\le U_t$,
with deterministic bounds. Put
\[
 A=\sum_t a_tV_t,\quad K_0=\sup_t a_tM_t^2,\quad
 B_0=\sum_t s_tB_t,\quad V_0=\sum_t s_t^2V_t,\quad J_0=\sup_t s_tU_t.
\]
For $p_0\ge0$ and $q>0$, define
\[
 C=p_0+A+2\sqrt{K_0Aq}+2K_0q,\qquad
 D=B_0+2\sqrt{V_0q}+2J_0q.
\]
Except on probability $e^{-q}$, on $\mathcal E$, simultaneously for every $m$ with
$J_m=1$ and $P_0\le p_0$,
\[
 P_m\le(\sqrt C+D)^2.
\]
The initial state may be random; the conclusion extends to $N=\infty$
if the budgets are finite.\looseness=-1
\end{lemma}
\begin{proof}
Set $H=(\sqrt C+D)^2$ and
$I_t=J_t\mathbf1_{\{P_0\le p_0,\ \max_{j<t}P_j\le H\}}$, a predictable indicator that does not involve $\mathcal E$.
Center only stopped quantities:
$Y_t=I_t\|e_t\|^2$, $m_t=\mathbb E_{t-1}Y_t$, and
$U'_t=I_t e_t-\mathbb E_{t-1}(I_te_t)$.
The martingale increments
$X_t=a_t(Y_t-m_t)+\langle U'_t,w_t\rangle$ satisfy
\[
 |X_t|\le K_0+\sqrt H J_0,\qquad
 \sqrt{\sum_t\mathbb E_{t-1}X_t^2}
 \le\sqrt{K_0A}+\sqrt{HV_0}.
\]
Indeed, conditional variance reduction and the conditional $L^2$ triangle
inequality give
$\sqrt{\mathbb E_{t-1}X_t^2}\le
 I_t\sqrt{V_t}(a_tM_t+s_t\sqrt H)$;
then use the triangle inequality for the deterministic sequence of bounds.
By Lemma~\ref{lem:martingale}, except on probability $e^{-q}$, every partial sum of the $X_t$ is at most
$2\sqrt{K_0Aq}+2K_0q+\sqrt H(2\sqrt{V_0q}+2J_0q)$. On this event and on $P_0\le p_0$, the mean and bias budgets give, for every $m$,
\[
 P_0+\sum_{t\le m}I_t\{a_t\|e_t\|^2+\langle e_t,w_t\rangle\}\le C+\sqrt HD\le H,
\]
since $H-C-\sqrt HD=\sqrt CD$. Suppose that, on this event and on $\mathcal E$, some $m$ with $J_m=1$ and $P_0\le p_0$ has $P_m>H$, and take the first such $m$. Then $I_t=1$ for every $t\le m$, so the assumed recursion bounds $P_m$ by the left side above, a contradiction.
Taking the union over finite prefixes yields the claim for $N=\infty$.
\end{proof}

\section{Clipped versus Unclipped SGD}\label{app:cuberoot}
\begin{remark}[Clipping and the cube-root barrier]\label{rem:cuberoot}
In concurrent and independent work, \citet[Theorem~2.2]{anonymous2026cuberoot} prove
that at $p=2$, for every positive deterministic schedule of unclipped step sizes and
every $r>1/3$, some finite-variance instance makes $(t+1)^r(f(x_t)-f^\star)$ unbounded
almost surely. We first check that their theorem also bounds the time-uniform
exponent of unclipped SGD on our class, and then compare it with clipping. Fix any positive
deterministic schedule of unclipped step sizes, $r'>1/3$, and finite $C>0$, $m\ge0$.
Choose $1/3<r<r'$. Their theorem supplies a fixed one-dimensional instance
with $L=1$, initial distance $1/2$, and an i.i.d.\ unbiased oracle with variance
at most one, on which $\limsup_{t\to\infty}(t+1)^r(f(x_t)-f^\star)=\infty$
almost surely. Their construction also satisfies our conditions on the sample losses:
the shock branch has sample loss $x^2/2+\xi x$, whose minimum is attained
at $-\xi$, while the noiseless branch uses a convex smooth function with
minimum at zero. Translating their initial point $x_0=1/2$ to our $x_1=0$
and shifting the time index preserves these properties, so both branches
belong to our normalized $p=2$ class. On the event
\[
 f(x_{t+1})-f^\star\le Ct^{-r'}\log^m(et)\qquad\text{for all }t\ge1,
\]
the gap multiplied by $t^r$ tends to zero, since
$t^{r-r'}\log^m(et)\to0$. The event therefore has probability zero.
Thus, by their theorem, deterministic unclipped step sizes attain no
time-uniform exponent above $1/3$ on our class.
In contrast, Theorem~\ref{thm:anytime-upper} attains $\beta_2=2/5$ with
clipping, and Theorem~\ref{thm:anytime} excludes every larger exponent for
all gradient-only rules (Definition~\ref{def:class}), including adaptive step sizes
and clipping thresholds. Since clipping is a scalar response to the observed stochastic
gradient, this resolves the time-uniform counterpart (at a fixed confidence level) of their
question regarding adaptive step sizes: the exponent $2/5$ is both attainable and optimal
on our finite-variance class. Moreover, the almost sure consequence of Theorem~\ref{thm:anytime-upper} gives $f(x_{t+1})-f^\star=O(t^{-2/5}\log^4t)$ almost surely on every instance of our class, including theirs, so clipping also exceeds the almost sure exponent $1/3$ of unclipped deterministic schedules.
\end{remark}

\Needspace{10\baselineskip}
\section{Proofs for the Time-Uniform Lower Bound and Fixed-Time Guarantees}\label{app:anytime}
This appendix proves Proposition~\ref{prop:simultaneous-budget} and Theorems~\ref{thm:anytime} and~\ref{thm:anytime-marginal}, then extends both lower bounds to rules whose record cannot advance on zero feedback (Proposition~\ref{prop:record-lower}). Theorem~\ref{thm:fixed-time-upper} is proved in Appendix~\ref{app:fixed-time-upper}. For the lower bounds we fix $p\in(1,2]$, $L=R=\sigma=1$, and $x_1=0$. In every comparison the same rule is used on every instance.

Section~\ref{sec:anytime} outlines the argument. Its core is the following progress bound.

\begin{proposition}[Progress bound under time-uniform guarantees]\label{prop:simultaneous-budget}
Fix $\varepsilon\in(0,1/2)$, $N\ge1$, a rule of Definition~\ref{def:class}, and positive numbers $b_1,\ldots,b_N$; the rule may be tuned to a confidence level different from $\varepsilon$. Suppose that every instance in the class of Theorem~\ref{thm:anytime} satisfies
\begin{equation}\label{eq:simultaneous-prefix-guarantee}
 \Prob\{f(x_{t+1})-f^\star\le b_t\text{ for all }1\le t\le N\}\ge1-\varepsilon.
\end{equation}
Then for every $\mu\in(4b_N,1/2]$, with $A=(2\mu)^{-1/(p-1)}$ and $\pi=\mu/A$,
\begin{equation}\label{eq:necessary-budget-main}
 \pi\sum_{t=1}^N\bigl(\mu+\sqrt{2b_t}\bigr)\ge\frac{1-2\varepsilon}{2}.
\end{equation}
\end{proposition}

\subsection{Translated Targets with a Common I.I.D. Transcript}\label{app:lower-witnesses}
For $0<\mu\le1/2$ and $y\in[\mu,1]$, define
\begin{equation*}
 f_{\mu,y}(x)=\tfrac12(x-y+\mu)_+^2-\mu x,
 \qquad f_{\mu,y}^\star=-\mu y+\mu^2/2.
\end{equation*}
Its derivative is $-\mu$ on $x\le y-\mu$ and $x-y$ on $x\ge y-\mu$. These branches agree at the transition. Thus $f_{\mu,y}$ is convex and globally $1$-smooth, with unique minimizer $y$ and initial distance at most one. Its gap is
\begin{equation}\label{eq:translated-gap}
 f_{\mu,y}(x)-f_{\mu,y}^\star=
 \begin{cases}
 \mu(y-x)-\mu^2/2,&x\le y-\mu,\\
 (x-y)^2/2,&x\ge y-\mu.
 \end{cases}
\end{equation}
This objective is a reflected and translated one-sided Huber loss \citep{huber1964}, following the construction of \citet[Theorem~2 and Section~3.2]{kornowski2024}; \citet[Lemma~3.1]{tsai2026} use a related overshoot at an arbitrary iteration. Those lower bounds concern exact gradients and prescribed step sizes. Against adaptive rules with i.i.d.\ noise, choosing the translation after observing a random trajectory would select the instance from the realized path, so we fix the translated family before the run and integrate, over the target location, the probabilities of failing at the first crossing. This accommodates adaptive step sizes and algorithmic randomness.

Set
\begin{equation}\label{eq:rare-progress-parameters}
 A=(2\mu)^{-1/(p-1)}\ge1,\qquad
 \pi=\mu/A=1/(2A^p),\qquad Y_t\ \text{i.i.d.}\ \operatorname{Bernoulli}(\pi).
\end{equation}
Independently of the algorithm's seed, each call draws a fresh $Y_t$. Use the sample loss
\begin{equation}\label{eq:progress-sample}
 F_{\mu,y}(x;Y)=\tfrac12(x-y+\mu)_+^2-AYx
              =f_{\mu,y}(x)+(\mu-AY)x.
\end{equation}
The additive noise $\zeta=\mu-AY$ is centered, independent of the query and time, and has the same law for every $y$. Since $\mu\le A$,
\begin{equation*}
 \E|\zeta|^p=(1-\pi)\mu^p+\pi(A-\mu)^p
 \le\mu^p+\mu A^{p-1}\le2\mu A^{p-1}=1.
\end{equation*}
Every sample loss is convex and globally $1$-smooth, bounded below, and attains its minimum: for $Y=0$, every $x\le y-\mu$ is a minimizer; for $Y=1$, the minimizer is $y-\mu+A$. On the common left branch the raw gradient is exactly
\begin{equation*}
 g_{\mu,y}(x;Y)=-AY\qquad(x\le y-\mu).
\end{equation*}
The algorithm observes only these gradients and does not know $y$ or $f^\star$.

\subsection{Proof of Proposition~\ref{prop:simultaneous-budget}}\label{app:anytime-simultaneous}
We first isolate two lemmas used both here and for the momentum lower bound in
Appendix~\ref{app:momentum}: a crossing bound for monotone processes, and its
application to the translated targets through a coupling at the first crossing.

\begin{lemma}[Crossing bound for monotone processes]\label{lem:first-crossing-budget}
Let $N\ge1$ be an integer, let $X_1=0\le\cdots\le X_{N+1}$ be finite random
variables, and put $D_t=X_{t+1}-X_t\ge0$. For deterministic finite
$c>0$, $s_1,\ldots,s_N\ge0$,
and $0\le\varepsilon<1/2$, suppose $\Prob(X_{N+1}>c)\ge1-\varepsilon$
and, for each fixed $z\in[0,c)$,
\[
 \Prob\{\exists t\le N:\ X_t\le z<X_{t+1},\quad
                    X_{t+1}>z+s_t\}\le\varepsilon.
\]
Then $c(1-2\varepsilon)\le\E\sum_{t=1}^N\min\{s_t,D_t\}$.
\end{lemma}

\begin{proof}
Let $B=[0,c)\cap\bigcup_{t=1}^N[X_t,X_{t+1}-s_t)$ be the badly
crossed boundaries, with reversed intervals empty. Measurable endpoints
make $\ind_{\{z\in B\}}$ jointly measurable. The fixed-$z$ hypothesis and
Tonelli's theorem give
$\E|B|=\int_0^c\Prob(z\in B)\,dz\le c\varepsilon$,
where $|B|$ denotes Lebesgue measure.
At update $t$, the safely crossed boundaries have total length at most
$\min\{s_t,D_t\}$.
On $\{X_{N+1}>c\}$ every boundary in $[0,c)$ has been crossed, so
\[
 c\ind_{\{X_{N+1}>c\}}\le |B|+\sum_{t=1}^N\min\{s_t,D_t\}.
\]
Taking expectations proves the claim.
\end{proof}

\begin{lemma}[Fixed targets coupled at the first crossing]\label{lem:first-crossing-coupling}
Fix $0<\varepsilon<1/2$, an integer $N\ge1$, positive numbers
$b_1,\ldots,b_N$, and $\mu\in(4b_N,1/2]$. Let $A$, $\pi$, and $Y_t$ be as
in~\eqref{eq:rare-progress-parameters}, and put $c=1/2$ and
$s_t=\mu+\sqrt{2b_t}$. Consider a method that starts at $x_1=0$, queries one
stochastic gradient at its current iterate $x_t$, and computes $x_{t+1}\in\R$, together
with any auxiliary state, from the gradients observed so far and an independent
seed, by the same rule on every instance. Suppose
every instance $f_{\mu,y}$, $y\in[\mu,1]$, with sample
loss~\eqref{eq:progress-sample} satisfies~\eqref{eq:simultaneous-prefix-guarantee}.
Let the reference run $X_t$ be the states produced from the observations
$-AY_t$ and the seed through the first update with $X_{t+1}>c$, held fixed
thereafter without further oracle calls. Define
$M_t=\max\{0,X_1,\ldots,X_t\}$ and $D_t=M_{t+1}-M_t$. Then
$c(1-2\varepsilon)\le\E\sum_{t=1}^N\min\{s_t,D_t\}$.
\end{lemma}

\begin{proof}
Since $c\le1-\mu$, every query before stopping lies where target $y=1$
returns $-AY_t$, so the reference agrees with the run on that
target until it stops. On $\{M_{N+1}\le c\}$ it has not stopped by
time $N$, so $X_{N+1}\le c$, and \eqref{eq:translated-gap} gives
\[
 f_{\mu,1}(X_{N+1})-f_{\mu,1}^\star
 \ge\mu(1-c)-\mu^2/2\ge\mu/4>b_N.
\]
Hence $\Prob(M_{N+1}>c)\ge1-\varepsilon$.

For any fixed $z\in[0,c)$, couple target $y=z+\mu$ using the
same Bernoulli sequence and the same seed. Its run agrees with the
reference in all states, including auxiliary ones, through the first
update strictly crossing $z$: all queries up to that update lie at or
below $z=y-\mu$ and have the same observations $-AY_t$. Each target is
coupled separately, before the run.
On $\{M_t\le z<M_{t+1}\}$ this is the first crossing and
$M_{t+1}=X_{t+1}$. If also $M_{t+1}>z+s_t$, the coupled target has
gap $(X_{t+1}-z-\mu)^2/2>b_t$. Its prefix guarantee verifies the
crossing hypothesis of Lemma~\ref{lem:first-crossing-budget} for the
nondecreasing process $M_t$, for each fixed $z$. Applying that lemma proves
the claim. When $X_t$ is nondecreasing, $M_t=X_t$ and $D_t=X_{t+1}-X_t$.
\end{proof}

\begin{proof}\textit{of Proposition~\ref{prop:simultaneous-budget}.}
Let $\varepsilon$, $N$, the rule, and $b_1,\ldots,b_N$ be as in its
statement, and fix $\mu\in(4b_N,1/2]$. The instances of
Appendix~\ref{app:lower-witnesses} belong to the class
of Theorem~\ref{thm:anytime}, and the rule meets the conditions of
Lemma~\ref{lem:first-crossing-coupling}.
Let $X_t$ be its reference run, with $c=1/2$ and $s_t=\mu+\sqrt{2b_t}$.
All observations prior to stopping equal $-AY_t$, so
\begin{equation*}
 X_1=0,\qquad X_{t+1}=X_t+Y_tq_t\ge X_t\ge0,
\end{equation*}
where $q_t=Ac_t\ge0$ before stopping and $q_t=0$ afterwards.
Here $M_t=X_t$, so the record increments equal the state increments.
The responses are finite but need not be integrable or independent of
$Y_t$. Since $D_t=Y_tq_t$, we have $\min\{s_t,D_t\}\le Y_ts_t$.
Since $c=1/2$, Lemma~\ref{lem:first-crossing-coupling} and $\E Y_t=\pi$ give~\eqref{eq:necessary-budget-main}:
\[
 c(1-2\varepsilon)\le\E\sum_{t=1}^N\min\{s_t,D_t\}
 \le\pi\sum_{t=1}^Ns_t.
\]
\end{proof}

\subsection{Proof of the Time-Uniform Lower Bound}
We first record the exponent calculation shared by Theorems~\ref{thm:anytime}
and~\ref{thm:momentum-obstruction}; the proof of Theorem~\ref{thm:anytime-marginal} repeats it with a confidence factor. Put $q=p/(p-1)$ and choose
$\beta_p<u<\min\{r,1\}$. For any positive subpolynomial factors
$D_N,K_N=N^{o(1)}$, let
\begin{equation}\label{eq:sum-envelope}
 b_{N,t}=D_Nt^{-u},\qquad \mu=8D_NN^{-u},\qquad
 \pi=2^{1/(p-1)}\mu^q.
\end{equation}
Then $\mu/\sqrt{2b_{N,N}}=4\sqrt2\sqrt{D_N}N^{-u/2}\to0$.
Since $b_{N,t}\ge b_{N,N}$, $u<1$, and $\beta_p=1/(q+1/2)$, summing $t^{-u/2}$ gives
\begin{equation}\label{eq:exact-rate-contradiction}
 K_N\pi\sum_{t=1}^N(\mu+\sqrt{2b_{N,t}})
 =O\!\left(K_ND_N^{q+1/2}N^{1-u(q+1/2)}\right)\longrightarrow0.
\end{equation}

\begin{proof}\textit{of Theorem~\ref{thm:anytime}.}
Suppose its envelope holds on every instance. Choose $u$ as above and
$D_N=C\log^m(eN)$, so $Ct^{-r}\log^m(et)\le b_{N,t}$ for $t\le N$.
For all sufficiently large $N$, $4b_{N,N}<\mu\le1/2$.
Proposition~\ref{prop:simultaneous-budget} with $\varepsilon=\delta$
contradicts \eqref{eq:exact-rate-contradiction} with $K_N=1$.
A failing finite prefix yields~\eqref{eq:anytime-failure}. The resulting
instance is fixed before its i.i.d.\ samples and the algorithm seed are
drawn, and it may depend on $N$.
\end{proof}

\Needspace{5\baselineskip}
\subsection{Proof of Theorem~\ref{thm:anytime-marginal}}\label{app:anytime-marginal}

\begin{proof}
Suppose the displayed envelope holds on every instance at every time and
confidence. Since $c<1-\beta_p$ and $r>\beta_p+c$, we may choose
$\beta_p+c<u<\min\{r,1\}$; put $v=u-c>\beta_p$. For each deterministic $N$, set
\[
 \delta_N=\frac1{4N},\qquad D_N=C4^c\log^m(4eN^2),\qquad
 b_{N,t}=D_NN^ct^{-u}.
\]
Since $\delta_N^{-c}=(4N)^c$ and $\log(et/\delta_N)\le\log(4eN^2)$, the claimed fixed-time envelope is at most $b_{N,t}$ for $t\le N$.
Use the $\delta_N$-tuned rule on the prefix. A union
bound gives failure probability at most $N\delta_N=1/4$ on every instance.
Proposition~\ref{prop:simultaneous-budget}, with the separate budget
$\varepsilon_0=1/4$ and $\mu=8b_{N,N}=8D_NN^{-v}$, which lies in $(4b_{N,N},1/2]$ for large $N$, therefore requires
$\pi\sum_{t\le N}(\mu+\sqrt{2b_{N,t}})\ge1/4$. With $\pi=2^{1/(p-1)}\mu^q$ and $\sum_{t\le N}t^{-u/2}\le2N^{1-u/2}$,
\[
 \pi\sum_{t\le N}(\mu+\sqrt{2b_{N,t}})
 \le2^{1/(p-1)}8^qD_N^q\bigl(8D_NN^{-v/2}+2\sqrt{2D_N}\bigr)N^{1-v(q+1/2)}.
\]
The right side tends to zero, because $D_N=N^{o(1)}$ and $v(q+1/2)=v/\beta_p>1$; this is a contradiction. By the union bound, a failure probability exceeding
$1/4$ on the prefix $1\le t\le N$ implies a fixed-time failure probability above
$\delta_N$ at some deterministic $t\le N$.
\end{proof}

\subsection{Rules Whose Record Does Not Advance on Zero Feedback}\label{app:record-rules}
We now allow the iterate to move toward a fixed anchor and, more generally,
to retreat or recover earlier positions after a zero observation. The
quantity used in the crossing argument is the running maximum of the
queries in a fixed direction.

\begin{definition}[Rules with no record advance on zero feedback]\label{def:record-rules}
A one-dimensional horizon-free rule starts at $x_1=0$, queries one fresh
stochastic gradient $g_t$ at $x_t$, and outputs its updated state $x_{t+1}$.
Its states are finite measurable functions of the gradients observed so
far, $t$, $p$, $\delta$, the bounds $L,R,\sigma$, and an independent random
seed; arbitrary auxiliary memory is allowed. The rule receives no other
information about the objective. There is a direction $e\in\{-1,1\}$,
fixed before the run, such that for every $A>0$, every finite feedback
history with $g_s\in\{0,-eA\}$, and every seed, the quantities
\[
 U_t=e(x_t-x_1),\qquad M_t=\max\{0,U_1,\ldots,U_t\}
\]
satisfy
\begin{equation}\label{eq:zero-feedback-record}
 g_t=0\quad\Longrightarrow\quad U_{t+1}\le M_t.
\end{equation}
The direction may depend on the rule, its deterministic initialization,
and its confidence input, but is independent of the unknown objective
and the realized seed.
\end{definition}

\begin{proposition}[Lower bounds for rules with no record advance]\label{prop:record-lower}
Fix $p\in(1,2]$ and $L=R=\sigma=1$.
\begin{enumerate}
\item For every $\delta\in(0,1/2)$, every rule of
Definition~\ref{def:record-rules}, every $C>0$, $m\ge0$, and $r>\beta_p$,
there is an instance for which
\[
 \Prob\{\exists t\ge1:\ f(x_{t+1})-f^\star>
 Ct^{-r}[\log(et)]^m\}>\delta.
\]
The rule, $C$, and $m$ may depend on $\delta$ and $p$.
\item For every family $(\mathrm{rule}_\delta)_{\delta\in(0,1)}$ of
rules of Definition~\ref{def:record-rules}, every $C>0$, $m\ge0$,
$c\in[0,1-\beta_p)$, and $r>\beta_p+c$, there are
$\delta\in(0,1)$, a deterministic integer $t\ge1$, and an instance for which
\[
 \Prob\{f(x_{t+1})-f^\star>
 Ct^{-r}\delta^{-c}[\log(et/\delta)]^m\}>\delta.
\]
Here $C,c,m$ are independent of $\delta,t$, and the instance.
\end{enumerate}
In both parts the instance satisfies Assumptions~\ref{as:objective}
and~\ref{as:oracle}, is one-dimensional, and has i.i.d.\ additive noise
and convex, globally $1$-smooth sample losses attaining their minima.
\end{proposition}

\begin{proof}
We first establish~\eqref{eq:necessary-budget-main} for the larger class.
Fix a rule, $N\ge1$, $\varepsilon\in(0,1/2)$, positive $b_1,\ldots,b_N$,
and $\mu\in(4b_N,1/2]$, and suppose its prefix guarantee
\eqref{eq:simultaneous-prefix-guarantee} holds on every instance.
Choose its direction $e$ and use the reflected sample losses
$F_{\mu,y}(ex;Y)$ from~\eqref{eq:progress-sample}. Reflection preserves
convexity, smoothness, the initial distance, and the central moment bound.
In the coordinate $u=ex$, every query with $u\le y-\mu$ receives raw
feedback $-eAY_t$. Thus the same seed and these observations define a
reference run $U_t$, held fixed after its first crossing of $1/2$.
Lemma~\ref{lem:first-crossing-coupling}, applied in this coordinate with
$M_t=\max\{0,U_1,\ldots,U_t\}$ and $D_t=M_{t+1}-M_t$, gives
\[
 \frac{1-2\varepsilon}{2}
 \le\E\sum_{t=1}^N\min\{s_t,D_t\},\qquad
 s_t=\mu+\sqrt{2b_t}.
\]
Condition~\eqref{eq:zero-feedback-record} implies $D_t=0$ when $Y_t=0$;
after stopping $D_t=0$ as well. Consequently
$\min\{s_t,D_t\}\le Y_ts_t$, and $\E Y_t=\pi$ proves exactly
\eqref{eq:necessary-budget-main}. The truncated increments are bounded
by $s_t$, so this step requires no moment bound on the state increments.

For the first part, suppose the asserted envelope holds on every
instance. Put $q=p/(p-1)$, choose
$\beta_p<u<\min\{r,1\}$, and set
\[
 D_N=C\log^m(eN),\qquad b_{N,t}=D_Nt^{-u},\qquad
 \mu=8D_NN^{-u}.
\]
These $b_{N,t}$ dominate the asserted envelope for $t\le N$, and
$4b_{N,N}<\mu\le1/2$ for all sufficiently large $N$.
The budget with $\varepsilon=\delta$ requires its left side to be at least
$(1-2\delta)/2$. But $\pi=2^{1/(p-1)}\mu^q$ and
$\sum_{t\le N}t^{-u/2}\le2N^{1-u/2}$ give
\[
 \pi\sum_{t=1}^N(\mu+\sqrt{2b_{N,t}})
 \le2^{1/(p-1)}8^qD_N^q
 \bigl(8D_NN^{-u/2}+2\sqrt{2D_N}\bigr)
 N^{1-u(q+1/2)}\longrightarrow0,
\]
since $D_N=N^{o(1)}$ and $u(q+1/2)>1$. This contradiction yields a fixed
instance with prefix failure probability greater than $\delta$.

For the second part, suppose every fixed-time envelope holds on every
instance at every confidence level. Choose
$\beta_p+c<u<\min\{r,1\}$ and put $v=u-c>\beta_p$. For each $N$, use the
rule tuned to $\delta_N=1/(4N)$ and set
\[
 D_N=C4^c\log^m(4eN^2),\qquad b_{N,t}=D_NN^ct^{-u},\qquad
 \mu=8D_NN^{-v}.
\]
For $t\le N$, these bounds dominate the claimed envelope at confidence
$\delta_N$. A union bound gives the prefix guarantee with the separate
failure budget $\varepsilon=1/4$. The direction $e$ is chosen for this
rule before the run. For large $N$, the budget therefore requires
\[
 \frac14\le\pi\sum_{t=1}^N(\mu+\sqrt{2b_{N,t}})
 \le2^{1/(p-1)}8^qD_N^q
 \bigl(8D_NN^{-v/2}+2\sqrt{2D_N}\bigr)
 N^{1-v(q+1/2)}\longrightarrow0,
\]
a contradiction. A prefix failure probability greater than $1/4$
implies a failure probability greater than $1/(4N)=\delta_N$ at some
deterministic $t\le N$, by the union bound.
\end{proof}

\input{appendices/learning_rate_free}

\section{Proof of the Horizon-Free Time-Uniform Upper Bound}\label{app:anytime-upper}
The proof first derives a time-uniform radius bound from a global decomposition of the conditional expectation of a clipped gradient. Once the clipping
threshold exceeds the gradient norms allowed by this radius, the local clipping estimates apply
on a deterministic suffix of the trajectory.
For each analysis horizon, we use bounded inverse tail weights to control the full potential, which consists of the objective gap at a comparison point and its squared distance to the last iterate. Smoothness converts both terms into the
endpoint gap by~\eqref{eq:potential-endpoint}. A summable allocation
over horizons and a direct bound at early times give the same
$D_p=4/c_p$ for all times.

We first prove Theorem~\ref{thm:anytime-upper} under the normalized upper bounds $L,R,\sigma\le1$. Thus $\norm{x_1-x^\star}\le1$, $f$ is convex and at most $1$-smooth, and the conditional central $p$th noise moment is at most one. Define $\beta,a,b$ by~\eqref{eq:anytime-exponents}, and write
\begin{equation*}
 \ell_t=\log(100t^2),\qquad
 \eta_t=ct^{-a}/\ell_t^2,\qquad \lambda_t=4t^b.
\end{equation*}
Here $c\in(0,1]$ is a deterministic constant depending only on $p$, chosen below. Set $G_t=f(x_t)-f^\star$, $e_t=\clip(g_t,\lambda_t)-\nabla f(x_t)$, $b_t=\Et e_t$, and $u_t=e_t-b_t$. Whenever $\norm{\nabla f(x_t)}\le\lambda_t/2$, Lemma~\ref{lem:clipping} gives
\begin{equation}\label{eq:anytime-clipping}
 \norm{e_t}\le\tfrac32\lambda_t,\quad \norm{u_t}\le2\lambda_t,\quad
 \norm{b_t}\le B_t:=2\lambda_t^{1-p},\quad
 \Et\norm{e_t}^2,\ \Et\norm{u_t}^2\le V_t:=2\lambda_t^{2-p}.
\end{equation}
The bounds $B_t,V_t$ are deterministic. The concentration arguments below apply these conditional estimates on predictable events.

\Needspace{10\baselineskip}
\subsection{Explicit Constants and Deterministic Budgets}\label{app:anytime-upper-constants}
Put $v=(2-p)b$. The exponent identities needed throughout are
\begin{equation}\label{eq:anytime-identities}
 b(p-1)=\beta,\qquad a+\beta=1,\qquad
 2(a-b)=\beta,\qquad a-v=\tfrac32\beta,\qquad
 2a-v=1+\tfrac12\beta.
\end{equation}
Let $q=2^{-\beta}$ and define the following finite constants:
\begin{align}
 J_\beta&=\sum_{j=0}^{\infty}q^j(j+2)^2
 =\frac{q(1+q)}{(1-q)^3}+\frac{4q}{(1-q)^2}+\frac4{1-q},\nonumber\\
 H_\beta&=64(1+\beta^{-1}+J_\beta),\qquad
 K=2^{20}H_\beta^4,\label{eq:anytime-constants}\\
 c_p&=\frac1{1000K},\qquad D_p=\frac4{c_p}.
 \nonumber
\end{align}
These constants depend only on $p$. The deterministic inequalities below hold for every $0<c\le1$; we then take $c=c_p$.

For a deterministic proof horizon $T$, set $S_t=\sum_{s=t}^T\eta_s$, $S_0=S_1$, $S_{T+1}=0$, $U=1+S_1$, and
\begin{equation}\label{eq:capped-Q}
 \mathsf Q_t=\frac{U\eta_t^2}{1+S_t}.
\end{equation}
The following estimates treat early times, late times, and small cumulative step sums.\looseness=-1

\begin{lemma}[Deterministic anytime budgets]\label{lem:anytime-budgets}
For all $T\ge1$ and $0<c\le1$,
\begin{align}
 \frac c2\frac{T^\beta}{\ell_T^2}\le S_1
 &\le H_\beta c\frac{T^\beta}{\ell_T^2},\label{eq:step-sum-anytime}\\
 \sum_{t\ge1}\eta_tB_t&\le Kc/\ell_1,\qquad
 \sum_{t\ge1}\eta_t^2V_t\le Kc^2/\ell_1^4,\qquad
 \sup_{t\ge1}\eta_t\lambda_t\le4c/\ell_1^2,
 \label{eq:global-anytime-budgets}\\
 \sum_{t=1}^T\sqrt{\mathsf Q_t}B_t&\le Kc/\ell_1,\qquad
 \sum_{t=1}^T\mathsf Q_tV_t\le Kc^2/\ell_1^3,\qquad
 \max_{t\le T}\mathsf Q_t\lambda_t^2\le Kc^2/\ell_1^4.
 \label{eq:capped-anytime-budgets}
\end{align}
\end{lemma}

\begin{proof}
From $B_t\le2t^{-\beta}$, $V_t\le8t^v$ and
\eqref{eq:anytime-identities},
\[
 \eta_tB_t\le\frac{2c}{t\ell_t^2},\qquad
 \eta_t^2V_t\le\frac{8c^2t^{-1-\beta/2}}{\ell_t^4},\qquad
 \eta_t^2\lambda_t^2=\frac{16c^2t^{-\beta}}{\ell_t^4}.
\]
Integral comparison gives $\sum_t(t\ell_t^2)^{-1}\le\ell_1^{-1}$
and $\sum_t t^{-1-\beta/2}\le1+2/\beta$, proving
\eqref{eq:global-anytime-budgets}.
The final $\lceil T/2\rceil$ indices give $S_1\ge(T/2)\eta_T$.
On each nonempty dyadic block $T/2^{j+1}<t\le T/2^j$,
\[
 \ell_T/\ell_t\le2(j+2),\qquad
 \sum_{T/2^{j+1}<t\le T/2^j}t^{-a}\le2T^\beta2^{-\beta j}.
\]
Summation gives $S_1\le8J_\beta cT^\beta/\ell_T^2$ and hence
\eqref{eq:step-sum-anytime}.

Split the sums of~\eqref{eq:capped-anytime-budgets} at $\lfloor T/2\rfloor$.
On the early half, $U/(1+S_t)\le4H_\beta$, so its bias,
variance and range contributions are at most
\begin{equation*}
 4\sqrt{H_\beta}\frac c{\ell_1},\qquad
 32H_\beta(1+2/\beta)\frac{c^2}{\ell_1^4},\qquad
 64H_\beta\frac{c^2}{\ell_1^4}.
\end{equation*}
On the late half, $t\ge T/2$ and $\ell_T/\ell_t\le2$ imply
\begin{equation*}
 \eta_tV_t\le64cT^{-3\beta/2}/\ell_T^2,\qquad
 \eta_t^2\lambda_t^2\le512c^2T^{-\beta}/\ell_T^4,\qquad
 B_t\le4T^{-\beta}.
\end{equation*}
Since $\eta_t=S_t-S_{t+1}$, comparing the sums with the corresponding integrals yields
\begin{equation}\label{eq:tail-integrals}
 \sum_t\frac{\eta_t}{1+S_t}\le\log U,\qquad
 \sum_t\frac{\eta_t}{\sqrt{1+S_t}}\le2(\sqrt U-1).
\end{equation}
Consequently the late bias and range are bounded by
\[
 8T^{-\beta}\sqrt U(\sqrt U-1)\le8H_\beta c/\ell_T^2,
 \qquad
 512Uc^2T^{-\beta}/\ell_T^4\le1024H_\beta c^2/\ell_1^4.
\]
For the late variance, use $U\log U\le S_1(1+\log U)$, valid also
when $S_1\le1$, and $\log U\le H_\beta\ell_T$ to obtain
\[
 \sum_{t>T/2}\mathsf Q_tV_t
 \le\frac{64cT^{-3\beta/2}}{\ell_T^2}U\log U
 \le\frac{128H_\beta^2c^2T^{-\beta/2}}{\ell_T^3}.
\]
Here $U\le1+H_\beta T^\beta$ gives the logarithmic bound. Adding the early
contributions proves~\eqref{eq:capped-anytime-budgets} with $K=2^{20}H_\beta^4$.
\end{proof}

\Needspace{8\baselineskip}
\subsection{A Radius Event for All Times}\label{app:anytime-upper-radius}
Use $c=c_p$ and fix $\delta\in(0,1)$; the schedule does not depend on $\delta$. Write
$h=\log(100/\delta)$ and $\iota_T=\iota_T(\delta)=\log(100T^2/\delta)$;
thus $\iota_T\ge h>2$.
Apply Lemma~\ref{lem:global-clipping} conditionally at every step, with
$\mu_t=\nabla f(x_t)$, and write
\[
 \clip(g_t,\lambda_t)=A_t^{\rm s}\mu_t+d_t,
 \quad d_t=\varepsilon_t+u_t,
 \quad \Et u_t=0.
\]
The following bounds hold globally: $0\le A_t^{\rm s}\le1$,
$\norm{\varepsilon_t}\le B_t$,
$\Et\norm{u_t}^2\le V_t$,
$\norm{u_t}\le2\lambda_t$,
$\norm{d_t}\le3\lambda_t$, and
$\Et\norm{d_t}^2\le2V_t$.\looseness=-1

Let $r_t=\norm{x_t-x^\star}$ and $G_t=f(x_t)-f^\star$.
Since $\eta_t A_t^{\rm s}\le1/2$, expansion and
$\norm{\mu_t}^2\le2G_t$, $\ip{\mu_t}{x_t-x^\star}\ge G_t$ give
\begin{equation*}
 r_{t+1}^2\le r_t^2-2\eta_t\ip{d_t}{x_t-x^\star}
                         +2\eta_t^2\norm{d_t}^2.
\end{equation*}
The remaining term with $\mu_t$ is at most
$2\eta_tA_t^{\rm s}(2\eta_tA_t^{\rm s}-1)G_t\le0$, as
$0\le\eta_tA_t^{\rm s}\le1/2$.

Apply Lemma~\ref{lem:energy-budget} to this recursion with
$P_{t-1}=r_t^2$, $p_0=1$, $a_t=2\eta_t^2$, $s_t=2\eta_t$,
and $w_t=-2\eta_t(x_t-x^\star)$, using confidence parameter $h$ and
$J_t\equiv 1$ (no localization indicator). The error bounds above and
\eqref{eq:global-anytime-budgets}, using $\ell_1>1$, give
\[
 C\le1+650Kc^2h,\qquad
 D\le2Kc+6\sqrt Kc\sqrt h+32ch.
\]
Since $K\ge1$, $h>2$ and $Kc=1/1000$,
$\sqrt C+D\le1+66Kch<h$.
Hence there is an event $\mathcal R_\delta$, with
$\Prob(\mathcal R_\delta)\ge1-e^{-h}$, on which
\begin{equation}\label{eq:all-radius}
 r_t\le h,\qquad G_t\le h^2/2\qquad\hbox{for every }t\ge1.
\end{equation}
The gap bound uses smoothness at the minimizer.

\Needspace{12\baselineskip}
\subsection{A Potential on a Suffix of the Trajectory}\label{app:anytime-upper-potential}
Set $k=\lceil h^{1/b}\rceil$ and split the analysis of the original recursion at time $k$.
For $t\ge k$, $\lambda_t\ge4h$, so~\eqref{eq:anytime-clipping} holds
whenever $r_t\le2h$, since $\norm{\nabla f(x_t)}\le r_t\le\lambda_t/2$.

Fix a deterministic $T\ge2k$. Apply Lemma~\ref{lem:potential}, with
$\zeta=1$, to the updates $k,\ldots,T$.
Keep $S_t=\sum_{s=t}^T\eta_s$, and set $U_t=1+S_t$, $W=U_k$,
$U_{k-1}=U_k$, and write $\mathcal A_m,\mathcal H_m$ for the quantities $A_m,H_m$ of Lemma~\ref{lem:potential} for these updates. Start its analysis recursions at
$z_{k-1}=z_k=x^\star$ and $\mathcal A_{k-1}=0$, with
$\mathcal H_{k-1}=r_k^2/2$. Write
\begin{equation*}
 \mathsf Q_t^{(k)}=\frac{W\eta_t^2}{U_t},\qquad
 w_t=\frac{W\eta_t}{U_{t-1}}(z_{t-1}-x_t),\qquad
 \mathcal H_t=\frac{W}{2U_t}\norm{z_t-x_{t+1}}^2.
\end{equation*}
The proof of Lemma~\ref{lem:potential} is pathwise, so here its initial distance term is
$r_k^2/2$, also when $x_k$ is random. Its inequality for the partial sums is
\begin{equation}\label{eq:capped-partial-potential}
 W\mathcal A_m+\mathcal H_m
 \le\tfrac12r_k^2+2\sum_{t=k}^m\mathsf Q_t^{(k)}\norm{e_t}^2
                  +\sum_{t=k}^m\ip{e_t}{w_t}.
\end{equation}
The denominator of $\mathsf Q_t^{(k)}$ agrees with that of
$\mathsf Q_t$ in~\eqref{eq:capped-Q}, and its numerator is smaller.
Thus $\mathsf Q_t^{(k)}\le\mathsf Q_t$, so all three
budgets~\eqref{eq:capped-anytime-budgets} apply on the suffix.

Apply Lemma~\ref{lem:energy-budget} with the predictable indicators
$J_t=\mathbf1_{\{r_s\le2h\text{ for }k\le s\le t\}}$, with initial event $r_k\le h$ and
$p_0=h^2/2$. Set
\[
 P_m=W\mathcal A_m+\mathcal H_m,\qquad
 a_t=2\mathsf Q_t^{(k)},\qquad s_t=\sqrt{2\mathsf Q_t^{(k)}}.
\]
Indeed $\|w_t\|^2\le2\mathsf Q_t^{(k)}\mathcal H_{t-1}
\le s_t^2P_{t-1}$, including $t=k$.
With exponent $q=\iota_T$, the budgets~\eqref{eq:capped-anytime-budgets} and $\ell_1>1$ give
$C\le h^2/2+17Kc^2q$ and $D\le11Kcq$. Hence
$\sqrt C+D\le(1/\sqrt2+16Kc)q<q$.
Except on an unconditional
event of probability $e^{-\iota_T}$, the lemma therefore gives
\begin{equation*}
 W\mathcal A_m+\mathcal H_m\le\iota_T^2\qquad(k\le m\le T)
\end{equation*}
on $\mathcal R_\delta$, where the initial event holds and every localization
indicator equals one. Lemma~\ref{lem:energy-budget} allows a random initial state,
so the concentration bound holds unconditionally before the intersection with $\mathcal R_\delta$.
Intersecting these concentration events over deterministic $T\ge2k$
defines $\mathcal P_\delta$ with total failure at most
$\sum_{T\ge2k}e^{-\iota_T}$.

\subsection{Direct Last-Iterate Control}\label{app:anytime-upper-direct}
On $\mathcal R_\delta\cap\mathcal P_\delta$, for every $T\ge2k$,
the full endpoint estimate~\eqref{eq:potential-endpoint} gives
\begin{equation*}
 \begin{split}
 G_{T+1}&\le\frac{2(W\mathcal A_T+\mathcal H_T)}{W}
 \le\frac{2\iota_T^2}{W}\\
 &\le\frac4{c_p}T^{-\beta}\ell_T^2\iota_T^2
 \le\frac4{c_p}T^{-\beta}\iota_T^4.
 \end{split}
\end{equation*}
Indeed, the suffix contains the last half of the updates, so
$W\ge\sum_{t=k}^T\eta_t\ge(T/2)\eta_T
=(c_p/2)T^\beta/\ell_T^2$.
For $1\le T<2k$, use $k\le2h^{1/b}$, $0<\beta\le2/5$ and
$\beta/b=p-1$. On $\mathcal R_\delta$,
\[
 T^\beta G_{T+1}\le\tfrac12h^2(4h^{1/b})^\beta
 \le h^{p+1}\le h^3\le\iota_T^4.
\]
Since $4/c_p\ge1$, the same rate thus covers the early times as well.
The total failure probability is at most
\begin{equation*}
 e^{-h}+\sum_{T\ge2k}e^{-\iota_T}
 \le\frac{\delta}{100}\left(1+\sum_{T\ge1}T^{-2}\right)<\delta.
\end{equation*}
\subsection{Rescaling to the Original Parameters}\label{app:anytime-upper-scaling}
For $L,R>0$ and $\sigma\ge0$, put $Q=\max\{LR,\sigma\}$ and use the normalized coordinate $y=(x-x_1)/R$ with objective
\[
 F(y)=\frac{f(x_1+Ry)-f^\star}{RQ}.
\]
Its minimizer is within distance one of zero, its smoothness is $LR/Q\le1$, and the normalized oracle is $g/Q$, with conditional central $p$th moment at most $(\sigma/Q)^p\le1$. Applying the normalized theorem with step $c_pt^{-a}/\ell_t^2$ and threshold $4t^b$ gives, after multiplying coordinates by $R$, exactly~\eqref{eq:anytime-upper-schedule}. Multiplying the objective bound by $RQ$ yields~\eqref{eq:anytime-upper-rate}, and~\eqref{eq:all-radius} becomes $\norm{x_t-x^\star}^2\le R^2[\log(100/\delta)]^2$. This scaling also covers $\sigma=0$ since $Q=LR>0$.

\subsection{Almost Sure Consequence}\label{app:anytime-upper-as}
The schedule~\eqref{eq:anytime-upper-schedule} does not depend on $\delta$, so Theorem~\ref{thm:anytime-upper} applies to the same trajectory for every $\delta\in(0,1)$. Setting $\delta=1/n$ yields events of probability at least $1-1/n$ on which $f(x_{T+1})-f^\star\le D_pRQ\,T^{-\beta_p}\ell_T^4(1+\log n/\ell_T)^4$ for every $T\ge1$. Consequently, the rate also holds almost surely: with probability one,
\[
 \limsup_{T\to\infty}\frac{f(x_{T+1})-f^\star}{T^{-\beta_p}\ell_T^4}\le D_pRQ.
\]

\Needspace{10\baselineskip}
\subsection{Proof of Theorem~\ref{thm:fixed-time-upper}}\label{app:fixed-time-upper}
The proof reuses the radius event and the suffix potential of Theorem~\ref{thm:anytime-upper}, with the fixed time as the analysis horizon. For $r>\beta_p$, the range of the late clipping errors, weighted by the potential, grows with the horizon, so Lemma~\ref{lem:energy-budget} does not apply directly. We truncate these errors at deterministic levels that depend on $T$ and $\delta$, apply the lemma on the event that no truncation occurs, and bound the probability of truncation through the $p$th moment. The factor $\delta^u$ shortens the steps and lowers the thresholds, which postpones truncation and makes it less likely; the shorter steps cost the factor $\delta^{-u}$.

We keep the normalized setting and the constants~\eqref{eq:anytime-constants}, fix $r\in(\beta_p,\alpha_p]$ and $\delta\in(0,1)$, and write $h=\log(100/\delta)$ and $\iota_T=\log(100T^2/\delta)$. Put
\begin{gather*}
 a_r=1-r,\qquad b_r=\frac r{p-1},\qquad d_r=a_r-b_r=\frac p{p-1}(\alpha_p-r)\ge0,\\
 \kappa=\frac{3p-1}{p-1},\qquad E=\kappa(r-\beta_p)>0.
\end{gather*}
Since $\kappa\beta_p=2$, also $E=2b_r+r-2a_r$, and
\begin{equation}\label{eq:tradeoff-identity}
 \frac p{p-1}E+\kappa d_r=\frac{p\kappa}{p-1}(\alpha_p-\beta_p)=1.
\end{equation}
For $u,w\ge0$ let
\[
 \Xi(u,w)=\frac{pu}{p-1}+2w+\frac{d_r}{E}(\kappa u+2w).
\]
We prove the following claim: if $\Xi(u,w)\ge1$, then the schedule~\eqref{eq:fixed-time-schedule} with this $u$ satisfies, for every deterministic $T\ge1$,
\begin{equation}\label{eq:tradeoff-claim}
 \Prob\bigl\{f(x_{T+1})-f^\star\le D_pT^{-r}\delta^{-(u+2w)}\iota_T^4\bigr\}\ge1-\delta.
\end{equation}
The parameter $w$ enters only the analysis. For $u=E=c_{\rm in}(r)$ and $w=0$, \eqref{eq:tradeoff-identity} gives $\Xi=1$. For $u=0$ and $2w=E/(E+d_r)$, $\Xi=2w(1+d_r/E)=1$, and
\[
 \frac E{E+d_r}=\frac{(3p-1)r-2(p-1)}{(2p-1)r-(p-1)}=c_{\rm free}(r).
\]
Thus the claim proves Theorem~\ref{thm:fixed-time-upper} in the normalized setting, and the rescaling of Appendix~\ref{app:anytime-upper-scaling} multiplies $\eta_t$ by $R/Q$, $\lambda_t$ by $Q$, and the bound by $RQ$.

\begin{lemma}[Budgets for the fixed-time schedules]\label{lem:tradeoff-budgets}
Let $u\ge0$, $0<c\le1$, and $\delta\in(0,1)$, and let $\eta_t=c\,\delta^ut^{-a_r}/\ell_t^2$ and
$\lambda_t=4\max\{1,\delta^{u/(p-1)}t^{b_r}\}$, with $B_t$ and $V_t$ as in~\eqref{eq:anytime-clipping}. Then~\eqref{eq:global-anytime-budgets} holds. Fix $T\ge2$ and an integer $1\le k\le T/2$ such that $\lambda_t=4\delta^{u/(p-1)}t^{b_r}$ for all $t\ge k$, and put $S_t=\sum_{s=t}^T\eta_s$, $W=1+S_k$, and $\mathsf Q_t=W\eta_t^2/(1+S_t)$. Then
\begin{gather}
 \frac c2\,\frac{\delta^uT^r}{\ell_T^2}\le S_k\le S_1\le H_\beta c\,\frac{\delta^uT^r}{\ell_T^2},\label{eq:tradeoff-step-sum}\\
 \sum_{t=k}^T\sqrt{\mathsf Q_t}B_t\le\frac{Kc}{2\ell_1},\qquad
 \sum_{t=k}^T\mathsf Q_tV_t\le\frac{Kc^2}{\ell_1^3},\qquad
 \max_{k\le t\le T/2}\mathsf Q_t\lambda_t^2\le\frac{Kc^2}{\ell_1^4},\label{eq:tradeoff-capped}\\
 \sum_{T/2<t\le T}\mathsf Q_tV_t\le\frac{128H_\beta^2c^2}{\ell_T^3}\,\delta^{pu/(p-1)}T^{-d_r}.\label{eq:tradeoff-late-variance}
\end{gather}
\end{lemma}

\begin{proof}
In both branches of $\lambda_t$, $B_t\le2\delta^{-u}t^{-r}$ and $V_t\le8\max\{1,\delta^{(2-p)u/(p-1)}t^{(2-p)b_r}\}$. Since $a_r+r=1$, $d_r\ge0$, $2a_r\ge1$, and $2a_r-(2-p)b_r=1+d_r$,
\[
 \eta_tB_t\le\frac{2c}{t\ell_t^2},\qquad
 \eta_t^2V_t\le\frac{8c^2}{t\ell_t^4},\qquad
 \eta_t\lambda_t\le\frac{4c}{\ell_1^2},
\]
and integral comparison, as in Lemma~\ref{lem:anytime-budgets}, gives~\eqref{eq:global-anytime-budgets}. The proof of~\eqref{eq:step-sum-anytime} gives~\eqref{eq:tradeoff-step-sum}: $S_k\ge(T/2)\eta_T$, and every dyadic block has $\sum t^{-a_r}\le2T^r2^{-rj}\le2T^r2^{-\beta_pj}$. As there, $W/(1+S_t)\le4H_\beta$ for $t\le T/2$, so the global budgets bound the early parts of the bias and variance sums in~\eqref{eq:tradeoff-capped} by $4\sqrt{H_\beta}c/\ell_1$ and $64H_\beta c^2/\ell_1^4$, and the early range by $64H_\beta c^2/\ell_1^4$. For $T/2<t\le T$, the branch assumption, $\ell_T/\ell_t\le2$, and $-a_r+(2-p)b_r=-r-d_r$ give
\[
 B_t\le4\delta^{-u}T^{-r},\qquad
 \eta_tV_t\le\frac{64c\,\delta^{u/(p-1)}T^{-r-d_r}}{\ell_T^2}.
\]
By~\eqref{eq:tail-integrals}, with $W$ in place of $U$ and the sums starting at $k$, the late bias sum is at most $8\delta^{-u}T^{-r}S_k\le8H_\beta c/\ell_T^2$. For the late variance, $W\log W\le S_k(1+\log W)$ and $1+\log W\le H_\beta\ell_T$ give
\[
 \sum_{T/2<t\le T}\mathsf Q_tV_t
 \le\frac{64c\,\delta^{u/(p-1)}T^{-r-d_r}}{\ell_T^2}\,W\log W
 \le\frac{64H_\beta^2c^2}{\ell_T^3}\,\delta^{pu/(p-1)}T^{-d_r}.
\]
Adding both halves and using $K=2^{20}H_\beta^4$ proves~\eqref{eq:tradeoff-capped} and~\eqref{eq:tradeoff-late-variance}.
\end{proof}

\emph{Radius and early times.} Use $c=c_p$. Since $\lambda_t\ge4$ and $\eta_t\le c_p$, the argument of Appendix~\ref{app:anytime-upper-radius}, with the global budgets of Lemma~\ref{lem:tradeoff-budgets}, gives an event $\mathcal R_\delta$ with $\Prob(\mathcal R_\delta^c)\le\delta/100$ on which $\norm{x_t-x^\star}\le h$ and $G_t\le h^2/2$ for all $t\ge1$. Let $k=\lceil\delta^{-u/r}h^{1/b_r}\rceil$. Since $b_r/r=1/(p-1)$, every $t\ge k$ has $\delta^{u/(p-1)}t^{b_r}\ge h\ge1$, hence $\lambda_t=4\delta^{u/(p-1)}t^{b_r}\ge4h$. If $T<2k$, then $T<4\delta^{-u/r}h^{1/b_r}$, and on $\mathcal R_\delta$, using $r\le1/2$, $r/b_r=p-1$, and $p\le2$,
\[
 \delta^uT^rG_{T+1}\le\tfrac12\,4^rh^{p+1}\le h^3\le\iota_T^4,
\]
which proves~\eqref{eq:tradeoff-claim}, since $D_p\ge1$ and $\delta^{-2w}\ge1$.

\emph{Truncation.} Fix a deterministic $T\ge2k$, and apply Lemma~\ref{lem:potential} with $\zeta=1$ to the updates $k,\ldots,T$ as in Appendix~\ref{app:anytime-upper-potential}. This gives $P_m=W\mathcal A_m+\mathcal H_m\ge0$ with $W=1+S_k$, the recursion~\eqref{eq:capped-partial-potential}, and $\norm{w_t}\le\sqrt{2\mathsf Q_t}\sqrt{P_{t-1}}$, where $\mathsf Q_t=\mathsf Q_t^{(k)}$ is the quantity of Lemma~\ref{lem:tradeoff-budgets}. Let $J_t$ indicate that $\norm{x_s-x^\star}\le2h$ for $k\le s\le t$; it is predictable and nonincreasing. On $J_t=1$, $\norm{\nabla f(x_t)}\le2h\le\lambda_t/2$, so Lemma~\ref{lem:clipping} gives~\eqref{eq:anytime-clipping} and $\norm{e_t}\le\norm{g_t-\nabla f(x_t)}$. Set $M_0=\sqrt Kc_p/\ell_1^2$, and define deterministic levels and truncated errors by
\[
 m_t=\infty\quad(k\le t\le T/2),\qquad
 m_t=\frac{M_0\delta^{-w}}{\sqrt{\mathsf Q_t}}\quad(T/2<t\le T),\qquad
 \hat e_t=e_t\mathbf 1_{\{\norm{e_t}\le m_t\}}.
\]
Let $\mathcal T=\{t:m_t<3\lambda_t/2\}$ and $\mathcal E_T=\{J_te_t=J_t\hat e_t\text{ for }k\le t\le T\}$. Only times in $\mathcal T$ can violate $\mathcal E_T$, and the conditional Markov inequality with $\mathbb E_{t-1}\norm{g_t-\nabla f(x_t)}^p\le1$ gives
\[
 \Prob(\mathcal E_T^c)\le\Sigma_T:=\sum_{t\in\mathcal T}m_t^{-p}.
\]
Suppose $\mathcal T\neq\varnothing$. Since $\delta^uT^r\ge\delta^uk^r\ge1$ and $H_\beta c_p\le1$, \eqref{eq:tradeoff-step-sum} gives $W\le2\delta^uT^r$. For $T/2<t\le T$, also $\eta_t\le8c_p\delta^uT^{-a_r}/\ell_T^2$ and $\lambda_t\le4\delta^{u/(p-1)}T^{b_r}$, so
\[
 \mathsf Q_t\lambda_t^2\le W\eta_t^2\lambda_t^2\le\frac{2048c_p^2}{\ell_T^4}\,\delta^{\kappa u}T^E.
\]
A time $t\in\mathcal T$ has $M_0^2\delta^{-2w}<\frac94\mathsf Q_t\lambda_t^2$, so $K\ge4608$ gives
\begin{equation}\label{eq:tradeoff-horizon}
 T^E>\delta^{-(\kappa u+2w)}.
\end{equation}
For $t\in\mathcal T$, $m_t<3\lambda_t/2$ and $0\le2-p<1$ give $m_t^{2-p}\le\frac34V_t$, hence $m_t^{-p}=\delta^{2w}\mathsf Q_tm_t^{2-p}/M_0^2\le\frac34\delta^{2w}\mathsf Q_tV_t/M_0^2$. By~\eqref{eq:tradeoff-late-variance}, $d_r\ge0$, and~\eqref{eq:tradeoff-horizon},
\[
 \Sigma_T\le\frac{96H_\beta^2\ell_1}{K}\,\delta^{\frac{pu}{p-1}+2w}T^{-d_r}
 \le\frac{96H_\beta^2\ell_1}{K}\,\delta^{\Xi(u,w)}\le\frac\delta2.
\]
If $\mathcal T=\varnothing$, then $\Sigma_T=0$.

\emph{Energy bound without truncation.} On $J_t=1$, the truncated errors satisfy
\begin{gather*}
 \norm{\mathbb E_{t-1}\hat e_t}\le B_t+\mathbf 1_{\{t\in\mathcal T\}}m_t^{1-p},\qquad
 \mathbb E_{t-1}\norm{\hat e_t}^2\le V_t,\\
 \norm{\hat e_t}\le\min\{\tfrac32\lambda_t,m_t\},\qquad
 \norm{\hat e_t-\mathbb E_{t-1}\hat e_t}\le2\min\{\lambda_t,m_t\}.
\end{gather*}
The first bound uses $\norm{e_t}\mathbf 1_{\{\norm{e_t}>m_t\}}\le\norm{g_t-\nabla f(x_t)}^pm_t^{1-p}$. For the last, $\hat e_t+\nabla f(x_t)$ equals $\clip(g_t,\lambda_t)$ or $\nabla f(x_t)$, both in the ball of radius $\lambda_t$, and $\norm{\hat e_t}\le m_t$. Put $\tilde P_m=\delta^{2w}P_m$. On $\mathcal E_T\cap\{J_m=1\}$, the recursion~\eqref{eq:capped-partial-potential} holds with $\hat e_t$ in place of $e_t$, so
\[
 \tilde P_m\le\tilde P_{k-1}+\sum_{t=k}^m\bigl(2\delta^{2w}\mathsf Q_t\norm{\hat e_t}^2+\ip{\hat e_t}{\delta^{2w}w_t}\bigr),\qquad
 \norm{\delta^{2w}w_t}\le\delta^w\sqrt{2\mathsf Q_t}\sqrt{\tilde P_{t-1}}.
\]
Apply Lemma~\ref{lem:energy-budget} with the event $\mathcal E_T$, the errors $\hat e_t$, $a_t=2\delta^{2w}\mathsf Q_t$, $s_t=\delta^w\sqrt{2\mathsf Q_t}$, $p_0=h^2/2$, and exponent $q=\iota_T$. Lemma~\ref{lem:tradeoff-budgets}, $\Sigma_T\le1/2$, and the identities $\delta^w\sqrt{\mathsf Q_t}m_t^{1-p}=M_0m_t^{-p}$ and $\delta^{2w}\mathsf Q_tm_t^2=M_0^2$ for $t>T/2$ give
\[
 A\le\frac{2Kc_p^2}{\ell_1^3},\quad K_0\le\frac{9Kc_p^2}{2\ell_1^4},\quad
 B_0\le\frac{\sqrt2Kc_p}{\ell_1},\quad V_0\le\frac{2Kc_p^2}{\ell_1^3},\quad
 J_0\le\frac{2\sqrt{2K}c_p}{\ell_1^2}.
\]
Hence $C\le h^2/2+17Kc_p^2q$ and $D\le11Kc_pq$, so $\sqrt C+D\le(1/\sqrt2+16Kc_p)q<q$, as in Appendix~\ref{app:anytime-upper-potential}. On $\mathcal R_\delta$, $\tilde P_{k-1}\le h^2/2$ and every $J_t=1$. Thus, except on probability $e^{-\iota_T}\le\delta/100$, every outcome in $\mathcal R_\delta\cap\mathcal E_T$ has $\tilde P_T\le\iota_T^2$.

\emph{Conclusion.} On $\mathcal R_\delta\cap\mathcal E_T\cap\{\tilde P_T\le\iota_T^2\}$, the endpoint estimate~\eqref{eq:potential-endpoint}, $W\ge S_k$, and~\eqref{eq:tradeoff-step-sum} give
\[
 G_{T+1}\le\frac{2P_T}W\le\frac{2\delta^{-2w}\iota_T^2}{(c_p/2)\delta^uT^r/\ell_T^2}
 \le D_pT^{-r}\delta^{-(u+2w)}\iota_T^4,
\]
since $\ell_T\le\iota_T$ and $D_p=4/c_p$. The three exceptional events have total probability at most $\delta/100+\delta/2+\delta/100<\delta$, which proves~\eqref{eq:tradeoff-claim}.

\input{appendices/momentum_extension}

\Needspace{12\baselineskip}
\input{appendices/momentum_dimension_uniform}

\input{appendices/momentum_averaging}

%% file: appendices/learning_rate_free.tex
\section{Learning-Rate-Free Methods}\label{app:learning-rate-free}
We apply Proposition~\ref{prop:record-lower} to fixed anchors, coin betting,
and the one-dimensional conic algorithms of
\citet{grandclement2026conic}. We then prove a stronger failure for the
standard zero-initialized CBA and CBA+ iterates.
Throughout this section the oracle is queried at $x_t$ and the output is
$x_{t+1}$. We write $h_t$ for the feedback supplied to an online method,
which can be the raw gradient $g_t$ or its clipped value.

\subsection{Fixed Anchors and Contracting Weight Decay}\label{app:fixed-anchors}
Rules of Definition~\ref{def:class} keep $x_{t+1}=x_t$ whenever $g_t=0$,
so they satisfy Definition~\ref{def:record-rules} in either direction.
More generally, fix a deterministic anchor $a\in\R$ and consider
\begin{equation}\label{eq:fixed-anchor-update}
 x_{t+1}=(1-\rho_t)x_t+\rho_t a-c_tg_t,
 \qquad 0\le\rho_t\le1,\quad 0\le c_t<\infty.
\end{equation}
The coefficients may depend measurably on the current and past gradients
and the seed. Choose $e\in\{-1,1\}$ so that $e(a-x_1)\le0$.
With $U_t=e(x_t-x_1)$ and $M_t=\max\{0,U_1,\ldots,U_t\}$, a zero
observation gives
\[
 U_{t+1}=(1-\rho_t)U_t+\rho_te(a-x_1)
 \le(1-\rho_t)M_t+\rho_tM_t=M_t.
\]
This proves the record condition for every fixed anchor, regardless of
its distance from $x_1$. When $a=x_1$, either direction works.
In particular, weight decay toward the initialization,
\[
 x_{t+1}=x_1+(1-\eta_t\omega_t)(x_t-x_1)
          -\eta_t\clip(g_t,\lambda_t),
 \qquad \eta_t\ge0,\quad 0\le\eta_t\omega_t\le1,
\]
has the form~\eqref{eq:fixed-anchor-update}, with
$\rho_t=\eta_t\omega_t$ and the clipping multiplier absorbed into $c_t$.
Thus both lower bounds of Proposition~\ref{prop:record-lower} hold for
these rules.

\subsection{One-Dimensional KT Coin Betting}\label{app:kt-record}
We use the Krichevsky--Trofimov (KT) update of
\citet[Algorithm~1]{orabona2016coin}, with normalized rewards $-h_t/G$
and initial endowment $\epsilon/G$.
Fix $G,\epsilon>0$, let $|h_t|\le G$, and define
\begin{equation}\label{eq:kt-update}
 S_{t-1}=\sum_{s<t}h_s,\qquad
 W_t=\epsilon-\sum_{s<t}h_sx_s,\qquad
 x_t=-\frac{W_tS_{t-1}}{G^2t}.
\end{equation}
Here $x_1=0$. Since $|S_{t-1}|\le G(t-1)$,
\[
 W_{t+1}=W_t-h_tx_t
 =W_t\left(1+\frac{h_tS_{t-1}}{G^2t}\right)
 \ge\frac{W_t}{t}>0,
\]
starting from $W_1=\epsilon$. Expanding the next iterate gives
\begin{align*}
 x_{t+1}
 &=-\frac{W_t}{G^2(t+1)}
 \left(S_{t-1}+h_t+\frac{h_tS_{t-1}^2}{G^2t}
                  +\frac{h_t^2S_{t-1}}{G^2t}\right)\\
 &=\left(1-\frac{1-h_t^2/G^2}{t+1}\right)x_t
   -\frac{W_t(1+S_{t-1}^2/(G^2t))}{G^2(t+1)}h_t.
\end{align*}
The anchor coefficient lies in $[0,1/(t+1)]$ and the step factor is
positive. For raw gradients of arbitrary magnitude we may use any
finite measurable preprocessing
\[
 h_t=k_tg_t,\qquad k_t\ge0,\qquad |h_t|\le G,
\]
with $h_t=0$ when $g_t=0$. Examples are $h_t=\clip(g_t,G)$, or
$h_t=\clip(g_t,\lambda_t)/\lambda_t$ with $G=1$ and $\lambda_t>0$.
Multiplying the displayed step factor by $k_t$ gives
\eqref{eq:fixed-anchor-update} with anchor zero, so
Proposition~\ref{prop:record-lower} applies. In a zero-feedback round,
$x_{t+1}=tx_t/(t+1)$; after a nonzero first observation this differs from
$x_t$. Thus KT belongs to the larger class even though it is not a rule
of Definition~\ref{def:class}.

\subsection{One-Dimensional Conic Algorithms}\label{app:conic-record}
We first consider Algorithms~1 and~2 of \citet{grandclement2026conic} on
$[-D,D]$, where $D>0$ is the decision radius. Their lifted state
$(y_t,v_t)$ belongs to
\[
 \mathcal C_D=\{(y,v):v\ge0,\ |y|\le Dv\},
 \qquad
 x_t=\begin{cases}y_t/v_t,&v_t>0,\\0,&(y_t,v_t)=(0,0).
 \end{cases}
\]
Both start from $(y_1,v_1)=(0,0)$ and use lifted feedback
$(h_t,-h_tx_t)$. We allow $h_t=k_tg_t$ with finite $k_t\ge0$,
including arbitrary positive clipping thresholds.

\paragraph{CBA.}
The fixed quadratic FTRL update (Algorithm~1) is
\begin{equation}\label{eq:cba-ftrl}
 (y_{t+1},v_{t+1})=\mathop{\rm argmin}_{(y,v)\in\mathcal C_D}
 \left\{\frac{y^2+D^2v^2}{2}
       +\sum_{s=1}^t h_s(y-x_sv)\right\}.
\end{equation}
The objective is coercive and strictly convex on a closed convex set,
so its minimizer is unique. If $g_t=0$, then $h_t=0$, and the cumulative
objective and its minimizer are unchanged. Hence $x_{t+1}=x_t$, proving
the record condition in both directions for every history.

\paragraph{CBA+ and its projection.}
In the Hilbert specialization of Algorithm~2, the change of variables
$(y,v)\mapsto(y,Dv)$ makes the weighted projection Euclidean. Its
unprojected point is
\[
 (\widetilde y,\widetilde v)=(y_t-h_t,Dv_t+h_tx_t/D),
\]
and the cone in these coordinates is $\widetilde v\ge|\widetilde y|$.
Set $r_t=Dv_t+y_t$ and $\ell_t=Dv_t-y_t$. Squared Euclidean distance in
$(y,Dv)$ is one half the sum of squared distances in $(r,\ell)$, so
projection onto the cone takes the positive part of each coordinate:
\begin{equation}\label{eq:cba-plus-coordinates}
 \begin{aligned}
 r_{t+1}&=[r_t-h_t(1-x_t/D)]_+,\\
 \ell_{t+1}&=[\ell_t+h_t(1+x_t/D)]_+.
 \end{aligned}
\end{equation}
Away from the origin, $x_t=D(r_t-\ell_t)/(r_t+\ell_t)$.
If $h_t>0$, then $r_{t+1}\le r_t$ and $\ell_{t+1}\ge\ell_t$.
A nonzero state remains nonzero: if $\ell_t>0$, then
$\ell_{t+1}\ge\ell_t>0$; if $\ell_t=0$, then $x_t=D$ and
$r_{t+1}=r_t>0$. Therefore
\[
 x_{t+1}-x_t=
 \frac{2D(r_{t+1}\ell_t-r_t\ell_{t+1})}
 {(r_{t+1}+\ell_{t+1})(r_t+\ell_t)}\le0.
\]
For $h_t<0$, the two coordinate inequalities reverse. The denominator
remains positive because either $r_t>0$ persists or, if $r_t=0$, then
$x_t=-D$ and $\ell_{t+1}=\ell_t>0$. Thus $x_{t+1}-x_t\ge0$.
For $h_t=0$ both coordinates stay fixed. At the origin, a nonzero
feedback gives
\begin{equation}\label{eq:cba-first-response}
 r_{t+1}=(-h_t)_+,\qquad \ell_{t+1}=(h_t)_+,\qquad
 x_{t+1}=-D\operatorname{sign}(h_t).
\end{equation}
These cases show that, relative to the raw gradient,
$x_{t+1}=x_t-c_tg_t$, where $c_t=0$ if $g_t=0$ and
$c_t=(x_t-x_{t+1})/g_t$ otherwise. This factor is nonnegative,
measurable, and at most $2D/|g_t|$ when $g_t\ne0$. Thus one-dimensional
CBA+ satisfies Definition~\ref{def:class} itself, including the projection.

\paragraph{Conic FTRL with a known feedback bound.}
For Algorithm~3 of \citet{grandclement2026conic}, take anchor zero,
$|h_t|\le G$, $\epsilon>0$, and $d_t=\vartheta G^2t$ with
$\vartheta\ge1$. The domain is $[-D,D]$ or $\R$. With $S_{t-1}$ and
$W_t$ as in~\eqref{eq:kt-update}, the lifted objective is
\begin{equation}\label{eq:conic-barrier-objective}
 \Phi_t(y,v)=\frac{d_t}{2\epsilon}y^2+S_{t-1}y+W_tv-\epsilon\log v,
 \qquad v>0,
\end{equation}
subject to $|y|\le Dv$ on the interval; the constraint is absent on $\R$.
We minimize this objective and decode $x_t=y_t/v_t$.
Suppose first that $W_t>0$. Completing the square in $y$ shows that
sublevel sets stay away from $v=0$ and have bounded $v$ and $y$, since
$W_tv-\epsilon\log v$ tends to infinity at both ends. Strict convexity
therefore gives a unique minimizer. The unconstrained minimizer is
\[
 y_t=-\frac{\epsilon S_{t-1}}{d_t},\qquad v_t=\frac{\epsilon}{W_t}.
\]
If it satisfies $|y_t|\le Dv_t$, it is also the constrained minimizer.
Otherwise the constrained minimizer lies on $|y|=Dv$: an interior
minimizer would satisfy the same unconstrained first-order equations.
For $S_{t-1}\ne0$, its $y$ coordinate has the opposite sign, since
reflection decreases the linear term if their signs agree, and $y=0$
has a feasible descent direction. When $S_{t-1}=0$, the minimizing
$y$ is zero. We have proved
\begin{equation}\label{eq:conic-ftrl-decoding}
 x_t=\clip\left(-\frac{W_tS_{t-1}}{d_t},D\right)
\end{equation}
on the interval, with no clipping on $\R$.
In either case,
\[
 |x_t|\le\frac{W_t|S_{t-1}|}{d_t}
 \le\frac{W_t(t-1)}{\vartheta Gt},\qquad
 W_{t+1}\ge W_t-G|x_t|
 \ge W_t\left(1-\frac{t-1}{\vartheta t}\right)>0.
\]
Starting from $W_1=\epsilon$, this proves existence and positivity at
every round. On a history of nonpositive feedback, $S_{t-1}\le0$ and
$x_t\ge0$. A zero observation leaves $S_{t-1}$ and $W_t$ unchanged,
while $d_{t+1}\ge d_t$, so~\eqref{eq:conic-ftrl-decoding} gives
$0\le x_{t+1}\le x_t$. Reflection proves the corresponding statement
for nonnegative histories. The bounded preprocessing used for KT
preserves these signs and zeros, so Algorithm~3 also satisfies
Definition~\ref{def:record-rules}. On $\R$ with $\vartheta=1$ it is KT.

\subsection{Why Querying at the Running Average Is Different}\label{app:averaging-record}
The algorithm of Theorem~\ref{thm:averaging-positive} queries the running
average $x_t$ and updates an auxiliary state $z_t$ by
\eqref{eq:averaging-algorithm}. Starting from $z_1=x_1=0$, take
$g_1=-A$ and $g_2=0$, and put
$d=\gamma_1\min\{A,\lambda_1\}>0$. Its first two updates give
\[
 z_2=d,\qquad x_2=d/2,\qquad z_3=d,\qquad x_3=2d/3.
\]
The second round advances the rightward query record despite zero
feedback. Reflecting the history excludes the leftward direction as
well. Thus neither direction satisfies~\eqref{eq:zero-feedback-record}.
The auxiliary state can move ahead of the queries, which then continue
toward it on zero-feedback rounds. This is consistent with the
$\alpha_p$ time-uniform rate in Theorem~\ref{thm:averaging-positive}.

\subsection{No Uniformly Vanishing Last-Iterate Rate for CBA and CBA+}\label{app:cba-no-rate}
For standard CBA and CBA+, a first nonzero feedback produces the same
nonzero decision after any number of zero observations. We use this
property to prove a stronger lower bound than
Proposition~\ref{prop:record-lower} for these two methods.

\begin{proposition}[No vanishing last-iterate rate for CBA and CBA+]\label{prop:cba-no-rate}
Fix $p\in(1,2]$. Run CBA~\eqref{eq:cba-ftrl} or
CBA+~\eqref{eq:cba-plus-coordinates} on $[-1,1]$ with zero lifted
initialization and default decision zero, querying at the decoded state
and outputting the next decoded state. The method receives either $g_t$
or $h_t=\clip(g_t,\lambda_t)$ with any finite strictly positive
measurable thresholds $\lambda_t$.
\begin{enumerate}
\item For every deterministic integer $t\ge1$, the objective
$f(x)=x^2/2$ admits an oracle satisfying Assumption~\ref{as:oracle}
with $\sigma=1$ for which
\[
 \Prob\{f(x_{t+1})-f^\star=1/2\}=1.
\]
\item For every deterministic envelope $b_t\to0$ and every
$\delta\in(0,1)$, this same objective admits an i.i.d.\ additive-noise
oracle with convex, globally $1$-smooth sample losses attaining their
minima, such that
\[
 \Prob\{\exists t\ge1:\ f(x_{t+1})-f^\star>b_t\}>\delta.
\]
\end{enumerate}
Both constructions satisfy Assumptions~\ref{as:objective}
and~\ref{as:oracle} with $L=R=\sigma=1$ and $x_1=x^\star=0$.
\end{proposition}

\begin{proof}
We first compute the response from the initial state. For CBA+, a zero
feedback leaves the lifted state at zero by
\eqref{eq:cba-plus-coordinates}; for CBA, its cumulative linear term
remains zero, so~\eqref{eq:cba-ftrl} has the same zero minimizer.
After any such prefix, let $h\ne0$ be the first nonzero supplied feedback.
For both methods the next lifted state is
\[
 \mathop{\rm argmin}_{v\ge|y|}
 \left\{\frac{y^2+v^2}{2}+hy\right\}
 =\left(-\frac h2,\frac{|h|}{2}\right).
\]
Indeed, for each $y$ the minimizing height is $v=|y|$, and the remaining
objective is $y^2+hy$. Decoding gives $x=-\operatorname{sign}(h)$.
Positive clipping preserves the nonzero sign of any pulse, independently
of its time and the threshold's magnitude.

For the first part, fix a time $\tau\ge1$ and define on the whole
infinite run
\[
 g_s=x_s+\xi_s,\qquad
 \xi_s=\begin{cases}
 +1\text{ or }-1\text{ with equal probability},&s=\tau,\\
 0,&s\ne\tau.
 \end{cases}
\]
The pulse sign is independent of the past, so
$\E[\xi_s\mid\F_{s-1}]=0$ and
$\E[|\xi_s|^p\mid\F_{s-1}]\le1$ for every $s$.
Before $\tau$, induction gives $x_s=0$ and zero lifted state.
At $\tau$, the first nonzero feedback therefore yields
$x_{\tau+1}=-\operatorname{sign}(\xi_\tau)$ and
$f(x_{\tau+1})-f^\star=1/2$ surely. Since $\tau$ was arbitrary, this
proves the fixed-time statement. It also gives
$\sup_{\mathcal O}\E_{\mathcal O}[f(x_{t+1})-f^\star]\ge1/2$ for every
$t$, where the supremum is over the admissible oracles for this objective.
For any fixed confidence $\delta<1$, a uniformly vanishing fixed-time
envelope eventually falls below $1/2$ and hence fails with probability one.

For the second part, choose $N\ge1$ so that $b_t<1/2$ for all $t\ge N$,
and choose $\pi\in(0,1)$ such that $(1-\pi)^{N-1}>\delta$.
Let the additive noise be i.i.d.\ with
\[
 \Prob(\xi_t=1)=\Prob(\xi_t=-1)=\pi/2,\qquad
 \Prob(\xi_t=0)=1-\pi.
\]
It is centered and has $\Et|\xi_t|^p=\pi\le1$. The first nonzero time
$\tau$ is finite almost surely and satisfies
$\Prob(\tau\ge N)=(1-\pi)^{N-1}>\delta$.
The state stays zero until this time, and the first-response calculation
gives gap $1/2$ at $x_{\tau+1}$. Consequently
\[
 \Prob\{\exists t\ge1:\ f(x_{t+1})-f^\star>b_t\}
 \ge\Prob(\tau\ge N)>\delta.
\]
The sample losses $F(x;\xi)=x^2/2+\xi x$ are convex and globally
$1$-smooth, and each attains its minimum at $-\xi$.
\end{proof}

For $f(x)=x^2/2$, the delayed-pulse argument applies whenever the complete effective state
has a time-homogeneous representation
$s_{t+1}=\Phi(s_t,h_t)$, $x_t=\psi(s_t)$, the initial state satisfies
$\Phi(s_1,0)=s_1$ with $\psi(s_1)=0$, and a fixed bounded symmetric
pulse produces a positive objective gap for both signs. CBA and CBA+
with the fixed quadratic regularizer and zero initialization satisfy
these conditions; moreover, their decoded first response is independent
of every positive pulse magnitude. This last property allows arbitrary
strictly positive clipping thresholds in Proposition~\ref{prop:cba-no-rate}.
A time counter is part of the effective state when it affects decisions.
For example, after $\tau-1$ zero observations, KT has
$S_{\tau-1}=0$ and $W_\tau=\epsilon$, so its first bounded feedback
$h_\tau$ gives
\[
 x_{\tau+1}=-\frac{\epsilon h_\tau}{G^2(\tau+1)}.
\]
Its first-response magnitude decreases with the delay, while its
last-iterate guarantees remain subject to Proposition~\ref{prop:record-lower}.

%% file: appendices/momentum_extension.tex
\section{Proof of the One-Dimensional Momentum Lower Bound}
\label{app:momentum}

We prove Theorem~\ref{thm:momentum-obstruction} for rules that follow the
heavy-ball recursion~\eqref{eq:momentum-main-recursion} under the
bounds~\eqref{eq:momentum-bounds}, with $\clip(g,0)=0$. Such rules observe only stochastic
gradients and output their current iterate; their clipping thresholds, momentum coefficients
(including $\beta_t=0$), and step sizes may all adapt.
Under~\eqref{eq:momentum-main-recursion}, the clipped gradient $\clip(g_j,\lambda_j)$ computed at step $j$ enters the velocity $v_{t+1}$ with weight $\prod_{i=j+1}^t\beta_i$ and the iterate $x_{t+1}$ with weight $-\sum_{s=j}^t\eta_s\prod_{i=j+1}^s\beta_i$.
For an integer $N\ge1$ and $k\ge0$, put
\begin{equation*}
 \kappa_{N,k}=\exp\bigl(-\rho(k-t_0)_+(1+N)^{-a}\bigr).
\end{equation*}

\begin{proposition}[Necessary momentum progress budget]
\label{prop:momentum-budget}
Fix such a rule and a failure budget $0<\varepsilon<1/2$. Let $N\ge1$ and
let $b_1,\ldots,b_N$ be positive. Suppose every instance in the class of Theorem~\ref{thm:momentum-obstruction} satisfies
\[
 \Prob\{f(x_{t+1})-f^\star\le b_t\text{ for all }1\le t\le N\}
 \ge1-\varepsilon.
\]
For every $4b_N<\mu\le1/2$, put
$A=(2\mu)^{-1/(p-1)}$, $\pi=\mu/A$, $s_t=\mu+\sqrt{2b_t}$, and $H_N=H(1+N)^h$. Then
\begin{equation}
 \frac{1-2\varepsilon}{2}
 \le\pi\sum_{t=1}^N\sum_{k=0}^{t-1}
       \min\{s_t,H_NA\kappa_{N,k}\}.
 \label{eq:momentum-convolution-budget}
\end{equation}
\end{proposition}

\begin{proof}
Use the translated family of Appendix~\ref{app:lower-witnesses}, fixed
before the run:
\[
 F_{\mu,y}(x;Y)=\tfrac12(x-y+\mu)_+^2-AYx,
 \qquad y\in[\mu,1],\qquad Y\sim\operatorname{Bernoulli}(\pi).
\]
The population objective gap is given by~\eqref{eq:translated-gap}, and for $x\le y-\mu$ the
stochastic gradient is $-AY$. In particular, all sample and population
properties stated in the theorem hold, with
$\E|\mu-AY|^p\le\mu^p+\mu A^{p-1}\le1$.
Thus these instances satisfy the assumed prefix guarantee, and the rule
meets the conditions of Lemma~\ref{lem:first-crossing-coupling}, with the
velocity as auxiliary state. Let $X_t$ and $V_t$ be the
position and velocity of its reference run, with $c=1/2$. Before stopping,
let $R_t=-V_t$ and write
\[
 R_{t+1}=\beta_tR_t+Y_ta_t,\quad
 a_t=\min\{A,\lambda_t\}\in[0,A],\quad
 X_{t+1}=X_t+D_t,\quad D_t=\eta_tR_{t+1},
\]
with $X_1=R_1=0$. Set $D_t=0$ after stopping. Unrolling the recursion gives
$R_{t+1}=\sum_{j\le t}\bigl(\prod_{i=j+1}^t\beta_i\bigr)Y_ja_j$. For $t\le N$,
each index $i\ge t_0$ in the product satisfies
$\beta_i\le1-\rho(1+N)^{-a}\le\exp(-\rho(1+N)^{-a})$, and at least
$(t-j-t_0)_+$ of its indices are at least $t_0$, so
$\prod_{i=j+1}^t\beta_i\le\kappa_{N,t-j}$. Hence
\begin{equation}
 0\le D_t\le H_NA\sum_{j=1}^t\kappa_{N,t-j}Y_j\quad(t\le N).
 \label{eq:momentum-reference-bound}
\end{equation}
This induction uses only the signs and pathwise bounds, so it holds for
controls that depend on the current gradient and on internal randomness.
In particular, the reference run is nondecreasing, so $M_t=X_t$ and its
record increments are $D_t$. Thus
Lemma~\ref{lem:first-crossing-coupling} and~\eqref{eq:momentum-reference-bound} give
\[
 c(1-2\varepsilon)\le\E\sum_{t=1}^N\min\{s_t,D_t\}
 \le\pi\sum_{t=1}^N\sum_{k=0}^{t-1}
              \min\{s_t,H_NA\kappa_{N,k}\}.
\]
For the second inequality, use
$\min\{s,\sum_jw_j\}\le\sum_j\min\{s,w_j\}$ for $w_j\ge0$ and $\E Y_j=\pi$.
\end{proof}

\begin{proof}\textit{of Theorem~\ref{thm:momentum-obstruction}.}
Suppose the asserted envelope holds on every instance. Since $a<1$ and
$\beta_p\le2/5$, we have $(1+a)\beta_p<1$; choose
$(1+a)\beta_p<u<\min\{r,1\}$ and put $b_t=b_{N,t}$ from
\eqref{eq:sum-envelope}, with $D_N=C\log^m(eN)$. This dominates
the asserted envelope for $t\le N$. For sufficiently large $N$,
$4b_N<\mu\le1/2$, and Proposition~\ref{prop:momentum-budget} applies with
$\varepsilon=\delta$. If $H=0$, the required progress is already
impossible. Otherwise split the inner sum at
\[
 K_N=t_0+\left\lceil\frac{(1+N)^a}{\rho}\log_+\Bigl(\frac{H_NA}{\mu}\Bigr)\right\rceil,
 \qquad \log_+(v)=\max\{0,\log v\}.
\]
For $k\ge K_N$, $H_NA\kappa_{N,k}\le\mu\exp(-\rho(k-K_N)(1+N)^{-a})$, and
$(1-e^{-x})^{-1}\le1+x^{-1}$ for $x>0$, so
\begin{equation}
 \sum_{k=0}^{t-1}\min\{s_t,H_NA\kappa_{N,k}\}
 \le s_t\left(K_N+1+\frac{(1+N)^a}{\rho}\right).
 \label{eq:momentum-log-budget}
\end{equation}
The factor is $O(N^a\log(eN))$, with constants depending on $H,h,\rho,t_0,a$, since
$\log H_N=\log H+h\log(1+N)$ and
\[
 A/\mu=2^{-1/(p-1)}\mu^{-p/(p-1)}.
\]
By~\eqref{eq:exact-rate-contradiction} with $K_N$ replaced by this factor,
the right side of~\eqref{eq:momentum-convolution-budget} is
\[
 O\bigl(D_N^{q+1/2}\log(eN)\,N^{1+a-u/\beta_p}\bigr),
\]
which tends to zero because $u>(1+a)\beta_p$; this contradicts $(1-2\delta)/2>0$.
The contradiction at one finite $N$ yields a fixed instance of this class
on which~\eqref{eq:momentum-obstruction} holds.
\end{proof}

The same budget gives the fixed-time analogue, in the sense of Theorem~\ref{thm:anytime-marginal}.

\begin{proposition}[Fixed-time momentum lower bound]\label{prop:momentum-fixed-time}
Fix $p\in(1,2]$ and $a\in[0,1)$. Let $(\mathrm{rule}_\delta)_{\delta\in(0,1)}$ satisfy~\eqref{eq:momentum-main-recursion}
and~\eqref{eq:momentum-bounds} with $\rho$, $t_0$, and $h$ independent of $\delta$, and with $H$
replaced by $H\delta^{-h_0}$ for constants $H,h_0\ge0$ independent of $\delta$. For every finite
$C>0$, $m\ge0$, $c\in[0,1-(1+a)\beta_p)$, and $r>(1+a)\beta_p+c$, there are $\delta\in(0,1)$,
a deterministic integer $t\ge1$, and a one-dimensional instance with $L=R=\sigma=1$, fixed before the run and satisfying the conditions of Theorem~\ref{thm:momentum-obstruction}, such that
$\mathrm{rule}_\delta$ satisfies
$\Prob\{f(x_{t+1})-f^\star>Ct^{-r}\delta^{-c}[\log(et/\delta)]^m\}>\delta$.
\end{proposition}

An exponent above $(1+a)\beta_p$ at a fixed time thus costs a polynomial confidence factor, and with
uniformly polylogarithmic dependence on $1/\delta$ it is unattainable.

\begin{proof}\textit{of Proposition~\ref{prop:momentum-fixed-time}.}
Suppose the envelope holds for every $\delta$, every deterministic time, and every instance.
Since $c<1-(1+a)\beta_p$ and $r>(1+a)\beta_p+c$, choose $(1+a)\beta_p+c<u<\min\{r,1\}$ and
put $v=u-c>(1+a)\beta_p$. For each deterministic $N$, let $\delta_N$, $D_N$, and
$b_{N,t}=D_NN^ct^{-u}$ be as in the proof of Theorem~\ref{thm:anytime-marginal}; the envelope of
$\mathrm{rule}_{\delta_N}$ is at most $b_{N,t}$ for $t\le N$, and a union bound gives the prefix
guarantee of Proposition~\ref{prop:momentum-budget} with failure budget $1/4$ on every instance.
On this prefix, $\eta_t\le H_N:=H(4N)^{h_0}(1+N)^h$, and the proof of that proposition uses the
step sizes only through this bound, so it applies with $\mu=8b_{N,N}=8D_NN^{-v}$, which lies in
$(4b_{N,N},1/2]$ for large $N$. Splitting the inner sums as in~\eqref{eq:momentum-log-budget},
whose factor is again $O(N^a\log(eN))$ because $\log H_N=O(\log(eN))$, and summing $s_t$ as in the
proof of Theorem~\ref{thm:anytime-marginal}, the right side of~\eqref{eq:momentum-convolution-budget} is
\[
 O\bigl(D_N^{q+1/2}\log(eN)\,N^{1+a-v/\beta_p}\bigr),
\]
which tends to zero because $v>(1+a)\beta_p$; this contradicts $1/4$. At such an $N$, some
instance fixed before the run fails the prefix guarantee with probability above $1/4=N\delta_N$,
and the union bound gives a deterministic $t\le N$ whose fixed-time failure probability exceeds $\delta_N$.
\end{proof}

%% file: appendices/momentum_dimension_uniform.tex
\section{A Dimension-Uniform Momentum Lower Bound}
\label{app:momentum-dimension}

We prove Theorem~\ref{orth:thm:main}. Throughout, the control constants are common over dimensions, as required by
$\mathcal A_{\mathrm{unif}}(a)$, and so are the rate constants $C$ and $k$, as in Theorem~\ref{orth:thm:main}. Shared constants match our upper bounds, whose constants are likewise dimension-free, and allow the proof to choose the dimension after fixing the horizon.
The proof augments the objective with auxiliary quadratic coordinates, the nuisance block, that carry i.i.d.\ signed coordinate noise at every round. On a finite prefix with distinct coordinates, an orthogonal test vector bounds the useful displacement that each rare observation accumulates over its whole forgetting window by the allowed nuisance radius, even when the controls depend on the current noisy gradient. This removes the factor $\tilde O(N^a)$ of the one-dimensional bound, and polynomial forgetting makes the remaining tail of each spike's effect negligible.

\subsection{A Pathwise Projection Identity}

Fix the numerical scalar controls on a realized finite path; they may
depend on the observations. Write
$\clip(g_s,\lambda_s)=c_sg_s$, with $c_s\in[0,1]$ and $c_s=0$ when
$g_s=0$. On the nuisance block, where $g_s$ has component $y_s+\zeta_s$ with nuisance noise $\zeta_s$,
the components $y_s$ of the state and $v_s$ of the velocity satisfy
$v_{s+1}=\beta_sv_s+c_s(y_s+\zeta_s)$ and
$y_{s+1}=y_s-\eta_sv_{s+1}$, with $y_1=v_1=0$.
An impulse of size $A$ at time $\tau$ has flat response
\begin{equation}\label{orth:eq:impulse}
 h_t=Ac_\tau\sum_{s=\tau}^{t-1}\eta_s
                   \prod_{i=\tau+1}^{s}\beta_i\quad(t>\tau),
 \qquad h_t=0\quad(t\le\tau).
\end{equation}
Nonnegative controls make $h_t$ nonnegative and nondecreasing.

\begin{lemma}[Orthogonal pulse budget]\label{orth:lem:pulse}
Suppose all $\zeta_j$, $1\le j<t$, are nonzero and mutually orthogonal,
with $\zeta_j=R_j\epsilon_je_{J_j}$, $R_j\ge B>0$, signs
$\epsilon_j\in\{-1,1\}$, and distinct $J_j$.
For $\tau<t$ with $R_\tau=A$, let $h_t$ be~\eqref{orth:eq:impulse}.
If $\|y_t\|\le R$ and $R\sqrt{t-\tau-1}\le B/2$, then $h_t\le2R$.
\end{lemma}
\begin{proof}
Define the test vector
\[
 w=\frac{\zeta_\tau}{A}
   +\sum_{j=\tau+1}^{t-1}\frac{h_j}{R_j^2}\zeta_j.
\]
Orthogonality gives $\ip{\zeta_s}{w}=h_s+A\ind_{\{s=\tau\}}$ for $s<t$.
Let $k_1=0$ and $k_{s+1}=\beta_sk_s+Ac_s\ind_{\{s=\tau\}}$;
then~\eqref{orth:eq:impulse} gives $h_{s+1}=h_s+\eta_sk_{s+1}$.
Induction in the nuisance recurrence yields
$\ip{v_s}{w}=k_s$ and $\ip{y_s}{w}=-h_s$ for $s\le t$:
the terms $-h_s+h_s$ cancel at every update. Therefore
\begin{equation*}
 h_t^2\le\|y_t\|^2\|w\|^2
 =\|y_t\|^2\left(1+\sum_{j=\tau+1}^{t-1}\frac{h_j^2}{R_j^2}\right)
 \le R^2+\tfrac14h_t^2.
\end{equation*}
Thus $h_t\le2R/\sqrt3\le2R$; as the identity is pathwise, $w$ may be chosen after the path.
\end{proof}

\subsection{A Fixed I.I.D. Witness and the Rate Contradiction}

\begin{proof}\textit{of Theorem~\ref{orth:thm:main}.}
Choose $\max\{\beta_p,a\}<u<\min\{r,1\}$.
A finite common $D>0$ makes $b_t=Dt^{-u}$ dominate the proposed
logarithmic envelope. Suppose this weaker envelope holds with
probability at least $1-\delta$ on every finite-dimensional instance.
For a large deterministic $N$, set
\begin{equation*}
 \mu=8DN^{-u},\quad \kappa=1/128,\quad B=1/16,\quad
 A=(\kappa/\mu)^{1/(p-1)},\quad \pi=\mu/A.
\end{equation*}
Choose an integer $d_N\ge2N^2/(1-\delta)$. Take $N$ large enough that
$\mu\le1/2$, $A\ge1$, and $\pi\le1/2$.
Independently each round, draw $Y\sim\operatorname{Bernoulli}(\pi)$,
a uniform $J\in\{1,\ldots,d_N\}$, and a uniform sign $\epsilon$,
independent within a sample as well. Let $R_0=B,R_1=A$, and define
\begin{equation*}
 F(x,y;Y,J,\epsilon)=\tfrac12(x-1+\mu)_+^2-AYx
                   +\tfrac12\|y\|^2+R_Y\epsilon y_J.
\end{equation*}

\emph{Admissibility.} Each sample is convex and globally $1$-smooth:
the two orthogonal Hessian blocks are bounded by one.
In the first coordinate $x$, the sample loss attains its minimum at any $x\le1-\mu$ when $Y=0$, and at
$x=1-\mu+A$ when $Y=1$; in the nuisance block $y$, its minimizer is $-R_Y\epsilon e_J$.
Thus every sample minimum is finite and attained.
The population is
$f(x,y)=\tfrac12(x-1+\mu)_+^2-\mu x+\tfrac12\|y\|^2$,
with minimizer $(1,0)$ at initial distance one. The fixed additive
noise $\xi=(\mu-AY,R_Y\epsilon e_J)$ is centered. Since
$\pi A^p=\kappa$ and $\mu=\pi A$,
\[
 \mathbb E|\mu-AY|^p\le2\kappa,\qquad
 \mathbb E\|\xi\|^p
 \le2^{p-1}(2\kappa+B^p+\kappa)
 \le 6/128+2/16<1.
\]
For every state the gap is at least $\|y\|^2/2$; when $x\le1/2$ it
is also at least $\mu(1-x)-\mu^2/2\ge\mu/4$.

\emph{Coupling and useful events.} Feed the rule reference gradients
$(-AY_t,y_t+R_{Y_t}\epsilon_te_{J_t})$ until the first update whose
first coordinate $X_{t+1}$ exceeds $1/2$, or until round $N$.
Denote that last retained round by $S\le N$.
The first velocity is nonpositive, so $X_t$ is nondecreasing.
Under the shared random draws and seed, the reference and actual trajectories coincide
through step $S$: prior to step $S$, every queried first coordinate satisfies $X_t\le 1/2\le 1-\mu$,
so all stochastic gradients, histories, and controls coincide.
Let $\mathcal G$ be the actual instance's success event for all
$t\le N$. Then $\Prob(\mathcal G)\ge1-\delta$.
On $\mathcal G$ a crossing occurs by $N$, since otherwise the final
gap is at least $\mu/4=2b_N$; also
\begin{equation}\label{orth:eq:radius}
 \|y_t\|\le\sqrt{2b_{t-1}}\quad(2\le t\le S+1).
\end{equation}
Let $\mathcal C$ mean that $J_1,\ldots,J_N$ are distinct. The union
bound gives $\Prob(\mathcal C^c)\le N(N-1)/(2d_N)\le(1-\delta)/4$.
For $Q_p=p/(p-1)$ choose $0<\theta<\min\{1,uQ_p\}$ and set
$n_0=\lceil N^\theta\rceil$. Let $\mathcal E$ mean $Y_t=0$ for
$t<n_0$. Then $\Prob(\mathcal E^c)\le n_0\pi
=O(N^{\theta-uQ_p})\to0$; the nuisance noise still acts.\looseness=-1

\emph{Each pulse has a short useful window.} Set
$P=h+u/(p-1)+5$, $\Lambda=3^aP/\rho$, and
$w_j=\lceil\Lambda(1+j)^a\log N\rceil$.
Uniformly for $n_0\le j\le N$, large $N$ gives $j\ge t_0$, $w_j\le j$,
and
\begin{equation}\label{orth:eq:small-window}
 \frac{2b_jw_j}{B^2}
 =O\bigl(N^{\theta(a-u)}\log N+N^{-\theta u}\bigr)\longrightarrow0.
\end{equation}
Fix the realized controls through $S$. For a rare arrival $j\le S$,
let $h_t^{(j)}$ be the impulse response~\eqref{orth:eq:impulse} with
$\tau=j$, and put
$t_j=\min\{S+1,j+w_j+1\}$.
On $\mathcal G\cap\mathcal C$, the nuisance observations
$\zeta_1,\ldots,\zeta_{t_j-1}$ are nonzero and orthogonal, each has
norm at least $B$, and the vector at $j$ has norm $A$.
Equation~\eqref{orth:eq:radius} gives
$\|y_{t_j}\|\le\sqrt{2b_j}$; thus
Lemma~\ref{orth:lem:pulse} and \eqref{orth:eq:small-window} give
$h_{t_j}^{(j)}\le2\sqrt{2b_j}$ for every $j\ge n_0$.

If $t_j<S+1$, the first $w_j$ subsequent momentum coefficients satisfy
\[
 \prod_{i=j+1}^{j+w_j}\beta_i
 \le\exp\left(-\frac{\rho w_j}{3^a(1+j)^a}\right)\le N^{-P}.
\]
Every subsequent factor is at most one. Since $c_j\le1$, the polynomial bound on the step sizes
in~\eqref{eq:momentum-dimension-envelope} implies pathwise that
\begin{equation}\label{orth:eq:tail}
 0\le h_{S+1}^{(j)}-h_{t_j}^{(j)}
 \le AHN(1+N)^hN^{-P}\le C_0N^{-4},
\end{equation}
with $C_0$ depending only on common fixed parameters; indeed
$A=(\kappa/(8D))^{1/(p-1)}N^{u/(p-1)}$.
For $t_j=S+1$ the tail is zero.\looseness=-1

\emph{Contradiction on total progress.} The displacement along the first coordinate is
$X_{S+1}=\sum_{j\le S}Y_jh_{S+1}^{(j)}$.
On $\mathcal G\cap\mathcal C\cap\mathcal E$, it must exceed $1/2$,
but the preceding bounds give
\[
 \tfrac12<X_{S+1}\le
 2\sqrt{2D}\sum_{j=n_0}^N Y_jj^{-u/2}+C_0N^{-3}.
\]
Multiply by that event's indicator and take expectations. Dropping
the indicator on the nonnegative right side and using $\mathbb E Y_j=\pi$ yields
\begin{equation}\label{orth:eq:final-budget}
 \tfrac12\Prob(\mathcal G\cap\mathcal C\cap\mathcal E)
 \le2\sqrt{2D}\,\pi\sum_{j=1}^N j^{-u/2}+C_0N^{-3}.
\end{equation}
The left side has liminf at least $3(1-\delta)/8$.
The right side is $O(N^{1-u(Q_p+1/2)})+O(N^{-3})\to0$ because
$u>\beta_p=1/(Q_p+1/2)$ and $u<1$.
This contradiction gives prefix failure strictly greater than $\delta$
for one fixed law and finite dimension.
\end{proof}

The same argument gives the fixed-time analogue.

\begin{proposition}[Fixed-time dimension-uniform lower bound]\label{orth:prop:fixed-time}
Fix $p\in(1,2]$ and $a\in[0,1)$, let $C>0$ and $k\ge0$ be finite, and let $r>\max\{\beta_p,a\}$. Let $(\mathrm{rule}_\delta)_{\delta\in(0,1)}$ satisfy the conditions of
$\mathcal A_{\mathrm{unif}}(a)$ with $\rho$, $t_0$, and $h$ independent of $\delta$, and with
$H$ replaced by $H\delta^{-h_0}$ for constants $H,h_0\ge0$ independent of $\delta$. There are
$\delta\in(0,1)$, a deterministic integer $t\ge1$, and a finite-dimensional instance with $L=R=\sigma=1$, fixed before the samples and the seed are drawn and satisfying the conditions of Theorem~\ref{orth:thm:main}, such that
$\mathrm{rule}_\delta$ satisfies
$\Prob\{f(x_{t+1})-f^\star>Ct^{-r}[\log(et/\delta)]^k\}>\delta$.
Here $C$ and $k$ are shared across dimensions and do not depend on $\delta$ or $t$.
\end{proposition}

With polylogarithmic dependence on $1/\delta$, no fixed-time exponent above
$\max\{\beta_p,a\}$ is therefore attainable either.

\begin{proof}\textit{of Proposition~\ref{orth:prop:fixed-time}.}
Suppose the envelope holds for every $\delta$, every deterministic time, and every
finite-dimensional instance. Choose $\max\{\beta_p,a\}<u<\min\{r,1\}$. For each
deterministic $N$, put
\[
 \delta_N=\frac1{4N},\qquad D_N=C\log^k(4eN^2),\qquad b_{N,t}=D_Nt^{-u}.
\]
Since $t^{-r}\le t^{-u}$ and $\log(et/\delta_N)\le\log(4eN^2)$ for $1\le t\le N$, the envelope of
$\mathrm{rule}_{\delta_N}$ is at most $b_{N,t}$ on this prefix, and a union bound shows that, on every instance,
$f(x_{t+1})-f^\star\le b_{N,t}$ for all $t\le N$ with probability at least $3/4$; there $\eta_t\le H(4N)^{h_0}(1+N)^h$.

Run the proof of Theorem~\ref{orth:thm:main} for $\mathrm{rule}_{\delta_N}$, with $D$ replaced by $D_N$, $1-\delta$ by
$3/4$, and $P$ by $h+h_0+u/(p-1)+5$. Its estimates depend on the rule only through the common
constants $\rho,t_0,H,h,h_0$, and $C\le D_N=N^{o(1)}$, so they hold for all large $N$, with
constants independent of $N$ and of the rule. Explicitly, $\mu=8D_NN^{-u}\to0$, $d_N\ge8N^2/3$
gives $\Prob(\mathcal C^c)\le3/16$, and, uniformly for $n_0\le j\le N$,
\[
 \begin{gathered}
 \pi=\kappa^{-1/(p-1)}\mu^{Q_p}=O(D_N^{Q_p}N^{-uQ_p}),\qquad
 \Prob(\mathcal E^c)\le n_0\pi=O(D_N^{Q_p}N^{\theta-uQ_p}),\\
 \frac{2b_{N,j}w_j}{B^2}=O\bigl(D_N[N^{\theta(a-u)}\log N+N^{-\theta u}]\bigr),
 \end{gathered}
\]
all of which tend to zero. In~\eqref{orth:eq:tail}, $A\le(\kappa/(8C))^{1/(p-1)}N^{u/(p-1)}$ and
$\eta_s\le H(4N)^{h_0}(1+N)^h$; since $1+h+h_0+u/(p-1)-P=-4$, the tail is still $O(N^{-4})$. The left
side of~\eqref{orth:eq:final-budget} is at least $(3/4-3/16-o(1))/2$, with liminf at least
$9/32$, whereas its right side is $O(D_N^{Q_p+1/2}N^{1-u(Q_p+1/2)})+O(N^{-3})\to0$ as
$u>\beta_p=1/(Q_p+1/2)$. Hence, for some $N$, the instance of dimension $d_N$ fails the prefix
guarantee with probability above $1/4=N\delta_N$, and the union bound gives a deterministic $t\le N$
with fixed-time failure probability above $\delta_N$.
\end{proof}

%% file: appendices/momentum_averaging.tex
\section{Time-Uniform Averaging Momentum}
\label{sec:averaging-positive}

This appendix proves Theorem~\ref{thm:averaging-positive} and Propositions~\ref{prop:averaging-momentum} and~\ref{prop:averaging-unbounded-gain}.

\begin{proof}\textit{of Theorem~\ref{thm:averaging-positive}.}
Write $F_t=f(x_t)-f^\star$, $\mu_t=\nabla f(x_t)$, and
$D_t=\|z_t-x^\star\|$.
The averaging update gives
\begin{equation}
 x_t=\frac1t\sum_{s=1}^t z_s,\qquad
 z_t=t x_t-(t-1)x_{t-1}.
 \label{eq:averaging-average-identity}
\end{equation}
At $t=1$ the second identity is independent of $x_0$.

\emph{A smooth online-to-batch inequality.}
By the standard inequality for smooth convex functions
\citep[Theorem~2.1.5]{nesterov2004} and $\nabla f(x^\star)=0$, we have
\begin{equation*}
 \langle\nabla f(x),x-x^\star\rangle
 \ge f(x)-f^\star+\frac{\|\nabla f(x)\|^2}{2L}.
\end{equation*}
Combining it with~\eqref{eq:averaging-average-identity} and
$\langle\mu_t,x_t-x_{t-1}\rangle\ge F_t-F_{t-1}$ gives
\begin{equation}
 \langle\mu_t,z_t-x^\star\rangle
 \ge tF_t-(t-1)F_{t-1}+\frac{\|\mu_t\|^2}{2L}.
 \label{eq:averaging-otb}
\end{equation}
This is the smooth form of the anytime online-to-batch inequality \citep{cutkosky2019anytime}.

\emph{A pathwise energy recursion.}
Let $e_t=d_t-\mu_t$. Expanding the update and using
\eqref{eq:averaging-otb} and $\|\mu_t+e_t\|^2\le2\|\mu_t\|^2+2\|e_t\|^2$ gives
\[
 D_{t+1}^2\le D_t^2-2\gamma_t\{tF_t-(t-1)F_{t-1}\}
 +\Delta_t,\qquad
 \Delta_t=2\gamma_t^2\|e_t\|^2-2\gamma_t\langle e_t,z_t-x^\star\rangle.
\]
Here the omitted term involving $\|\mu_t\|^2$ is nonpositive since
$\gamma_t\le c_0/(4L\bar\ell_t^2)\le1/(2L)$.
Define $W_1=D_1^2$ and
\begin{equation*}
 W_{t+1}=D_{t+1}^2+2t\gamma_tF_t
       +2\sum_{s=1}^{t-1}s(\gamma_s-\gamma_{s+1})F_s.
\end{equation*}
Since $\gamma_t$ is nonincreasing, $W_t\ge D_t^2\ge0$.
The exact difference
$W_{t+1}-W_t=D_{t+1}^2-D_t^2+2\gamma_t\{tF_t-(t-1)F_{t-1}\}$
therefore gives $W_{t+1}\le W_t+\Delta_t$, also at $t=1$ without an $F_0$ term.

\emph{Predictable localization and conditional clipping.}
Let $\tau=\inf\{t\ge1:W_t\ge2R^2\}$ and $I_t=\mathbf1_{\{t<\tau\}}$.
Each $W_s$ is $\mathcal F_{s-1}$-measurable, so $I_t$ is $\mathcal F_{t-1}$-measurable, that is, predictable, and nonincreasing in $t$.
Since $\|d_s\|\le\lambda_s$, the states $z_t$, $x_t$ and the errors $e_t$ are bounded at each fixed $t$, so $e_t$ is integrable.
On $t<\tau$, all $s\le t$ satisfy $D_s<\sqrt2R$. Thus
\eqref{eq:averaging-average-identity} gives $\|x_t-x^\star\|<2R$ and
$\|\mu_t\|\le2LR\le\lambda_t/2$.
Lemma~\ref{lem:clipping}, applied conditionally on this predictable event, yields
\begin{equation}
 \|e_t\|\le\frac32\lambda_t,\qquad
 \|\mathbb E_{t-1}e_t\|\le B_t:=2\sigma^p\lambda_t^{1-p},\qquad
 \mathbb E_{t-1}\|e_t\|^2\le V_t:=2\sigma^p\lambda_t^{2-p}.
 \label{eq:averaging-clipping}
\end{equation}
For $q>1$, integral comparison gives
\begin{equation*}
 \sum_{t\ge1}\frac1{t(h_\delta+\log t)^q}
 \le\frac1{h_\delta^q}+\frac1{(q-1)h_\delta^{q-1}}.
\end{equation*}
Since $\lambda_t^p\ge\sigma^pt$, the deterministic budgets are
\begin{equation}
 \sum_t\gamma_tB_t\le\frac{4c_0R}{h_\delta},\qquad
 \sum_t\gamma_t^2V_t\le\frac{4c_0^2R^2}{h_\delta^3},\qquad
 \gamma_t\lambda_t=\frac{c_0R}{\bar\ell_t^2}.
 \label{eq:averaging-noise-budgets}
\end{equation}
These include $\sigma=0$, when $B_t=V_t=0$.

\emph{One application of the common energy bound.}
Apply Lemma~\ref{lem:energy-budget} with $P_m=W_{m+1}$ for $m\ge0$, the predictable indicators $J_t=I_t$, $\mathcal E=\Omega$, $p_0=R^2$, $q=h_\delta$, and
\[
 a_t=2\gamma_t^2,\qquad w_t=-2\gamma_t(z_t-x^\star),\qquad s_t=2\gamma_t,\qquad
 M_t=\tfrac32\lambda_t,\qquad U_t=2\lambda_t.
\]
The pathwise recursion $W_{t+1}\le W_t+\Delta_t$ is the required one, $P_0=D_1^2\le R^2$, and $\|w_t\|\le s_t\sqrt{P_{t-1}}$ because $D_t^2\le W_t$. On $I_t=1$, \eqref{eq:averaging-clipping} gives $B_t$, $V_t$, and $M_t$, and Lemma~\ref{lem:clipping} also gives $\|e_t-\mathbb E_{t-1}e_t\|\le U_t$. By~\eqref{eq:averaging-noise-budgets} and $\gamma_t\lambda_t=c_0R/\bar\ell_t^2\le c_0R/h_\delta^2$, the budgets of the lemma satisfy
\[
 A\le\frac{8c_0^2R^2}{h_\delta^3},\qquad K_0\le\frac{9c_0^2R^2}{2h_\delta^4},\qquad
 B_0\le\frac{8c_0R}{h_\delta},\qquad V_0\le\frac{16c_0^2R^2}{h_\delta^3},\qquad J_0\le\frac{4c_0R}{h_\delta^2},
\]
so $C\le R^2(1+29c_0^2/h_\delta^3)$ and $D\le24c_0R/h_\delta$. Since $c_0=1/128$ and $h_\delta\ge8$, $(\sqrt C+D)^2\le R^2(1+29c_0^2/h_\delta^3+24c_0/h_\delta)^2<2R^2$. Except on probability $e^{-h_\delta}<\delta$, the lemma therefore gives $W_{m+1}<2R^2$ whenever $I_m=1$. If $\tau<\infty$, then $\tau\ge2$ because $W_1\le R^2$, and $m=\tau-1$ has $I_m=1$ but $W_{m+1}=W_\tau\ge2R^2$. Hence $\tau=\infty$ on this event, and simultaneously for every $T\ge1$,
\begin{equation*}
 D_{T+1}^2+2T\gamma_TF_T\le W_{T+1}<2R^2.
\end{equation*}
In particular, $D_{T+1}<\sqrt2R$ and
$F_T\le R^2/(T\gamma_T)=128R\lambda_T\bar\ell_T^2/T$.

\emph{The last iterate.}
Convexity of the implemented update and smoothness at the minimizer give
\begin{align*}
 f(x_{T+1})-f^\star
 &\le\frac{T}{T+1}F_T+\frac{f(z_{T+1})-f^\star}{T+1}\\
 &\le\frac{128R\lambda_T\bar\ell_T^2+2LR^2}{T+1}.
\end{align*}
Finally,
$\lambda_T\le4LR+\sigma T^{1/p}$, $\bar\ell_T\ge1$, and
$(T+1)^{-1}\le T^{-1}$ prove~\eqref{eq:averaging-rate}.
\end{proof}

\begin{proposition}[Exact momentum conventions]
\label{prop:averaging-momentum}
The states in~\eqref{eq:averaging-algorithm} obey
\[
 x_{t+1}=x_t-\eta_t d_t+\widetilde\beta_t(x_t-x_{t-1}),\qquad
 \eta_t=\frac{\gamma_t}{t+1},\qquad
 \widetilde\beta_t=\frac{t-1}{t+1},
\]
with $x_0=x_1$. In the velocity convention
$v_{t+1}=\beta_tv_t+d_t$, $x_{t+1}=x_t-\eta_tv_{t+1}$, $v_1=0$,
the equivalent parameters are $\beta_1=0$ and
\[
 \beta_t=\frac{t-1}{t}\frac{\gamma_{t-1}}{\gamma_t},\qquad
 \frac{\alpha_p}{2t}\le1-\beta_t\le\frac1t\quad(t\ge2),
 \qquad \alpha_p=\frac{p-1}{p}.
\]
The step sizes are uniformly bounded by $c_0/(8Lh_\delta^2)$.
\end{proposition}
\begin{proof}
Use $z_t-x_t=(t-1)(x_t-x_{t-1})$ in the averaging update.
Eliminating velocity then gives
$\beta_t=\widetilde\beta_t\eta_{t-1}/\eta_t$.
For $r(s)=\max\{4LR,\sigma s^{1/p}\}(h_\delta+\log s)^2$,
$\gamma_t=c_0R/r(t)$, and almost everywhere
\[
 0\le\frac{d\log r(s)}{d\log s}
 \le\frac1p+\frac2{h_\delta}\le1-\frac{\alpha_p}{2}.
\]
Integrating over $[t-1,t]$ proves
\[
 \frac{t-1}{t}\le\beta_t
 \le\left(\frac{t-1}{t}\right)^{\alpha_p/2}<1.
\]
Concavity of $u\mapsto u^{\alpha_p/2}$ yields the asserted lower
bound on $1-\beta_t$; the upper bound is immediate. The bound on the step sizes
follows from $\lambda_t\ge4LR$ and $t+1\ge2$.
\end{proof}

\begin{proposition}[Fixed momentum coefficient with unbounded step sizes]
\label{prop:averaging-unbounded-gain}
For every fixed $b\in(0,1)$, the trajectory in
Theorem~\ref{thm:averaging-positive} admits the exact representation
\[
 w_{t+1}=bw_t+\operatorname{clip}(g_t,\Lambda_t),\qquad
 x_{t+1}=x_t-H_tw_{t+1},\qquad w_1=0,
\]
with deterministic finite step sizes $H_t\to\infty$ and clipping thresholds $\Lambda_t$
that may depend on the current stochastic gradient. In particular, a constant momentum coefficient $b\in(0,1)$
does not by itself force the $\beta_p$ barrier if exponentially growing step sizes $H_t$ and
clipping thresholds depending on the current gradient are permitted. Because $H_t$ grows exponentially, it violates the polynomial bound on the step sizes in
Theorem~\ref{thm:momentum-obstruction}, so some restriction on the growth of the step sizes is needed there.
\end{proposition}
\begin{proof}
Set
\[
 H_1=\frac{\gamma_1}{2(1-b)^2},\quad
 H_t=H_1\frac{2b^{-(t-1)}}{t(t+1)},\quad
 s_t=\frac{t\gamma_tb^{t-1}}{2H_1},\quad
 \Lambda_t=s_t\min\{\|g_t\|,\lambda_t\}.
\]
Since $\gamma_t\le\gamma_1$ and
$\sum_{t\ge1}tb^{t-1}=(1-b)^{-2}$, $0<s_t\le1$.
Therefore $\operatorname{clip}(g_t,\Lambda_t)=s_td_t$.
At $t=1$ the displacement is $-\gamma_1d_1/2$.
For $t\ge2$, eliminating $w_t$ gives
\[
 x_{t+1}-x_t
 =b\frac{H_t}{H_{t-1}}(x_t-x_{t-1})-H_ts_td_t
 =\frac{t-1}{t+1}(x_t-x_{t-1})-\frac{\gamma_t}{t+1}d_t.
\]
This is exactly Proposition~\ref{prop:averaging-momentum}, path by path.
Also $H_t\to\infty$ exponentially. When $g_t=0$, take clipping at
zero to be zero; if strictly positive thresholds are required, assign
any positive threshold on that event without changing the update.
\end{proof}

\emph{Comparison with existing averaging methods.}
\citet[Theorems~C.2 and~D.1(2)]{nguyen2023} obtain the
statistical exponent with horizon-free averaged methods. For their
Algorithm~3, the proof of their Theorem~D.1, via Lemma~D.4 and Proposition~D.5, gives $\|z_t-x^\star\|\le2R$
on nested events over each finite prefix. Combining their query relation
$x_t=(1-a_t)y_t+a_tz_t$ ($a_t=2/(t+1)$) with convexity and $L$-smoothness at $x^\star$ gives
\[
 f(x_{T+1})-f^\star\le f(y_{T+1})-f^\star+\frac{4LR^2}{T+2}.
\]
Intersecting these nested events yields a time-uniform guarantee at the query state $x_{T+1}$ as well, although their three-sequence accelerated scheme differs from the two-sequence recursion~\eqref{eq:averaging-algorithm}. The high-probability regret bounds of
\citet[Theorems~4 and~5]{zhangcutkosky2022} assume bounded expected gradients
and a horizon-dependent learner.